\documentclass[final,12pt,oneside,cleveref]{jmlr}

\usepackage{mathtools}
\usepackage{mathrsfs}
\usepackage{bm}
\usepackage{dsfont}
\usepackage{booktabs}
\usepackage{array}
\usepackage{microtype}
\usepackage{wrapfig}
\usepackage{tikz}
\usepackage{algorithm}
\usepackage{algorithmic}

\newcommand{\R}{\mathbb{R}}
\newcommand{\w}{\mathbf{w}}
\renewcommand{\u}{\mathbf{u}}
\newcommand{\x}{\mathbf{x}}
\newcommand{\p}{\mathbf{p}}
\newcommand{\q}{\mathbf{q}}
\newcommand{\z}{\mathbf{z}}
\newcommand{\e}{\mathbf{e}}
\newcommand{\g}{\mathbf{g}}
\newcommand{\alphab}{\bm{\alpha}}
\newcommand{\W}{\mathcal{W}}
\newcommand{\U}{\mathcal{U}}
\newcommand{\Y}{\mathcal{Y}}
\newcommand{\E}{\mathbb{E}}
\newcommand{\Ical}{\mathcal{I}}
\newcommand{\KL}{\mathrm{KL}}
\newcommand{\Zb}{\mathbf{Z}}
\newcommand{\DReg}{\mathrm{D}\text{-}\mathrm{Reg}}
\newcommand{\indicator}{\mathds{1}}
\renewcommand{\O}{\mathcal{O}}
\DeclarePairedDelimiter{\norm}{\lVert}{\rVert}
\DeclareMathOperator*{\argmin}{arg\,min}
\DeclareMathOperator*{\argmax}{arg\,max}

\let\originalproof\proof
\let\endoriginalproof\endproof
\renewenvironment{proof}[1][Proof]
  {\renewcommand{\proofname}{#1}\originalproof\ignorespaces}
  {\endoriginalproof}

\newenvironment{proofsketch}[1][Proof Sketch]
  {\par\noindent\textbf{#1.}\ \ignorespaces}
  {\hfill$\square$\par}

\newtheorem{assumption}{Assumption}
\crefname{assumption}{Assumption}{Assumptions}
\Crefname{assumption}{Assumption}{Assumptions}
\newaliascnt{property}{theorem}
\newtheorem{property}[property]{Property}
\aliascntresetthe{property}
\crefname{property}{property}{properties}
\Crefname{property}{Property}{Properties}

\title[Universal Dynamic Portfolios]{Universal Dynamic Portfolios}

\author[Zhang et al.]{\Name{Yu-Jie Zhang} \Email{yujiez7@cs.washington.edu}\\
  \addr University of Washington
  \AND
  \Name{Yu-Xiang Wang} \Email{yuxiangw@ucsd.edu}\\
  \addr University of California, San Diego
  \AND
  \Name{Peng Zhao} \Email{zhaop@lamda.nju.edu.cn}\\
  \addr State Key Laboratory for Novel Software Technology, Nanjing University\\
  \addr School of Artificial Intelligence, Nanjing University
  \AND
  \Name{Kevin Jamieson} \Email{jamieson@cs.washington.edu}\\
  \addr University of Washington
}

\jmlrproceedings{}{}
\jmlrpages{}

\begin{document}

\maketitle

\begin{abstract}
  Cover's Universal Portfolio~\citep{91:Cover-universal-portfolios} matches the performance of the best constant rebalanced portfolio in hindsight. We generalize this framework to compete with an arbitrary comparator sequence $\u_1,\ldots,\u_T$, leading to a dynamic regret minimization problem for the log loss where existing methods break down due to potentially unbounded gradients. The log loss is exp-concave, a curvature property that classically yields fast rates for static regret, yet we show that this advantage generally disappears in the dynamic setting. In particular, a linear-loss-type $\sqrt{TP_T}$ dependence is unavoidable, where $P_T=\sum_{t=2}^T\lVert\u_t-\u_{t-1}\rVert_1$ is the standard path length. This limitation stems from the coarse nature of $P_T$, which obscures finer spatial and temporal structure of the comparator sequence. We therefore introduce two structure-aware measures---the Jensen--Shannon distance for spatial structure and the JS$^q$-path length for temporal structure---under which faster rates are attainable when the comparator sequence has favorable structure. To achieve sharp bounds for both measures simultaneously, we develop Universal Dynamic Portfolio, a parameter-free method that combines a new Dirichlet Hedge algorithm with a fixed-share update, while retaining a near-optimal $P_T$ guarantee in the worst case. Finally, under an additional bounded-gradient assumption, we show that OPS admits the faster $T^{1/3}P_T^{2/3}$ dynamic regret rate over all comparator sequences. We attain this rate with a tractable proper algorithm that applies more broadly to general online exp-concave optimization over arbitrary compact convex domains.

\end{abstract}

\section{Introduction}
\label{sec:intro}
Online portfolio selection (OPS), which studies how to sequentially allocate wealth among a set of assets to maximize cumulative returns, is a textbook motivating example in online learning~\citep{book'16:Hazan-OCO}. \citet{91:Cover-universal-portfolios} introduced the problem as a distribution-free model for sequential investment and proposed the seminal Universal Portfolio algorithm. Since then, OPS has attracted substantial interest from the online learning community because it can be formulated as an online convex optimization problem with Cover's logarithmic loss. The rich curvature of this loss can be exploited to obtain fast learning guarantees, making OPS one of the central testbeds for understanding the role of loss curvature in both algorithm design and regret analysis~\citep{ICML'07:ONS-OPS,COLT'20:Open-problem-OPS,NeurIPS'18:log-barrier-portfolio,COLT'22:Portfolio-BISONS,COLT'22:Portfolio-AdaMix-DONS,MOR'25:VB-FTRL}.

Most existing OPS studies focus on static regret, which compares the learner's cumulative loss with that of the best single portfolio in hindsight. In OPS, this benchmark is known as the best constant rebalanced portfolio (CRP), which allocates wealth according to fixed proportions. Many methods, including the classical Universal Portfolio algorithm, achieve the minimax-optimal $\O(d\log T)$ static regret~\citep{91:Cover-universal-portfolios,MOR'98:cost-best-portfolio}. In its dependence on $T$, this logarithmic rate improves upon the $\sqrt{T}$ rate typical of general convex losses with bounded gradients. However, in a continuously evolving and possibly adversarial market, such a constant comparator may be restrictive because it cannot adapt to market changes. This motivates us to extend the OPS framework to allow time-varying comparators.

This non-stationary extension is naturally captured by the notion of dynamic regret in online convex optimization. To formalize this objective, let $\ell_t(\w)=-\ln(\w^\top\x_t)$ denote Cover's loss, where $\w\in\Delta_d$ is the learner's portfolio and $\x_t\in\R_+^d$ is the market return. The dynamic regret~\citep{JMLR'01:Herbster,ICML'03:zinkvich,NIPS'18:Zhang-Ader} measures the gap between the learner's cumulative loss and that of a time-varying comparator sequence $\{\u_t\}_{t=1}^T$ by
\begin{equation}
  \DReg_T(\{\u_t\}_{t=1}^T) = \sum_{t=1}^T \ell_t(\w_t) - \sum_{t=1}^T \ell_t(\u_t). \label{eq:dynamic-regret}
\end{equation}
The above measure is often referred to as \emph{universal} dynamic regret because it seeks guarantees that hold uniformly over all comparator sequences and adapt to their complexity. Dynamic regret reduces to the classical notion of (static) regret by setting $\u_t = \w_* = \arg\min_{\w\in\Delta_d}\sum_{t=1}^T \ell_t(\w)$. Meanwhile, it offers greater flexibility by allowing the comparator sequence ${\u_t}$ to adapt to the underlying environment, rather than being tied to a single realized return. For instance, when the market return $\x_t$ is sampled from a time-varying distribution $\mathcal{D}_t$, a natural choice is $\u_t = \arg\min_{\w \in \Delta_d} \mathbb{E}_{\x_t \sim \mathcal{D}_t}\!\left[\ell_t(\w)\right].$ Compared with the minimizer $\w_t^\star = \arg\min_{\w\in\Delta_d} \ell_t(\w)$ at each round, this avoids chasing noise from a single observation.

\subsection{Related Work and Research Question}
Although non-stationary online learning has been extensively studied over the past decades~\citep{journals/ml/HerbsterW98,ICML'09:Hazan-adaptive,NeurIPS'12:fixed-share-OMD,ICML'16:GyorgyS-shiftregret,NIPS'18:Zhang-Ader,NeurIPS'20:sword,JMLR'21:BCO,COLT'21:Master,NeurIPS'23:covariate_shift,ICML'24:wavelet,JMLR'24:Sword++,JMLR'25:efficient,NeurIPS'25:Jacobsen-kernelized}, dynamic regret for OPS remains surprisingly underexplored. The only closely related result is due to~\citet{97:singer-switching}, who proposed a method that competes with comparators switching among $N$ fixed portfolios. This guarantee covers only a restricted form of dynamic regret because $\u_t$ is confined to a finite set and cannot evolve continuously over time.

The most well-developed results on dynamic regret minimization in the online convex optimization literature measure nonstationarity through the path length
\begin{equation*}
  P_T = \sum_{t=2}^T \norm{\u_t - \u_{t-1}}_1,
\end{equation*}
which quantifies the cumulative variation of the comparator sequence. For general convex losses with bounded gradients, Ader~\citep{NIPS'18:Zhang-Ader} achieves the $\O(G\sqrt{T(1+P_T)})$ dynamic-regret guarantee, where $G$ is the upper bound on the gradient norm. Since Cover's logarithmic loss can have unbounded gradients, this result does not apply directly to OPS, leaving unresolved even the attainability of the canonical $\sqrt{TP_T}$-type dynamic-regret guarantee without a gradient bound.

More importantly, Cover's logarithmic loss is known to be exp-concave, a curvature property that yields logarithmic regret against a static benchmark. This rate is substantially faster than the $\sqrt{T}$-type static regret typical of general convex losses, motivating us to seek an analogous acceleration for dynamic comparators. The literature offers a partial clue. Under bounded gradients, \citet{COLT'21:baby-strong-convex} and \citet{ICML'25:Zhang-mixability} establish that exp-concavity is also beneficial in the dynamic setting, improving the dynamic regret to $T^{1/3}P_T^{2/3}$. Neither result, however, resolves the case of OPS. The bounded-gradient condition is restrictive for Cover's loss, and even under this condition, the algorithms are not compatible with the geometry of portfolio selection: the method of \citet{COLT'21:baby-strong-convex} may predict outside the simplex, whereas that of \citet{ICML'25:Zhang-mixability} requires a projection over distributions that is computationally prohibitive. Taken together, these gaps lead us to ask:
\begin{center}
  \emph{What is the achievable dynamic regret rate for OPS?}
\end{center}

In fact, we believe resolving this question would also clarify the role of loss curvature in dynamic regret minimization for non-stationary online learning.

\begin{table}[!t]
  \centering
  \caption{Summary of dynamic regret bounds under different complexity measures. Here, $J_t=\sqrt{(\mathrm{KL}(\u_t\Vert\bar{\u}_t)+\mathrm{KL}(\u_{t-1}\Vert\bar{\u}_t))/2}$ denotes the Jensen--Shannon (JS) distance between consecutive comparators, where $\bar{\u}_t=(\u_{t-1}+\u_t)/2$ is their midpoint. The notation $\widetilde{\O}(\cdot)$ hides logarithmic factors.}
  \label{tab:dynamic-ops-summary}
  \scriptsize
  \setlength{\tabcolsep}{2pt}
  \renewcommand{\arraystretch}{1.45}
  \resizebox{\linewidth}{!}{\begin{tabular}{ccc}
    \toprule
    \textbf{Measure} & \textbf{Upper bounds} & \textbf{Lower bounds} \\
    \midrule
    \shortstack[c]{{$L_1$-path length}\\
      $P_T=\sum_{t=2}^T \Vert \u_t-\u_{t-1}\Vert_1$}
    & \shortstack[c]{$\widetilde{\O}\big(d+\sqrt{dTP_T}\big)$\\
      (Section~\ref{subsec:path-length})}
    & \shortstack[c]{$\Omega\big(\max\{d\log T,\min\{T,\sqrt{dTP_T}\}\}\big)$\\
      (Section~\ref{subsec:path-length})} \\
    \cmidrule(l){1-3}
    \shortstack[c]{JS-path length\\
      $P_T^{\mathrm{JS}}=\sum_{t=2}^T J_t$}
    & \shortstack[c]{$\widetilde{\O}\big(d(1+T^{\frac{1}{3}}(P_T^{\mathrm{JS}})^{\frac{2}{3}})\big)$\\
      (Section~\ref{subsec:jeffreys-path})}
    & \shortstack[c]{$\Omega\big(\max\{d\log T,\min\{T,d^{\frac{2}{3}}T^{\frac{1}{3}}(P_T^{\mathrm{JS}})^{\frac{2}{3}}\}\}\big)$\\
      (Section~\ref{subsec:jeffreys-path})} \\
    \cmidrule(l){1-3}
    \shortstack[c]{JS$^q$-path length\\
      $P_{T,q}^{\mathrm{JS}}=\sum_{t=2}^T J_t^q$, $q\in[0,1]$}
    & \shortstack[c]{$\widetilde{\O}\big(d(1+T^{\frac{q}{q+2}}(P_{T,q}^{\mathrm{JS}})^{\frac{2}{q+2}})\big)$\\
      (Section~\ref{subsec:q-path})}
    & --- \\
    \bottomrule
  \end{tabular}
  }
\end{table}

\subsection{Our Results}
In this paper, we develop algorithms and matching lower bounds that characterize the dynamic regret achievable for OPS against arbitrary comparator sequences. We first settle the minimax rate under the classical $L_1$-path length, showing that the fast rates typical of exp-concave losses are unattainable. This limitation arises because $L_1$-path length ignores where the movement occurs and how it evolves over time, even though both can substantially affect tracking difficulty. This motivates refined guarantees that adapt to the spatial and temporal structure of the comparator sequence. Our main results are summarized in Table~\ref{tab:dynamic-ops-summary} and detailed below.

\begin{itemize}
  \item {\emergencystretch=2em \textbf{Minimax Rate under the Standard Path Length.} We establish an $\Omega\big(\max\big\{d\log T,\sqrt{dTP_T}\}\big)$ lower bound for OPS. This rules out the favorable $T^{1/3}P_T^{2/3}$ dependence uniformly over arbitrary comparator sequences, despite the exp-concavity of Cover's loss. We complement this lower bound with a black-box reduction from interval regret to dynamic regret, showing that any algorithm with an $\O(d\log T)$ interval regret guarantee achieves $\O\big(\max\{d\log T,\sqrt{dTP_T\log T}\}\big)$ dynamic regret,thereby matching the lower bound up to logarithmic factors.\par}

  \item {\emergencystretch=2em \textbf{Spatial Adaptivity through the Jensen-Shannon Distance.} We provide an algorithm that adapts to the spatial structure of the comparator sequence and achieves $\widetilde{\O}\big(d(1+T^{1/3}(P_T^{\mathrm{JS}})^{2/3})\big)$ dynamic regret. Here, $P_T^{\mathrm{JS}}=\sum_{t=2}^T J_t$ is the path length based on the Jensen--Shannon (JS) distance, where $J_t=\sqrt{(\mathrm{KL}(\u_t\Vert\bar{\u}_t)+\mathrm{KL}(\u_{t-1}\Vert\bar{\u}_t))/2}$ and $\bar{\u}_t=(\u_{t-1}+\u_t)/2$. The resulting bound yields a faster rate for interior comparators while recovering the worst-case $T^{1/2}P_T^{1/2}$ dependence for arbitrary comparator sequences. A corresponding lower bound matches its dependence on $T$ and $P_T^{\mathrm{JS}}$, up to logarithmic factors.\par}

  \item \textbf{Temporal Adaptivity through JS$^q$-Path Length.}
  We further show that comparator sequences with the same JS-path length can differ in tracking difficulty because of their temporal structure. We capture this structure using the JS$^q$-path length $P_{T,q}^{\mathrm{JS}}=\sum_{t=2}^T J_t^q$ for $q\in[0,1]$, with the endpoint convention $J_t^0=\indicator\{J_t>0\}$, and establish a $\widetilde{\O}(d+dT^{q/(q+2)}(P_{T,q}^{\mathrm{JS}})^{2/(q+2)})$ dynamic regret bound that holds simultaneously for all $q\in[0,1]$. The endpoint $q=1$ recovers the $T^{1/3}(P_T^{\mathrm{JS}})^{2/3}$ rate, whereas $q=0$ yields $\O(d(1+\mathsf{S}_T)\log(dT))$, where $\mathsf{S}_T$ is the number of comparator switches. Intermediate values of $q$ interpolate between these endpoints, allowing the regret bound to adapt more finely to the temporal structure of comparator variation.
\end{itemize}

We achieve all of the above upper bounds with a single algorithm by equipping Cover's Universal Portfolio algorithm with fixed-share updates. Despite the simplicity of this modification, proving these guarantees requires a novel mixability-based analysis that uses a Dirichlet comparator to accommodate both the unbounded log loss and the simplex constraint. Section~\ref{subsec:q-path-proof-sketch} outlines the main technical ideas.

\paragraph{Toward General OXO.}
Under an additional bounded-gradient assumption, we show that a fast rate of $\widetilde{\O}\bigl(dG(1+T^{1/3}P_T^{2/3})\bigr)$ is attainable for all comparator sequences. Existing methods achieving this rate either require improper learning or lack a computationally tractable implementation~\citep{COLT'21:baby-strong-convex,ICML'25:Zhang-mixability}, whereas our method is both proper and computationally tractable. Beyond OPS, our approach extends to general online exp-concave optimization over arbitrary compact convex domains, providing a tractable affirmative answer to the question raised by~\citet{COLT'21:baby-strong-convex} of whether strongly adaptive methods can achieve optimal dynamic regret in the proper learning setting. We establish this guarantee via a new two-layer mixability argument, which is detailed in Section~\ref{subsection:two-layer-analysis}.

\paragraph{Organization.} The rest of the paper is organized as follows. Section~\ref{sec:problem-setup} introduces the problem setup and additional related work. Section~\ref{sec:complexity-measures} presents minimax-optimal regret bounds and spatially and temporally adaptive guarantees for OPS without a gradient bound. Section~\ref{sec:results-bounded-gradient} develops a computationally tractable proper method for general OXO under bounded gradients. Finally, Section~\ref{sec:conclusion} concludes the paper.

\section{Problem Setup and Related Work}
\label{sec:problem-setup}

\subsection{Notation and Setup}
\label{subsec:notation-setup}

For a positive integer $n$, let $[n]=\{1,\ldots,n\}$. We write
$\R_+^d$ for the nonnegative orthant and  $\Delta_d = \big\{\w\in\R_+^d:\sum_{i=1}^d w_i=1\big\}$ for the $(d-1)$-dimensional probability simplex.

Online portfolio selection proceeds over $T$ rounds of interaction between the learner and the market. At each round $t \in [T]$, the learner starts with wealth $S_{t-1}$ and distributes it across $d$ assets according to a probability vector $\w_t \in \Delta_d$. The market then reveals the non-negative price relative vector $\x_t \in \mathbb{R}_+^d$, where each component $x_{t,i} \geq 0$ represents the relative return of asset $i$. The learner’s wealth is updated as $S_t = S_{t-1}\sum_{i=1}^d w_{t,i}\cdot x_{t,i}$. After $T$ rounds, the learner's wealth is $S_T = S_0 \cdot \prod_{t=1}^T (\w_t^\top \x_t)$. For OPS in non-stationary environments, our goal is to minimize the dynamic regret~\eqref{eq:dynamic-regret} with the log loss $\ell_t(\w) = -\ln (\w^\top\x_t)$:
\begin{align*}
  \sum_{t=1}^T \ell_t(\w_t) - \sum_{t=1}^T \ell_t(\u_t) = -\ln \frac{\prod_{t=1}^T \w_t^\top\x_t}{\prod_{t=1}^T \u_t^\top \x_t},
\end{align*}
which is equivalent to maximizing the logarithmic ratio of the learner’s cumulative wealth to that of the time-varying investment strategy $\u_1,\dots,\u_T\in\Delta_d$.

For the OPS results in Section~\ref{sec:complexity-measures}, we only assume $\max_{i\in[d]}x_{t,i}>0$ for every $t\in[T]$, which merely excludes degenerate rounds where every portfolio incurs infinite loss. This minimal assumption is what makes the problem technically challenging: the log-loss gradients $\nabla\ell_t(\w)=-\x_t/(\w^\top\x_t)$ can be unbounded, since the denominator $\w^\top\x_t$ may approach zero near the boundary. Standard online convex optimization techniques rely on a uniform gradient bound, which is available only under further restrictions such as bounded return ratios or a domain clipped away from the simplex boundary~\citep{98:EG-OPS,ICML'07:ONS-OPS}. Neither restriction is imposed for these results; Section~\ref{sec:results-bounded-gradient} separately considers bounded gradients.

\subsection{Related Work}
\label{subsec:related-work}

\paragraph{Static Regret for OPS.}
Under the nonzero-return condition stated above, Universal Portfolio~\citep{cover1996universal} achieves the minimax-optimal $\O(d\log T)$ static regret. The method requires integrating over the simplex to generate predictions, which can be implemented in polynomial time using log-concave sampling techniques~\citep{JMLR'02:KalaiV}. More recent work develops more efficient algorithms under the same condition~\citep{ALT'17:Soft-bayes,NeurIPS'18:log-barrier-portfolio,COLT'22:Portfolio-BISONS,COLT'22:Portfolio-AdaMix-DONS,MOR'25:VB-FTRL}. Two Pareto-optimal results in terms of regret and computational efficiency are VB-FTRL~\citep{MOR'25:VB-FTRL}, which achieves
$\O(d\log T)$ regret with $\O(d^2T)$ computational cost per round, and AdaMix+DONS~\citep{COLT'22:Portfolio-AdaMix-DONS}, which attains an $\O(d^2\log^5 T)$ regret bound with $\O(d^3\log^2 T)$ per-round complexity. With bounded gradients, Exponential Gradient~\citep{98:EG-OPS} achieves
$\O(G\sqrt{T})$ regret, while Online Newton Step~\citep{ICML'07:ONS-OPS} attains $\O(dG\log T)$ regret with $\O(d^3)$ cost per round.

\paragraph{Dynamic Regret for Curved Losses.}
There are two lines of research that achieve an $\O(G\max\{\log T, T^{1/3}P_T^{2/3}\})$ dynamic regret for non-stationary OXO~\citep{COLT'21:baby-strong-convex,ICML'25:Zhang-mixability}. Since these methods are primarily designed for general OXO purposes, their regret bounds typically scale with a bound on the gradient norm. However, setting aside the gradient-bound issue, these results still do not directly apply to OPS due to the restrictions imposed by domain constraints.

\begin{itemize}
\item\emph{Reduction-based analysis.} An important research line for non-stationary OXO starts from~\citet{COLT'21:baby-strong-convex} and is followed by~\citet{AISTATS'22:sc-proper,NeurIPS'22:Baby-LQR}. Under certain domain conditions, they provide a reduction from the interval regret bound~\citep{ICML'09:Hazan-adaptive}, which guarantees a static regret bound on each interval, to a fast-rate dynamic regret bound. A key component of their analysis is a precise characterization of the optimal time-varying sequence $\{\u_t^*\}_{t=1}^T$ via KKT conditions and shows that $\{\u_t^*\}_{t=1}^T$ can be tracked by a piecewise-stationary sequence with $M = \O(T^{{1}/{3}}P_T^{{2}/{3}})$ switches. The initial work~\citep{COLT'21:baby-strong-convex} requires improper learning, allowing the algorithm to predict in an extended box-constrained domain in order to obtain a sufficiently strong piecewise-stationary approximation. Later,~\citet{AISTATS'22:sc-proper} show that proper learning can be achieved when the domain is exactly a box. This restriction to box constraints is intrinsic to the KKT-based analysis, as it only imposes coordinate-wise constraints on the optimal sequence. It remains unclear how to extend these analyses to the simplex or more general domains, which would introduce additional coupling constraints and complicate the analysis.

\item\emph{Mixability-based analysis.} Recently, \citet{ICML'25:Zhang-mixability} showed that continuous exponential weights with a fixed-share update achieve fast-rate dynamic regret via mixability~\citep{Vovk'98:mixabiilty-PEA}, which lifts the analysis from pointwise predictors to distributional comparators. While this framework provides additional flexibility, its analysis relies on Gaussian comparators, which are incompatible with the simplex constraint. To enforce the domain constraint, the method requires an information projection onto a set of Gaussian mixture models with potentially infinitely many components, with bounded component means and variances, which makes the procedure computationally intractable. Our work instead uses Dirichlet comparators, which naturally respect the simplex and allow the analysis to accommodate unbounded log-loss gradients. Further details are provided in Section~\ref{subsec:jeffreys-path}.
\end{itemize}

\section{Dynamic Regret for Online Portfolio Selection}
\label{sec:complexity-measures}
This section characterizes the achievable dynamic regret rates for OPS. We first establish the minimax-optimal rate in terms of the commonly used norm-based path length. Our minimax analysis reveals that the standard path length can obscure fine-grained spatial and temporal differences among comparator sequences. Building on this insight, we establish guarantees in terms of the Jensen–Shannon distance that adapt to the spatial structure of comparator movements, together with a family of $q$-order guarantees that further adapt to their temporal distribution.

\subsection{Minimax Rate under the Standard Path Length}
\label{subsec:path-length}

We begin by establishing a lower bound that captures the worst-case difficulty of OPS over the full range of the standard path-length budget.

\begin{theorem}
  \label{thm:lower-bound}
  Consider the OPS problem with $d\geq 2$ assets and $T>2d$. For any online algorithm
  and any $C\in[0,T]$, there exists a comparator sequence
  $\u_1,\dots,\u_T\in\Delta_d$ and $\x_1,\dots,\x_T\in\R_+^d$ such that
  \begin{equation*}
    P_T \leq C \qquad\text{and}\qquad \DReg_T(\{\u_t\}_{t=1}^T) \geq \Omega\!\left( \max\!\left\{ d\log T, \min\!\left\{T,\sqrt{dTC}\right\} \right\} \right).
  \end{equation*}
\end{theorem}

Theorem~\ref{thm:lower-bound} reveals a sharp contrast between stationary and genuinely non-stationary OPS. When $C=0$, our result recovers the standard $\Theta(d\log T)$ minimax rate for competing with a static comparator. Once the comparator variation becomes nontrivial, the dynamic term scales as $\sqrt{dTC}$ for $C\lesssim T/d$, recovering the same $\sqrt{TC}$ dependence as in general convex dynamic regret. Thus, the exp-concavity of Cover's loss alone does not guarantee the favorable $T^{1/3}C^{2/3}$ dependence uniformly over arbitrary comparator sequences, although such a fast rate is achievable for other curved losses, such as squared loss on bounded domains.

The proof of Theorem~\ref{thm:lower-bound} is provided in
Appendix~\ref{app:path-length-lower-bound} via a reduction to the sequential probability assignment (SPA) problem.
The static term follows from the classical minimax lower bound for
competing with the best constant rebalanced
portfolio~\citep{MOR'98:cost-best-portfolio}. For the dynamic term, we restrict attention to the Kelly market with
return vectors $\x_t\in\{\mathbf{e}_i\}_{i=1}^d$, where $\mathbf{e}_i$ denotes the
$i$-th standard basis vector in $\R^d$. Under this restriction, the OPS problem reduces to multi-class SPA under logarithmic loss~\citep[Chapter~9.1]{book/Cambridge/cesa2006prediction}. To establish the lower bound, we partition the horizon into $K$ blocks and construct a
piecewise-stationary environment. In each block, the optimal comparator lies near the boundary of the simplex: most of its mass is placed on the $d$-th asset, while each of the first $d-1$ assets independently receives either a small probability mass or zero. The learner must identify a new set of rare active assets in each block,
incurring $\Omega(d)$ regret per block and hence $\Omega(dK)$ regret in total. Meanwhile, the near-boundary construction ensures that adjacent blockwise comparators differ by only $\O(dK/T)$, yielding $P_T=\O(dK^2/T)$. Choosing $K=\Theta(\min\{T/d,\sqrt{TC/d}\})$ therefore gives $\Omega(dK)=\Omega(\min\{T,\sqrt{dTC}\})$ while ensuring $P_T\leq C$.

\paragraph{Matching Upper Bound via a Black-Box Reduction.}
Classical online learning methods, such as online gradient descent, typically assume that gradients are bounded by $G>0$. Since gradients in OPS need not be bounded, it was previously unknown whether even the $\sqrt{TP_T}$ dependence could be achieved. We close this gap through a black-box reduction from interval regret, obtaining a $G$-free dynamic regret guarantee that matches the preceding lower bound up to logarithmic factors.

\begin{lemma}
\label{lem:upper-bound-minimax}
For the OPS problem, assume there exists an online algorithm $\mathcal{A}$ that, for any interval $\mathcal{I}\subseteq[T]$, attains the interval-regret guarantee
\begin{align}
  \sum_{t\in\mathcal{I}} \ell_t(\w_t) - \min_{\w\in\Delta_d}\sum_{t\in\mathcal{I}} \ell_t(\w) \leq B(T),\label{eq:strongly-adaptive}
\end{align}
where $B\colon\mathbb N\to(0,\infty)$ is a function. Then, for every comparator sequence $\u_1,\ldots,\u_T\in\Delta_d$, the dynamic regret of $\mathcal{A}$ satisfies
\begin{align*}
  \DReg_T(\{\u_t\}_{t=1}^T) =\mathcal{O}\!\left(B(T)+\sqrt{B(T)TP_T}\right),
\end{align*}
where $P_T = \sum_{t=2}^T \Vert \u_t-\u_{t-1}\Vert_1$ is the path length defined in terms of the $L_1$ norm.
\end{lemma}

Lemma~\ref{lem:upper-bound-minimax} provides a black-box reduction from interval regret~\citep{ICML'09:Hazan-adaptive} to dynamic regret for the OPS problem. A key advantage of this guarantee is that it does not require a bounded gradient norm, provided that the algorithms are chosen appropriately. For OPS, such algorithms can be constructed with $B(T)=\mathcal{O}(d\log T)$. For example, one may run the FLH~\citep{ICML'09:Hazan-adaptive} algorithm with Universal Portfolio~\citep{91:Cover-universal-portfolios} or VB-FTRL~\citep{MOR'25:VB-FTRL} as the base learner. In this case, the resulting dynamic regret matches the optimal rate up to logarithmic factors. The proof is provided in Appendix~\ref{app:black-box-reduction}.

We note that reduction-based arguments are widely used in the dynamic regret minimization literature~\citep{ICML'20:Ashok,COLT'21:baby-strong-convex}. The main distinction in our setting is that the loss functions in OPS do not admit a uniform Lipschitz constant. In contrast, existing analyses in OCO typically rely on a bounded gradient norm to relate the instantaneous loss difference to the path length, e.g., $\ell_t(\mathbf{u}) - \ell_t(\mathbf{v}) \le G \lVert \mathbf{u} - \mathbf{v} \rVert$. Nevertheless, we show this issue can be overcome by a direct treatment of the log loss.

\subsection{Spatial Adaptivity through the Jensen--Shannon Distance  }
\label{subsec:jeffreys-path}

The standard $L_1$-path lengths or $L_2$-path lengths quantify the variation of the comparator sequence but are insensitive to their locations inside the simplex. This matters in OPS because the logarithmic loss has nonuniform geometry. Specifically, the hard instance constructed for the lower bound relies on a comparator sequence that stays near the boundary of the simplex, while an analogous sequence in the interior does not exhibit the same hardness. A guarantee based on $P_T$ can only reflect the worst-case difficulty and cannot distinguish easier comparator sequences in the interior. To this end, we introduce a path length based on the Jensen--Shannon distance
\begin{align}
  P_T^{\mathrm{JS}} \coloneqq {\sum_{t=2}^T}{\operatorname{JS}}(\u_t,\u_{t-1})=\sum_{t=2}^T \sqrt{ \frac{1}{2}\mathrm{KL}\left(\u_t\,\Vert\,\bar{\u}_t\right) + \frac{1}{2}\mathrm{KL}\left(\u_{t-1}\,\Vert\,\bar{\u}_t\right) }, \label{eq:jeffreys-path-length}
\end{align}
where $\bar{\u}_t\coloneqq(\u_t+\u_{t-1})/2$ and $\mathrm{KL}(\p\Vert\q) \coloneqq \sum_{i=1}^d p_i\log(p_i/q_i)$ denotes the Kullback--Leibler divergence. Its coordinate-wise logarithmic ratios align $P_T^{\mathrm{JS}}$ with the geometry of the log loss, making it sensitive to the local geometry of each comparator transition.

\paragraph{A Simple and Nearly Optimal Algorithm.}
Using the JS distance to measure comparator variation, we develop Algorithm~\ref{alg:hedge-oco}, whose regret bound adapts to the local geometry of comparator transitions and, in particular, implies a rate faster than $\sqrt{TP_T}$ for interior comparator sequences.

Algorithm~\ref{alg:hedge-oco} is a fixed-share variant of Cover's Universal Portfolio~\citep{91:Cover-universal-portfolios}. It maintains a distribution over the simplex and updates it with exponential weights. After each update, the algorithm mixes in a small fraction of the uniform Dirichlet distribution, replenishing probability mass across the simplex and allowing the learner to shift toward newly favorable portfolios as the environment changes. Algorithm~\ref{alg:hedge-oco} enjoys the following dynamic regret guarantee.

\begin{algorithm}[!t]
  \caption{Universal Dynamic Portfolio}
  \label{alg:hedge-oco}
  \begin{algorithmic}[1]
  \REQUIRE Fixed-share parameter $\mu_t = 1/t$.
  \STATE Initialize $\tilde{P}_1 = P_1 = \mathrm{Dir}(\alphab_1)$ as a Dirichlet distribution with parameters $\alphab_1 = \mathbf{1}$.
  \FOR{$t=1,2,\ldots,T$}
    \STATE The learner submits the prediction $\w_t = \E_{\u\sim P_t}[\u]$ and then observes $\x_t\in\R_+^d$.
    \STATE The learner updates the distributions by
    \begin{align}
      \tilde{P}_{t+1}(\u) &\propto P_t(\u)\exp(-\ell_t(\u)), \quad \forall \u\in\Delta_d, \label{eq:step1-EW}\\
      P_{t+1}(\u) &= (1-\mu_{t+1})\tilde{P}_{t+1}(\u) + \mu_{t+1}\mathrm{Dir}(\alphab_1). \label{eq:step2-FS}
    \end{align}
  \ENDFOR
  \end{algorithmic}
\end{algorithm}

\begin{theorem}
\label{thm:main}
Algorithm~\ref{alg:hedge-oco} with $\mu_t = 1/t$ ensures
\begin{align*}
  \DReg_T(\{\u_t\}_{t=1}^T) \leq\ \O\left( d \Big( T^{\frac{1}{3}} \big(P_T^{\mathrm{JS}}\big)^{\frac{2}{3}} \big(\ln (dT)\big)^{\frac{2}{3}} + \ln(dT) \Big)\right),
\end{align*}
for any comparator sequence $\u_1,\dots,\u_T\in\Delta_d$.
\end{theorem}

Notably, Algorithm~\ref{alg:hedge-oco} achieves this guarantee without parameter tuning or prior knowledge of $P_T^{\mathrm{JS}}$ or $T$. The following lower bound, proved in Appendix~\ref{app:jeffreys-lower-bound}, shows that the dependence on $T$ and $P_T^{\mathrm{JS}}$ is nearly optimal.

\begin{theorem}
\label{thm:lower-bound-kl}
Consider the OPS problem with $d\geq 2$ assets and $T>4d$. For any online algorithm and any $C\in[0,T]$, there exist a comparator sequence $\u_1,\dots,\u_T\in\Delta_d$ satisfying $P_T^{\mathrm{JS}}\leq C$ and market vectors $\x_1,\dots,\x_T\in\R_+^d$ such that
\begin{align*}
  \DReg_T(\{\u_t\}_{t=1}^T) &\geq \Omega\left( \max\left\{ d\log\left(1+\frac{T}{d}\right), \min\left\{ T,d^{\frac{2}{3}}T^{\frac{1}{3}}C^{\frac{2}{3}} \right\} \right\} \right).
\end{align*}
\end{theorem}

Theorem~\ref{thm:main} also yields a refined guarantee in terms of the standard $L_1$-path length $P_T$, with an explicit dependence on the margin of the comparator sequence from the boundary. Specifically, let $\alpha(\u_{1:T})\coloneqq\min_{t\in[T],i\in[d]}u_{t,i}>0$
denote the minimum margin to the boundary. Since
$P_T^{\mathrm{JS}}\leq\bigl(\alpha(\u_{1:T})\bigr)^{-1/2}P_T$,
Theorem~\ref{thm:main} implies
\begin{align*}
  \DReg_T(\{\u_t\}_{t=1}^T) \leq \O\left( d\left( \bigl(\alpha(\u_{1:T})\bigr)^{-\frac13} T^{\frac13}P_T^{\frac23}(\ln(dT))^{\frac23} +\ln(dT) \right) \right).
\end{align*}
This reveals a form of spatial adaptivity that is invisible to the standard path length alone: for the same $P_T$, the regret guarantee improves as the comparator sequence moves farther into the interior of the simplex. The resulting $T^{1/3}P_T^{2/3}$-type dependence for the interior comparators improves over the worst-case $\sqrt{TP_T}$ dependence. The direct $P_T^{\mathrm{JS}}$-based guarantee is sharper still, since it accounts for each comparator variation according to its local position in the simplex rather than through the minimum margin of the entire sequence.

\paragraph{Recovering the Minimax-Optimal Dependence on $T$ and $P_T$.}
The guarantee above also recovers the minimax-optimal $\sqrt{TP_T}$ dependence for comparator sequences on the boundary. For each $\u_t\in\Delta_d$, consider its smoothed counterpart $\tilde{\u}_t=(1-\beta)\u_t+(\beta/d)\mathbf{1}$. As shown in Appendix~\ref{app:jeffreys-minimax-corollary}, $\DReg_T(\{\u_t\}_{t=1}^T)\leq\DReg_T(\{\tilde{\u}_t\}_{t=1}^T)+T\ln(1/(1-\beta))$. Applying the bound above to the interior sequence $\{\tilde{\u}_t\}_{t=1}^T$ and balancing $\beta$ yields the claimed $\sqrt{TP_T}$ dependence up to logarithmic factors.

\begin{corollary}
\label{Cor:Drichlet-minimax}
Suppose $d,T\geq2$. Let $B\geq1$ and let $\mathcal{A}$ be an online algorithm guaranteeing, for every comparator sequence $\u_1,\dots,\u_T\in\operatorname{ri}(\Delta_d)$,
\begin{align*}
  \DReg_T(\{\u_t\}_{t=1}^T) \leq \O\left(dB T^{\frac13} \big(P_T^{\mathrm{JS}}\big)^{\frac23} +d\ln(dT)\right),
\end{align*}
where $\operatorname{ri}(\Delta_d)$ denotes the relative interior of $\Delta_d$. Then $\mathcal{A}$ also guarantees, for any comparator sequence $\u_1,\dots,\u_T\in\Delta_d$,
\begin{align*}
  \DReg_T(\{\u_t\}_{t=1}^T) \leq \O\left( dB^{\frac34}\sqrt{TP_T} +d\ln(dT)\right),
\end{align*}
\end{corollary}

Applying Corollary~\ref{Cor:Drichlet-minimax} to Theorem~\ref{thm:main} with $B=(\ln(dT))^{2/3}$ yields an $\tilde{\O}\big(d(\sqrt{TP_T}+1)\big)$ regret bound for arbitrary comparator sequences, matching the lower bound in terms of $T$ and $P_T$.

\paragraph{Equivalent Implementation and Interval-Regret Guarantee.}
Following the same arguments in~\citet{JMLR'16:closer-adaptive-regret,ICML'25:Zhang-mixability}, one can show Algorithm~\ref{alg:hedge-oco} is equivalent to running the FLH algorithm with Universal Portfolio as the base learner. Therefore, our method naturally enjoys the interval regret guarantee, which also implies a bound on the switching regret.

\begin{proposition}
\label{thm:interval-regret}
For any interval $\mathcal I=[r,s]\subseteq[T]$ and any comparator $\u\in\Delta_d$, Algorithm~\ref{alg:hedge-oco} with $\mu_t=1/t$ ensures
\[
  \sum_{t\in\mathcal I}\ell_t(\w_t) - \sum_{t\in\mathcal I}\ell_t(\u) \leq \O\bigl(d\ln(dT)\bigr).
\]
Furthermore, let $\mathsf{S}_T=\sum_{t=2}^T\indicator\{\u_t\neq\u_{t-1}\}$. Then the dynamic regret of Algorithm~\ref{alg:hedge-oco} is bounded by $\O\bigl(d(\mathsf{S}_T+1)\ln(dT)\bigr)$.
\end{proposition}

\subsection{Temporal Adaptivity through \texorpdfstring{JS$^q$}{JS\string^q}-Path Length}
\label{subsec:q-path}

The JS-path length captures the spatial geometry of comparator movements, but sequences with the same total JS variation can admit different regret guarantees depending on how that variation is distributed over time. For example, comparators with $\mathsf{S}_T$ switches admit $\O(d(\mathsf{S}_T+1)\ln(dT))$ dynamic regret (Proposition~\ref{thm:interval-regret}) even when $P_{T}^{\mathrm{JS}}=\Theta(\mathsf{S}_T)$, whereas the worst-case regret under the same path-length budget is $\Omega(T^{1/3}\mathsf{S}_T^{2/3})$ (Theorem~\ref{thm:lower-bound-kl}). We therefore introduce a family of measures that interpolates between the JS-path length and the switching number. Specifically, for $q\in[0,1]$ and a comparator sequence in $\Delta_d$, we define the JS$^q$-path length as
\begin{align}
  P_{T,q}^{\mathrm{JS}} \mathrel{\coloneqq}\sum_{t=2}^T \operatorname{JS}(\u_t,\u_{t-1})^{q}. \label{eq:q-order-path-length}
\end{align}
The JS$^q$-path length interpolates between two familiar quantities. At $q=0$, it reduces to the switching number $\mathsf{S}_T$ under the convention $0^0=0$, and at $q=1$ it recovers the JS-path length $P_{T}^{\mathrm{JS}}$ of the preceding subsection. The following theorem shows that Algorithm~\ref{alg:hedge-oco} simultaneously achieves the corresponding dynamic regret guarantee for every order $q\in[0,1]$.
\begin{theorem}
\label{thm:q-path-bound}
Let $d,T\geq2$. Algorithm~\ref{alg:hedge-oco} with $\mu_t=1/t$ simultaneously guarantees, for every $q\in[0,1]$,
\begin{align*}
  \DReg_T(\{\u_t\}_{t=1}^T) \leq \O\left( d\left( T^{\frac{q}{q+2}} \bigl(P_{T,q}^{\mathrm{JS}}\bigr)^{\frac{2}{q+2}} \bigl(\ln(dT)\bigr)^{\frac{2}{q+2}} + \ln(dT) \right) \right)
\end{align*}
for every comparator sequence $\u_1,\dots,\u_T\in\Delta_d$.
\end{theorem}

The guarantee holds simultaneously for all $q\in[0,1]$ because neither $q$ nor $P_{T,q}^{\mathrm{JS}}$ is an input to the algorithm. Consequently, one may take the best of these bounds in hindsight. The choices $q=1$ and $q=0$ recover the $T^{1/3}(P_{T}^{\mathrm{JS}})^{2/3}$ dependence of Theorem~\ref{thm:main} and the $\O\bigl(d(\mathsf{S}_T+1)\ln(dT)\bigr)$ switching guarantee of Proposition~\ref{thm:interval-regret}. Interestingly, for every $q>1$, the $q=1$ guarantee already yields the $T^{1-2/(3q)}(P_{T,q}^{\mathrm{JS}})^{2/(3q)}$ dependence via H\"older's inequality $P_{T,1}^{\mathrm{JS}}\leq T^{1-1/q}(P_{T,q}^{\mathrm{JS}})^{1/q}$, matching the leading dependence on the horizon and variation budget in minimax online forecasting~\mbox{\citep{NeurIPS'19:Baby-forecasting}}.

\pagebreak[3]
\begin{wrapfigure}{r}{0.27\textwidth}
\vspace{-0.85\baselineskip}
\centering
\begin{tikzpicture}[x=0.75cm,y=0.75cm,>=stealth]
  \draw[->] (0,0) -- (3.05,0) node[right] {\scriptsize $t$};
  \draw[->] (0,0) -- (0,1.35) node[above] {\scriptsize $[\u_t]_1$};
  \draw[densely dashed,gray] (0,1.05) -- (2.75,1.05);
  \node[anchor=east] at (-0.06,0) {\scriptsize $1/2$};
  \node[anchor=east] at (-0.06,1.05) {\scriptsize $3/4$};
  \draw[thick,blue!70!black]
    plot[domain=1:9,samples=80,variable=\samp]
    ({0.34*(\samp-1)},{1.05*(1-1/\samp)});
  \foreach \k in {1,...,9}
    \fill[blue!70!black]
      ({0.34*(\k-1)},{1.05*(1-1/\k)}) circle (0.75pt);
  \node[below] at (0,0) {\scriptsize $1$};
  \node[below] at (2.72,0) {\scriptsize $T$};
\end{tikzpicture}
\vspace{-0.55\baselineskip}
\end{wrapfigure}
\paragraph{An Intermediate-Order Example.}
The two endpoints do not exhaust the benefits of Theorem~\ref{thm:q-path-bound}. The following example shows how an intermediate order can exploit comparator movements whose magnitudes decay over time. We consider a two-asset market and define $\u_t=(1/2+z_t,1/2-z_t)$, where $z_t=\frac14(1-1/t)$. For this path, $P_{T,1}^{\mathrm{JS}}=\Theta(1)$, so choosing $q=1$ gives $\widetilde{\O}(T^{1/3})$.
At the other endpoint, $P_{T,0}^{\mathrm{JS}}=\mathsf{S}_T=T-1$, so choosing $q=0$ gives $\widetilde{\O}(T)$. When $q=1/2$, we have $P_{T,1/2}^{\mathrm{JS}}=\Theta(\ln T)$, yielding a bound of $\widetilde{\O}(T^{1/5})$, which is substantially better than those at the two endpoints. Appendix~\ref{app:rising-concave-calculation} provides the calculations.
\par
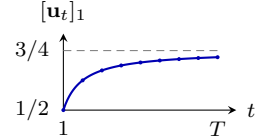
\WFclear

\subsection{Proof Sketch of Theorem~\ref{thm:q-path-bound}}
\label{subsec:q-path-proof-sketch}
Our analysis is based on the notion of mixability, which was first used to analyze the prediction with expert advice problem~\citep{Vovk'98:mixabiilty-PEA} and has proven useful for achieving fast rates in both stochastic learning and online learning~\citep{02:vovk-mixability,JMLR'15:fast-rate,COLT'18:Foster-logistic}. Recently, \citet{ICML'25:Zhang-mixability} used this notion to obtain fast-rate dynamic regret bounds. However, their analysis does not directly extend to OPS when the log-loss gradients are unbounded. We provide a more detailed discussion after briefly sketching the proof of Theorem~\ref{thm:q-path-bound} below. The complete proof is provided in Appendix~\ref{app:q-path-bound}.
~\\
\begin{proofsketch}[Proof Sketch of Theorem~\ref{thm:q-path-bound}]
The starting point of our analysis is that Cover's loss for OPS is 1-mixable over the simplex $\Delta_d$ for any $\x_t \in \R_+^d$, in the sense that for any distribution $P_t$ over $\Delta_d$ and $\w_t = \mathbb{E}_{\u \sim P_t}[\u]$, we have
$\ell_t(\w_t) = -\ln\big(\mathbb{E}_{\u \sim P_t}[\exp(-\ell_t(\u))]\big)$, which implies the following variational identity.
\begin{lemma}
\label{lem:mixability}
For any distribution $Q_t$ with $\mathrm{supp}(Q_t) \subseteq \mathrm{supp}(P_t)$, it holds that
\begin{align*}
  \ell_t(\w_t) = \mathbb{E}_{\u \sim Q_t}[\ell_t(\u)] + \mathrm{KL}(Q_t \,\|\, P_t) - \mathrm{KL}(Q_t \,\|\, \tilde{P}_{t+1}),
\end{align*}
where $\tilde{P}_{t+1}(\u) \propto P_t(\u)\exp(-\ell_t(\u))$ for all $\u \in \mathrm{supp}(P_t)$.
\end{lemma}

By the update rules~\eqref{eq:step1-EW} and~\eqref{eq:step2-FS}, we have $P_{t+1}= (1-\mu_{t+1})\tilde{P}_{t+1} + \mu_{t+1} \mbox{Dir}(\mathbf{1})$. Telescoping over $T$ iterations and upper-bounding the discrepancy between $P_{t+1}$ and $\tilde{P}_{t+1}$ due to the fixed-share update yield
\begin{align*}
  \sum_{t=1}^T \ell_t(\w_t) \mathrel{\leq}{}& \sum_{t=1}^T \E_{\u\sim Q_t}[\ell_t(\u)] + \sum_{t=2}^T \int_{\u\in\Delta_d} \left( Q_t(\u) - Q_{t-1}(\u)\right) \ln \frac{1}{P_t(\u)} \, \mathrm{d}\u\\
  &+\mathrm{KL}(Q_T\Vert P_1)+1+\ln T.
\end{align*}
To accommodate the simplex constraint, we choose the comparator $Q_t = \mbox{Dir}(\mathbf{1}+\gamma \u_t)$ as a Dirichlet distribution, where $\gamma>0$ is a free parameter in the analysis and $\u_t$ is the comparator sequence.

The Dirichlet comparator naturally respects the simplex constraint and allows us to control the expected log loss without any bounded-gradient assumption. By a careful analysis exploiting the structure of the Dirichlet distribution, we can show that
\begin{equation*}
  \left\{ \begin{aligned} &\sum_{t=1}^T\E_{Q_t}[\ell_t(\u)] \leq \sum_{t=1}^T \ell_t(\u_t) + T\ln\left(1+\frac{d}{\gamma}\right), \\
  &\sum_{t=2}^T\int_{\u\in\Delta_d} \left( Q_t(\u)-Q_{t-1}(\u) \right) \ln\frac{1}{P_t(\u)}\mathrm{d}\u \lesssim d\ln(dT)\gamma^{q/2} P_{T,q}^{\mathrm{JS}}. \end{aligned} \right.
\end{equation*}
The first inequality follows from reparameterizing the Dirichlet distribution by several Gamma distributions. For the second inequality, Lemma~\ref{lemma:dirichlet-js} gives
\begin{align*}
  \int_{\Delta_d}|Q_t(\u)-Q_{t-1}(\u)|\,\mathrm{d}\u &\leq \min\left\{2, 2\sqrt{2\gamma}\,\operatorname{JS}(\u_t,\u_{t-1}) \right\}\\
  &\leq 2(2\gamma)^{q/2}\operatorname{JS}(\u_t,\u_{t-1})^q,
\end{align*}
where the last inequality follows from $\min\{1,a\}\leq a^q$ for $a>0$ and $q\in[0,1]$, with the endpoint convention above when $\u_t=\u_{t-1}$.

Combining these bounds and accounting for the endpoint term yields
\begin{align*}
  \DReg_T(\{\u_t\}_{t=1}^T) \lesssim T\ln\left(1+\frac{d}{\gamma}\right) + d\ln(dT)\gamma^{q/2}P_{T,q}^{\mathrm{JS}} + d\ln(1+\gamma) + \ln T.
\end{align*}
Choosing $\gamma = \min\big\{\Theta\big((T / (P_{T,q}^{\mathrm{JS}}\ln(dT)))^{2/(q+2)}\big),\, T\big\}$ yields the claimed bound. Since $\gamma$ only specifies the comparator distributions used in the analysis, Algorithm~\ref{alg:hedge-oco} requires no prior knowledge of $q$ or $P_{T,q}^{\mathrm{JS}}$.
\end{proofsketch}

\begin{remark}[Comparison with {\citet{ICML'25:Zhang-mixability}}]
\label{rem:comparison-mixability}
Related mixability-based arguments were also employed by~\citet{ICML'25:Zhang-mixability}, where a Gaussian comparator distribution was adopted for analytical convenience. However, since Gaussian distributions have full support on $\mathbb{R}^d$, this choice is not directly applicable to constrained domains. To handle general constraints, \citet{ICML'25:Zhang-mixability} proposed learning with a quadratic surrogate loss that extends the domain to $\mathbb{R}^d$, which requires a bounded gradient. Moreover, when updating with the surrogate loss, the mean of the learned distribution $P_{t+1}$ may lie outside the constraint set. To address this issue, they project $P_t$ onto a family of Gaussian mixture models with possibly infinitely many components and bounded component means and variances, which is generally computationally intractable. By contrast, we adopt a Dirichlet comparator to naturally accommodate the simplex constraint and the unbounded loss.
\end{remark}

\section{Tractable Fast Rates for General OXO}
\label{sec:results-bounded-gradient}
This section studies OPS under an additional bounded-gradient assumption and establishes fast-rate dynamic regret bounds for all comparator sequences. In fact, we consider the more general setting of online exp-concave optimization over an arbitrary compact convex domain $\W$ under the following assumptions:

\begin{assumption}
\label{assum:exp-concavity}
For any $t\in[T]$, the loss $\ell_t:\W\to\R$ is $\kappa$-exp-concave over $\W$.
\end{assumption}

\begin{assumption}
\label{assum:bounded-domain}
The domain $\W$ is compact and convex with $\sup_{\u,\w\in\W}\|\u-\w\|_1\leq 2$.\footnote{Without loss of generality, we use the $\ell_1$ norm to match the OPS setting and set the domain diameter bound to $2$. The analysis also extends to $\ell_2$ geometry and general bounded convex domains.}
\end{assumption}

\begin{assumption}
\label{assum:bounded-gradient}
For any $t\in[T]$ and some $G>1$, we have $\sup_{\w\in\W}\|\nabla \ell_t(\w)\|_{\infty}\leq G$.
\end{assumption}

OPS satisfies Assumptions~\ref{assum:exp-concavity} and~\ref{assum:bounded-domain} since Cover's loss is $1$-exp-concave and the simplex has $\ell_1$ diameter at most $2$. The gradient bound can be satisfied under several natural conditions~\citep{98:EG-OPS,ICML'07:ONS-OPS}. For example, one may assume that the ratio of returns is bounded by $G$, that is, $\max_{i\in[d]}x_{t,i}/\min_{i\in[d]}x_{t,i}\leq G$ for every $t\in[T]$, or restrict the portfolio domain to $\W=\{\w\in\Delta_d:\min_{i\in[d]}w_i\geq1/G\}$ (with $G\geq d$). Beyond OPS, other examples of online exp-concave optimization include logistic regression and least-squares regression, to which our method can be applied.

\FloatBarrier
\begin{algorithm}[!t]
    \caption{Follow-the-Leading-History}
    \label{alg:flh-ew}
    \begin{algorithmic}[1]
    \REQUIRE Loss parameter $\eta = \frac{1}{5}\min\{1/(2G),\kappa\}$ and fixed-share parameter $\mu_t = 1/t$.
    \STATE Initialize $P_1 = \mathcal{N}(\mathbf{u}_0, I_d)$ as a Gaussian distribution with mean $\mathbf{u}_0 \in \W$.
    \STATE Initialize a pool with base-learners $\mathcal{H}_1 = \{\mathcal{B}_1\}$, where $\mathcal{B}_1$ is the initial base-learner with the distribution $P_{1,1} = P_1$ and weight $p_{1,1} = 1$.

    \FOR {$t=1,2,\ldots,T$}
      \STATE\label{line:flh-prediction} The learner submits the prediction $\w_t=\E_{\u\sim P_t}[\u]=\sum_{\mathcal{B}_i\in\mathcal{H}_t}p_{t,i}\,\w_{t,i}$, observes the loss $\ell_t$, and constructs the surrogate loss in~\eqref{eq:surrogate-loss}.
    \STATE\label{line:flh-base-update} Update the distribution of each base learner $\mathcal{B}_i\in\mathcal{H}_t$ by
    \begin{equation}
      \label{eq:update-rule} \begin{cases} &P'_{t+1,i}(\u) \propto P_{t,i}(\u)\cdot e^{-\eta\tilde{\ell}_t(\u)},\ \forall \u\in\R^d\\
      &P_{t+1,i} = \argmin_{Q\in\mathscr{W}} \mathrm{KL}\big(Q \,\|\, P'_{t+1,i}\big), \end{cases}
    \end{equation}
    where $\mathscr{W} = \{Q:\E_{\u\mathrel{\sim} Q}[\u]\in\W\}$ is the set of distributions with means in $\W$.
    \STATE\label{line:flh-meta-update} Update the weight for each base learner $\mathcal{B}_i\in\mathcal{H}_t$ by
    \begin{equation}
      \tilde{p}_{t+1,i} \propto p_{t,i}\cdot\E_{\u\sim P_{t,i}}[\exp(- \eta \tilde{\ell}_t(\u))]. \label{eq:meta-alg}
    \end{equation}
    \vspace{-4mm}
      \STATE\label{line:flh-fixed-share} Initialize a new base-learner $\mathcal{B}_{t+1}$ with the distribution $P_{t+1,t+1} = \mathcal{N}(\u_0,I_d)$ and update the weight for existing base-learner by
    \begin{equation}
      \label{eq:fixed-share-FLH} p_{t+1,i} = \begin{cases} & (1-\mu_{t+1})\cdot \tilde{p}_{t+1,i}\, \mbox{ for } \mathcal{B}_i\in\mathcal{H}_t \\
      &{} \mu_{t+1}\, \mbox{for } \,\mathcal{B}_i = \mathcal{B}_{t+1}. \end{cases}
    \end{equation}
    \vspace{-3mm}
    \STATE Update the pool $\mathcal{H}_{t+1} = \mathcal{H}_t \cup \{\mathcal{B}_{t+1}\}$ and obtain $P_{t+1}(\u) = \sum_{\mathcal{B}_i\in\mathcal{H}_{t+1}} p_{t+1,i}\cdot P_{t+1,i}(\u).$
    \ENDFOR
    \end{algorithmic}
    \end{algorithm}

\subsection{Proposed Method}
Our algorithm is summarized in Algorithm~\ref{alg:flh-ew}. Instead of learning directly with the original loss, we employ the following surrogate loss:
\begin{align}
  \tilde{\ell}_t(\w) = \g_t^\top(\w-\w_t) + \eta\,\Vert \w-\w_t\Vert^2_{\g_t\g_t^\top}. \label{eq:surrogate-loss}
\end{align}
Here, $\g_t=\nabla \ell_t(\w_t)$ denotes the gradient of the loss function. For a $\kappa$-exp-concave loss over the domain $\W$, \citet[Lemma 4.3]{book'16:Hazan-OCO} shows that the regret under the original loss can be upper bounded by that under the surrogate loss: $\ell_t(\w_t)-\ell_t(\u_t)\leq \tilde{\ell}_t(\w_t)-\tilde{\ell}_t(\u_t)$ for any $\u_t\in\W$, provided that $\eta\leq\tfrac{1}{4}\min\{(\max_{\w\in\W}|\g_t^\top(\w-\w_t)|)^{-1},\kappa\}$. This choice of a quadratic surrogate loss is standard in online learning for obtaining efficient updates. In our setting, its quadratic form also allows us to work with Gaussian distributions, which simplify the regret analysis.

Our method follows the FLH framework~\citep{ICML'09:Hazan-adaptive}, with multiple base learners started at different times and a meta-learner that aggregates their predictions.
\begin{itemize}
\item \emph{Base-learners}:
 At each iteration $t=i$, we initialize a new base learner $\mathcal{B}_i$ and add it to the expert pool $\mathcal{H}_t$. The base-learner is initialized with a Gaussian distribution $P_{t,i}=\mathcal{N}(\u_0, I_d)$, where $\u_0\in\W$ can be any point in the feasible domain. The distribution of each base learner $\mathcal{B}_i$ is updated using exponential weights with respect to the surrogate loss, followed by the projection in line~\ref{line:flh-base-update} of Algorithm~\ref{alg:flh-ew}.
Since the surrogate loss $\tilde{\ell}_t$ is quadratic, \citet[Theorem~5]{COLT'18:Hoeven-EW} show that the resulting distribution $P_{t,i}=\mathcal{N}(\w_{t,i}, H^{-1}_{t,i})$ remains Gaussian, with its mean and covariance updated via an ONS-type rule. An explicit update formula is provided in~\eqref{eq:update-ONS-EW} in Appendix~\ref{appendix:proof-main-OCO}.
\item \emph{Meta-learner}: We also maintain a meta-learner that assigns a weight $p_{t,i}$ to each base learner $\mathcal{B}_i\in\mathcal{H}_t$ to aggregate their predictions. Specifically, the weights are updated based on their historical performance (line~\ref{line:flh-meta-update}), with a fixed-share step that incorporates the new base learner (line~\ref{line:flh-fixed-share}). The final prediction is obtained by taking the weighted average of the base learners' predictions (line~\ref{line:flh-prediction}).
One slight difference between Algorithm~\ref{alg:flh-ew} and standard FLH is that line~\ref{line:flh-meta-update} updates the weights using the expectation term $\tilde{p}_{t+1,i}\propto p_{t,i}\mathbb{E}_{\u\sim P_{t,i}}[e^{-\eta \tilde{\ell}_t(\u)}]$, whereas the classical update is $\tilde{p}_{t+1,i}\propto p_{t,i} e^{-\eta \tilde{\ell}_t(\w_{t,i})}$. This difference is important for our analysis, as it allows us to align Algorithm~\ref{alg:flh-ew} with exponential-weights updates over distributions. We also note that the update in line~\ref{line:flh-meta-update} admits a closed-form expression, since $P_{t,i}$ is Gaussian and $\tilde{\ell}_t$ is a quadratic function.
\end{itemize}

We have the following guarantee, whose proof is provided in Appendix~\ref{appendix:proof-main-OCO}.

\begin{theorem}
\label{thm:main-OCO}
Under Assumptions~\ref{assum:exp-concavity}, \ref{assum:bounded-domain}, and~\ref{assum:bounded-gradient} and $T\geq 2$, Algorithm~\ref{alg:flh-ew} with $\mu_t = 1/t$ and $\eta = \frac{1}{5}\min\{1/(2G),\kappa\}$ ensures
\begin{align*}
  \DReg_T(\{\u_t\}_{t=1}^T) \leq \O\left(\frac d\eta\Big( \ln (dT) +T^{\frac{1}{3}} P_T^{\frac{2}{3}} (\ln (Td))^{\frac{2}{3}} \Big)\right),
\end{align*}
for any sequence $\u_1,\dots,\u_T\in\W$, where $P_T = \sum_{t=2}^T \|\u_t-\u_{t-1}\|_1$ is the path length and $\eta^{-1}=5\max\{2G,\kappa^{-1}\}$.
\end{theorem}

\begin{remark}[Relation to FLH-ONS]
Our algorithm is a variant of FLH~\citep{ICML'09:Hazan-adaptive} with ONS~\citep{MLJ'07:ONS} as its base learner.
Although FLH-ONS has well-established guarantees on interval regret, previous reductions to nearly optimal dynamic regret for exp-concave losses either require improper learning, which is infeasible in OPS, or are restricted to box-constrained domains~\citep{COLT'21:baby-strong-convex,AISTATS'22:sc-proper}.
Our dynamic regret guarantee holds for arbitrary compact convex domains while remaining proper and computationally tractable.
\end{remark}

\subsection{Two-layer Mixability-based Analysis}
\label{subsection:two-layer-analysis}
This section sketches the proof of Theorem~\ref{thm:main-OCO} using a mixability-based argument. \citet{ICML'25:Zhang-mixability} also used mixability to obtain nearly optimal dynamic regret for OXO, but their method is computationally intractable, as it requires projecting the full Gaussian mixture. We first explain why their analysis does not directly apply to our algorithm and then present the key ideas behind our two-layer analysis.
\vspace{3mm}

\noindent\textbf{Limitations of Previous Attempts.}~\citet{ICML'25:Zhang-mixability} showed that although the surrogate loss is not mixable in general, mixability-based analysis still applies when the mean of each component $P_{t,i}$ in the Gaussian mixture $P_t$ lies in the feasible domain, yielding the variational formulation:
\begin{align}
  \tilde{\ell}_t(\w_t) \leq \mathbb{E}_{\u \sim Q_t}[\tilde{\ell}_t(\u)] + \frac{1}{\eta}\mathrm{KL}(Q_t \,\|\, P_t) - \frac{1}{\eta}\mathrm{KL}(Q_t \,\|\, P'_{t+1}),\label{eq:mixability-whole}
\end{align}
where $P'_{t+1}(\u)\propto P_t(\u)\exp(-\eta\tilde{\ell}_t(\u))$, which remains a Gaussian mixture model. However, the exponential weights update does not guarantee that the component means of $P'_{t+1}$ stay within the decision domain, and thus one cannot set $P_{t+1}=P'_{t+1}$ to telescope the KL terms over $T$ rounds. To overcome this issue, \citet{ICML'25:Zhang-mixability} project $P'_{t+1}$ onto a set $\mathcal{M}$ of Gaussian mixtures with component means in $\W$ and bounded covariances, which ensures a KL--Pythagorean inequality such that $\mathrm{KL}\left(Q \,\|\,P_{t+1} \right) \leq\mathrm{KL}\left(Q \,\|\,P'_{t+1} \right)$ for any $Q\in\mathcal{M}$ with $\E_Q[\u]\in\W$. This allows telescoping and yields
\begin{align}
  \sum_{t=1}^T \tilde{\ell}_t(\w_t) \lesssim \sum_{t=1}^T \E_{Q_t}[\tilde{\ell}_t(\u)] + \frac{1}{\eta}\sum_{t=2}^T\bigl(\KL(Q_t\Vert P_t)-\KL(Q_{t-1}\Vert P_t)\bigr).\label{eq:mixability-whole-telescope}
\end{align}
Here, $\lesssim$ suppresses additive initialization and logarithmic terms. By choosing $Q_t = \mathcal{N}(\u_t,\sigma^2 I_d)$ and selecting $\sigma$ properly, the above inequality leads to the desired fast rate bound. This analysis does not apply to Algorithm~\ref{alg:flh-ew}, since our method projects each component $P'_{t+1,i}$ into the domain separately to gain computational efficiency, and such componentwise projections do not in general satisfy the required KL--Pythagorean guarantee for $P_{t+1}$ and $P'_{t+1}$.

\vspace{3mm}
\noindent\textbf{Our Analysis.} We overcome the projection issue by exploiting the two-layer structure, rather than applying mixability at the level of the aggregated distribution. Specifically, for the comparator distributions $Q_t$ with $\E_{Q_t}[\u]\in\W$ and $\q_t\in\Delta_{\vert \mathcal{H}_t\vert}$, we have
\begin{align}
  \tilde{\ell}_t(\w_t) \mathrel{\leq} \E_{\u\sim Q_t}[\tilde{\ell}_t(\u)] {}&+\frac{1}{\eta}\bigg( \sum_{\mathcal{B}_i\in\mathcal{H}_t} q_{t,i} \mathrm{KL}\left(Q_t \,\|\,P_{t,i} \right)+\mathrm{KL}\left(\q_t \,\|\, \p_t\right) \bigg) \notag \\
  {}&- \frac{1}{\eta}\bigg(\sum_{\mathcal{B}_i\in\mathcal{H}_t} q_{t,i} \mathrm{KL}\left(Q_t \,\|\, {P}_{t+1,i}\right)+ \mathrm{KL}\left(\q_t \,\|\, \tilde{\p}_{t+1}\right)\bigg).\label{eq:two-layer-mixability-form}
\end{align}
The above variational-form bound can be viewed as a two-layer counterpart of~\eqref{eq:mixability-whole}, providing the flexibility to choose the comparators $\q_t$ and $Q_t$ for the meta-learner and the base-learner, respectively. We note that componentwise projections can be performed safely under~\eqref{eq:two-layer-mixability-form}, since the KL divergence is defined in terms of the individual distributions rather than the aggregated one.

The next question is how to choose $\q_t$ and $Q_t$. To make the bound as tight as possible, we choose $\q_t = \argmin_{\q\in\Delta_{\vert \mathcal{H}_t\vert}}\sum_{\mathcal{B}_i\in\mathcal{H}_t} q_{i} \mathrm{KL}\left(Q_t \,\|\,P_{t,i} \right)+\mathrm{KL}\left(\q \,\|\, \p_t\right)$, whose optimal value attains a closed-form formula $V_t(Q_t) = -\ln\left(\E_{\p_t}[\exp(-\mathrm{KL}\left(Q_t \,\|\,P_{t,i} \right) )]\right)$. After a sequence of algebraic manipulations, the bound admits a telescoping structure as in~\eqref{eq:mixability-whole-telescope}.
\begin{align}
  \sum_{t=1}^T \tilde{\ell}_t(\w_t) \lesssim \sum_{t=1}^T \E_{Q_t}[\tilde{\ell}_t(\u)] + \frac{1}{\eta}\sum_{t=2}^T \left(V_t(Q_t) - V_t(Q_{t-1})\right).\label{eq:mixability-whole-telescope-2}
\end{align}
By specifying $Q_t = \mathcal{N}(\u_t,\sigma^2 I_d)$ as a Gaussian distribution, we can further show that
\begin{align*}
  \sum_{t=1}^T\E_{Q_t}[\tilde{\ell}_t(\u)] &\leq \sum_{t=1}^T\tilde{\ell}_t(\u_t) + \eta dG^2T\sigma^2,\\
  \sum_{t=2}^T (V_t(Q_t) - V_{t}(Q_{t-1})) &\lesssim\frac{dP_T\log(dT)}{\sigma}+dP_T\sigma^2.
\end{align*}
The first inequality follows from the quadratic formulation of the surrogate loss function, and the second inequality is obtained by showing that the gradient of $V_t(Q_t)$ with respect to $\u_t$ can be upper bounded by $\O(d\log(dT)/\sigma+d\sigma^2)$. We can obtain the desired bound by setting $\sigma = \widetilde{\Theta}\big(P_T^{1/3}T^{-1/3}\big)$ since $P_T \leq 2 T$.

\section{Conclusion}
\label{sec:conclusion}
We establish nearly matching upper and lower bounds for dynamic regret in non-stationary OPS under the standard $L_1$-path length. Universal Dynamic Portfolio achieves the upper bound and finer guarantees based on the JS-path length and the JS$^q$-path length, adapting to the spatial and temporal structure of comparator sequences without parameter tuning or bounded-gradient assumptions. Under bounded gradients, we also develop an efficient proper method with fast dynamic regret for OPS and general online exp-concave optimization over compact convex domains with tractable metric projections. Future work includes developing more computationally efficient algorithms that retain the refined OPS guarantees, and reducing the number of active base learners.

\section*{Acknowledgments and AI-use Statement}
KJ and YZ were supported in part by a Singapore National Research Foundation AI Visiting Professorship award and NSF TRIPODS II DMS-2023166.

The authors developed the main research ideas, technical results, and theoretical developments in this work before January 2026. During subsequent extensions of the work from July to September 2026, GPT-6 Astra was used to assist in exploring proof strategies for the lower bounds, particularly the multidimensional constructions, and in developing the illustrative example for the JS$^q$-path length results. GPT-5.6 and GPT-6 were also used during manuscript preparation for language editing, grammar checking, and polishing. The authors have carefully checked all AI-assisted mathematical arguments and take full responsibility for the content of the paper.

\bibliography{references}

\newpage
\appendix

\section{Properties of the Dirichlet Distribution}
\label{app:dirichlet-properties}

In this section, we present some useful properties of the Dirichlet distribution that will be used in our analysis.

\begin{definition}[Dirichlet Distribution]
A random vector $\w \in \Delta_d$ is said to follow a Dirichlet distribution with parameter vector $\alphab = (\alpha_1,\ldots,\alpha_d) \in \R_+^d$, denoted by $\w \sim \mathrm{Dir}(\alphab)$, if its probability density function is
\[
  p(\w \mid \alphab) = \frac{\Gamma\!\left(\sum_{i=1}^d \alpha_i\right)}{\prod_{i=1}^d \Gamma(\alpha_i)} \prod_{i=1}^d w_i^{\alpha_i - 1}, \qquad \w \in \Delta_d .
\]
where $\Gamma(\alpha) = \int_0^\infty t^{\alpha - 1} e^{-t} \mathrm{d}t$ is the Gamma function.
\end{definition}

\begin{property}[Properties of Dirichlet Distribution]
\label{prop:dirichlet-update}
Let $\w \sim \mathrm{Dir}(\alphab)$ and $\alpha_0 = \sum_{i=1}^d \alpha_i$. Then,
\begin{itemize}
\item The mean of the random vector is given by $\E_{\w \sim \mathrm{Dir}(\alphab)}[\w] = \alphab/ \alpha_0$.
\item The covariance matrix is $\mathrm{Cov}[\w] = \frac{1}{\alpha_0^2(\alpha_0 + 1)}(\alpha_0\mathrm{diag}(\alphab) - \alphab\alphab^\top)$.
\item Dirichlet distribution belongs to the exponential family with natural parameter $\boldsymbol{\eta} = \alphab - \bm{1}$ and log-partition function $A(\boldsymbol{\eta}) = \sum_{i=1}^d \ln \Gamma(\alpha_i)-\ln \Gamma\!\left(\alpha_0\right) $.
\item The KL divergence between two Dirichlet distributions $\mathrm{Dir}(\alphab)$ and $\mathrm{Dir}(\boldsymbol{\beta})$ is given by
\begin{align*}
  \KL({\mathrm{Dir}(\alphab)}\Vert{\mathrm{Dir}(\boldsymbol{\beta})}) = \ln \frac{\Gamma(\alpha_0)}{\Gamma(\beta_0)} - \sum_{i=1}^d \ln \frac{\Gamma(\alpha_i)}{\Gamma(\beta_i)} + \sum_{i=1}^d (\alpha_i - \beta_i) \left( \psi(\alpha_i) - \psi(\alpha_0) \right),
\end{align*}
where $\psi(\alpha) = \frac{d}{d\alpha} \ln \Gamma(\alpha)$ is the digamma function.
\item The differential Shannon entropy of $\mathrm{Dir}(\alphab)$ is
\begin{align*}
  H(\mathrm{Dir}(\alphab)) = \ln \left(\frac{\prod_{i=1}^d \Gamma(\alpha_i)}{\Gamma(\alpha_0)} \right) + (\alpha_0 - d) \psi(\alpha_0) - \sum_{i=1}^d (\alpha_i - 1) \psi(\alpha_i).
\end{align*}
\end{itemize}
\end{property}

\begin{property}[Properties of Gamma Function]
\label{prop:Gamma-function}
Let $\Gamma(\alpha) = \int_0^\infty t^{\alpha - 1} e^{-t} \mathrm{d}t$ be the Gamma function and $\psi(\alpha) = \frac{\mathrm{d}}{\mathrm{d}\alpha} \ln \Gamma(\alpha)$ be the digamma function. Then, for any $\alpha > 0$, we have
\begin{itemize}
    \item $\Gamma(\alpha + 1) = \alpha \Gamma(\alpha)$ and $\psi(\alpha+1) = \psi(\alpha) + 1/\alpha $.
    \item $\ln (\alpha) - \frac{1}{\alpha}\leq \psi(\alpha) \leq \ln (\alpha) - \frac{1}{2\alpha} $.
\end{itemize}
\end{property}

\begin{lemma}
\label{lemma:KL-upper}
Let $P = \mathrm{Dir}(\mathbf{1}+\gamma \u)$, $Q = \mathrm{Dir}(\mathbf{1}+\gamma \e_i)$ and $P_1 = \mathrm{Dir}(\mathbf{1})$. Then, $\mathrm{KL}(P\,\|\, P_1) \leq \mathrm{KL}\left(Q \,\|\,P_1 \right)$ for any $\u \in \Delta_d$, $i\in[d]$ and $\gamma > 0$.
\end{lemma}

\begin{proof}[Proof of Lemma~\ref{lemma:KL-upper}]
By definition of the KL divergence between two Dirichlet distributions, we have
\begin{align*}
  \mathrm{KL}\left(P \,\|\, P_1\right) ={}& \ln \frac{\Gamma(d+\gamma)}{\Gamma(d)} -\gamma \psi(d+\gamma) +\sum_{i=1}^d \left( \gamma u_i \psi(1+\gamma u_i) - \ln \Gamma(1+\gamma u_i)\right).
\end{align*}

Let $G(\u) = \sum_{i=1}^d g(u_i)$ where $g(u) = \gamma u \psi(1+\gamma u) -  \ln \Gamma(1+\gamma u)$. According to Lemma~\ref{lemma:trigamma_function}, one can show that $G(\u)$ is a convex function over the simplex. Then, we have $G(\u) = G(\sum_{j=1}^d u_j\e_j) \leq \sum_{j=1}^d u_j G(\e_j)\leq G(\e_i)$ for any $i\in[d]$, where the last inequality holds because $G(\u)$ is invariant under permutations of the coordinates. Then, we complete the proof by showing
\begin{align*}
  \mathrm{KL}\left(P \,\|\,P_1 \right) &=\ln \frac{\Gamma(d+\gamma)}{\Gamma(d)} -\gamma \psi(d+\gamma) +G(\u) \\
  &\leq \ln \frac{\Gamma(d+\gamma)}{\Gamma(d)} -\gamma \psi(d+\gamma) +G(\e_i) = \mathrm{KL}\left(Q \,\|\,P_1 \right).
\end{align*}
\end{proof}

\section{Omitted Proofs for Section~\ref{subsec:path-length}}
\label{app:path-length-proofs}
\subsection{Proof of Theorem~\ref{thm:lower-bound}}
\label{app:path-length-lower-bound}

\begin{proof}[Proof of Theorem~\ref{thm:lower-bound}]
We focus on the minimax regret for the $d$-asset OPS problem:
\begin{align*}
  \mathcal{W}_T(\U_C) = \inf_{f_{1:T}} \sup_{\x_1,\dots,\x_T\in \R_+^d} \sup_{\u_{1:T}\in\U_C} \left( \sum_{t=1}^T\ell_t(\w_t) - \sum_{t=1}^T\ell_t(\u_t) \right),
\end{align*}
where the infimum is taken over the algorithm's online prediction rules $f_{1:T}$, and $\w_t$ is the algorithm's prediction based only on past observations $\x_1,\dots,\x_{t-1}$. To make the dependence explicit, we will also
write $\w_t=f_t(\x_{1:t-1})$, where
$\x_{1:t-1}=(\x_1,\dots,\x_{t-1})$ and
$f_t:\R_+^{d\times(t-1)}\to\Delta_d$ is a measurable online prediction rule determined by the algorithm. For any sequence
$\u_{1:T}=(\u_1,\dots,\u_T)$, we define
\[
  \U_C = \left\{ \u_{1:T}\in\Delta_d^T \,\middle|\, \sum_{t=2}^T \Vert \u_t-\u_{t-1}\Vert_1 \leq C \right\}
\]
as the set of all comparator sequences whose $L_1$-path length is at
most $C$.

\paragraph{Reduction to SPA.}
For any $C\in[0,T]$, we derive a lower bound on the value
$\mathcal{W}_T(\U_C)$ by reducing the OPS problem to the sequential
probability assignment (SPA) problem
\citep[Chapter~9.1]{book/Cambridge/cesa2006prediction}. Specifically, we
consider the Kelly market vector setting, where
$\x_t\in\{\mathbf{e}_1,\dots,\mathbf{e}_d\}$ for all $t\in[T]$, and
$\mathbf{e}_i$ denotes the $i$-th standard basis vector. We encode the
market outcomes by defining $y_t=i$ when $\x_t=\mathbf{e}_i$, and
introduce the multi-class log loss
\[
  \ell_{\log}(\w,y) = -\log[\w]_y, \qquad \w\in\Delta_d,\quad y\in[d],
\]
where $[\w]_y$ denotes the $y$-th coordinate of $\w$. It is
straightforward to verify that
$\ell_t(\w)=\ell_{\log}(\w,y_t)$ when
$\x_t=\mathbf{e}_{y_t}$. Consequently, letting $\Y=[d]$, we obtain
\begin{align*}
  \mathcal{W}_T(\U_C) \geq \mathcal{V}_T(\U_C) &:= \inf_{f_{1:T}} \sup_{y_1,\dots,y_T\in\Y} \sup_{\u_{1:T}\in\U_C} \left( \sum_{t=1}^T\ell_{\log}(\w_t,y_t) - \sum_{t=1}^T\ell_{\log}(\u_t,y_t) \right).
\end{align*}

Throughout the proof, the dimension $d$ is fixed and the asymptotic
notation is with respect to $T$.

\paragraph{Hard Example Construction.}
To lower bound the dynamic term in the minimax regret
$\mathcal{V}_T(\U_C)$, we first restrict attention to the main regime
$C\in[2(d-1)/T,T/(2(d-1))]$. In this regime, we construct an environment in which the time
horizon is partitioned into consecutive intervals of length
\[
  L=\left\lceil \frac{d-1}{\epsilon}\right\rceil \mbox{ with } \epsilon=\sqrt{\frac{(d-1)C}{2T}}.
\]

This choice ensures
$(d-1)/T\leq\epsilon\leq1/2$ and hence $L\leq T$. The corner cases
will be handled at the end of the proof. We further let
$K=\lfloor T/L\rfloor$ denote the number of intervals. The first
$K-1$ intervals each have length $L$, while the final interval has
length $T-(K-1)L\in[L,2L-1]$. We will also use
$\mathcal{I}_k=[s_k,e_k]$ to denote the $k$-th interval for
$k\in[K]$, with start time $s_k$ and end time $e_k$.

For each block $k\in[K]$, we consider a collection of environments
indexed by a binary vector
\[
  \mathbf{I}_k = [I_{k,1},\dots,I_{k,d-1}] \in \{0,1\}^{d-1}.
\]
The coordinates of $\mathbf{I}_k$ are drawn independently and uniformly
at random, i.e.,
$I_{k,j}\sim\mathrm{Bern}(1/2)$ independently for every
$j\in[d-1]$ and $k\in[K]$. There are in total $2^{d-1}$ possible
realizations of the environment index $\mathbf{I}_k$ for each block
$k$. On each interval $\mathcal{I}_k$, the labels are generated
according to a static probability distribution
$\widetilde{\u}_t\in\Delta_d$ that depends on the index
$\mathbf{I}_k$. Specifically, for any $t\in\mathcal{I}_k$, we set
$\widetilde{\u}_t$ as
\begin{align*}
  [\widetilde{\u}_t]_j = \frac{\epsilon}{d-1}I_{k,j} \ \mbox{ for all }j\in[d-1] ~~ \mbox{and} ~~ [\widetilde{\u}_t]_d = 1- \frac{\epsilon}{d-1} \sum_{j=1}^{d-1}I_{k,j}.
\end{align*}

In the above, the first $d-1$ dimensions are associated with the
environment index coordinate-wise, while the $d$-th dimension is a
common asset. Since $\epsilon\leq1/2$, we have
$[\widetilde{\u}_t]_d\geq1-\epsilon\geq1/2$, and hence the above
vector is a valid portfolio in $\Delta_d$. The label is then generated
according to $y_t\sim\operatorname{Cat}(\widetilde{\u}_t)$, where
$\operatorname{Cat}(\widetilde{\u}_t)$ denotes the categorical
distribution with probability mass function $\widetilde{\u}_t$.

For any realization of the environment indices
$\{\mathbf{I}_k\}_{k=1}^K$, the cumulative path length of the comparator
sequence $\widetilde{\u}_{1:T}$ is bounded by
\begin{align*}
  \sum_{t=2}^T \Vert \widetilde{\u}_t-\widetilde{\u}_{t-1} \Vert_1 ={}& \sum_{k=2}^K \left( \frac{\epsilon}{d-1} \sum_{j=1}^{d-1} \vert I_{k,j}-I_{k-1,j}\vert + \frac{\epsilon}{d-1} \left| \sum_{j=1}^{d-1}I_{k,j} - \sum_{j=1}^{d-1}I_{k-1,j} \right| \right) \\
  \leq{}& \sum_{k=2}^K \frac{2\epsilon}{d-1} \sum_{j=1}^{d-1} \vert I_{k,j}-I_{k-1,j}\vert \leq 2\epsilon(K-1) \leq \frac{2T\epsilon^2}{d-1} = C,
\end{align*}
where the last inequality follows from
$K\leq T/L\leq T\epsilon/(d-1)$. The above displayed inequality
indicates that $\widetilde{\u}_{1:T}\in\U_C$. Then, the minimax regret $\mathcal{V}_T(\U_C)$ can be further lower
bounded by
\begin{align*}
  \mathcal{V}_T(\U_C) \geq{}& \inf_{f_{1:T}} \E_{\mathbf{I}_{1:K}} \left[ \E_{y_{1:T}} \left[ \sum_{t=1}^T \ell_{\log}(\w_t,y_t) - \sum_{t=1}^T \ell_{\log}(\widetilde{\u}_t,y_t) \,\middle|\, \mathbf{I}_{1:K} \right] \right] \\
  ={}& \inf_{f_{1:T}} \E_{\mathbf{I}_{1:K}} \left[ \E_{y_{1:T}} \left[ \log \left( \prod_{t=1}^T \frac{ [\widetilde{\u}_t]_{y_t} }{ [\w_t]_{y_t} } \right) \,\middle|\, \mathbf{I}_{1:K} \right] \right] \\
  ={}& \inf_{f_{1:T}} \E_{\mathbf{I}_{1:K}} \left[ \E_{y_{1:T}} \left[ \log \left( \prod_{t=1}^T \frac{ [\widetilde{\u}_t]_{y_t} }{ [f_t(y_{1:t-1})]_{y_t} } \right) \,\middle|\, \mathbf{I}_{1:K} \right] \right] \\
  ={}& \inf_{f_{1:T}} \E_{\mathbf{I}_{1:K}} \left[ \E_{y_{1:T}} \left[ \log \left( \frac{ \widetilde{\q}(y_{1:T}\mid\mathbf{I}_{1:K}) }{ \p(y_{1:T}) } \right) \,\middle|\, \mathbf{I}_{1:K} \right] \right].
\end{align*}
The first inequality follows from the fact that
$\widetilde{\u}_{1:T}\in\U_C$ for every realization of the environment
indices $\mathbf{I}_{1:K}$. The first equality follows from the
definition of the log loss. Under the restriction
$\x_t=\mathbf{e}_{y_t}$, we use
$f_t(y_{1:t-1})$ as shorthand for
$f_t(\mathbf{e}_{y_1},\dots,\mathbf{e}_{y_{t-1}})$. In the last equality, we define the joint mass function
\[
  \p(y_{1:T}) = \prod_{t=1}^T [f_t(y_{1:t-1})]_{y_t},
\]
where $f_t:\Y^{t-1}\to\Delta_d$ is the online prediction rule at time $t$. We note that the online prediction rule $f_t$ is deterministic given
$y_1,\dots,y_{t-1}$, and hence $\p$ is fully determined by the online
algorithm and is independent of the randomness in
$\mathbf{I}_{1:K}$. Similarly, conditional on the environment indices
$\mathbf{I}_{1:K}$, we define the joint probability mass function of
the label sequence $y_{1:T}$ as
\[
  \widetilde{\q}(y_{1:T}\mid\mathbf{I}_{1:K}) = \Pr\left( Y_{1:T}=y_{1:T} \,\middle|\, \mathbf{I}_{1:K} \right) = \prod_{t=1}^T [\widetilde{\u}_t]_{y_t}.
\]
With a slight abuse of notation, we use
$\widetilde{\q}_{Y_{1:T}\mid\mathbf{I}_{1:K}}$ to denote the
corresponding conditional distribution and $\p_{Y_{1:T}}$ to denote
the distribution induced by $\p(y_{1:T})$. One can check that
$\p(y_{1:T})$ is a valid probability mass function and, for every
realization of $\mathbf{I}_{1:K}$,
$\widetilde{\q}(y_{1:T}\mid\mathbf{I}_{1:K})$ is a valid conditional
probability mass function.
In particular,
\[
  \sum_{y_{1:T}\in\Y^T} \p(y_{1:T}) = \sum_{y_{1:T-1}\in\Y^{T-1}} \p(y_{1:T-1}) = \cdots = \sum_{y_1\in\Y} \p(y_1) = 1.
\]

Furthermore, since
$y_t\sim\operatorname{Cat}(\widetilde{\u}_t)$ conditional on
$\mathbf{I}_1,\dots,\mathbf{I}_K$, the conditional distribution of the
sequence $y_{1:T}$ given $\mathbf{I}_{1:K}$ has probability mass
function $\widetilde{\q}(y_{1:T}\mid\mathbf{I}_{1:K})$. Then, we have
\[
  \E_{\mathbf{I}_{1:K}} \left[ \E_{y_{1:T}} \left[ \log \left( \frac{ \widetilde{\q}(y_{1:T}\mid\mathbf{I}_{1:K}) }{ \p(y_{1:T}) } \right) \,\middle|\, \mathbf{I}_{1:K} \right] \right] = \E_{\mathbf{I}_{1:K}} \left[ \KL( \widetilde{\q}_{Y_{1:T}\mid\mathbf{I}_{1:K}} \Vert \p_{Y_{1:T}} ) \right].
\]
Then, the minimax regret can be further lower bounded by
\begin{align*}
  \mathcal{V}_T(\U_C) \geq{}& \inf_{f_{1:T}} \E_{\mathbf{I}_{1:K}} \left[ \KL( \widetilde{\q}_{Y_{1:T}\mid\mathbf{I}_{1:K}} \Vert \p_{Y_{1:T}} ) \right] \\
  ={}& \E_{\mathbf{I}_{1:K}} \left[ \KL( \widetilde{\q}_{Y_{1:T}\mid\mathbf{I}_{1:K}} \Vert \bar{\q}_{Y_{1:T}} ) \right] + \inf_{f_{1:T}} \E_{\mathbf{I}_{1:K}} \left[ \sum_{y_{1:T}\in\Y^T} \widetilde{\q}(y_{1:T}\mid\mathbf{I}_{1:K}) \log \frac{ \bar{\q}(y_{1:T}) }{ \p(y_{1:T}) } \right] \\
  ={}& \E_{\mathbf{I}_{1:K}} \left[ \KL( \widetilde{\q}_{Y_{1:T}\mid\mathbf{I}_{1:K}} \Vert \bar{\q}_{Y_{1:T}} ) \right] + \inf_{f_{1:T}} \sum_{y_{1:T}\in\Y^T} \E_{\mathbf{I}_{1:K}} \left[ \widetilde{\q}(y_{1:T}\mid\mathbf{I}_{1:K}) \right] \log \frac{ \bar{\q}(y_{1:T}) }{ \p(y_{1:T}) } \\
  ={}& \E_{\mathbf{I}_{1:K}} \left[ \KL( \widetilde{\q}_{Y_{1:T}\mid\mathbf{I}_{1:K}} \Vert \bar{\q}_{Y_{1:T}} ) \right] + \inf_{f_{1:T}} \KL( \bar{\q}_{Y_{1:T}} \Vert \p_{Y_{1:T}} ) \\
  \geq{}& \E_{\mathbf{I}_{1:K}} \left[ \KL( \widetilde{\q}_{Y_{1:T}\mid\mathbf{I}_{1:K}} \Vert \bar{\q}_{Y_{1:T}} ) \right],
\end{align*}
where
$\bar{\q}(y_{1:T})
=\Pr(Y_{1:T}=y_{1:T})
=\E_{\mathbf{I}_{1:K}}
[\widetilde{\q}(y_{1:T}\mid\mathbf{I}_{1:K})]$
is the marginal probability mass function of the label sequence obtained
by averaging over the environment indices. With the same abuse of
notation, we use $\bar{\q}_{Y_{1:T}}$ to denote the corresponding
marginal distribution. The penultimate equality holds because both
probability mass functions $\bar{\q}$ and $\p$ are independent of
$\mathbf{I}_{1:K}$.

Then, let
$y_{\mathcal{I}_k}=\{y_t\}_{t\in\mathcal{I}_k}$ be the sequence of
labels on interval $\mathcal{I}_k$, and define its conditional
probability mass function by
\[
  \widetilde{\q}_k(y_{\mathcal{I}_k}\mid\mathbf{I}_k) = \Pr\left( Y_{\mathcal{I}_k}=y_{\mathcal{I}_k} \,\middle|\, \mathbf{I}_k \right) = \prod_{t\in\mathcal{I}_k} [\widetilde{\u}_t]_{y_t}.
\]
With a slight abuse of notation, we use
$\widetilde{\q}_{Y_{\mathcal{I}_k}\mid\mathbf{I}_k}$ to denote the
corresponding conditional distribution.
Conditional on
$\mathbf{I}_{1:K}=(\mathbf{I}_1,\dots,\mathbf{I}_K)$, the label
sequences on different intervals are independent. Therefore, the
conditional probability mass function of the full label sequence
factorizes as
$\widetilde{\q}(y_{1:T}\mid\mathbf{I}_{1:K})
=\prod_{k=1}^K\widetilde{\q}_k(y_{\mathcal{I}_k}\mid\mathbf{I}_k)$.
Using the independence of the environment indices
$\mathbf{I}_1,\dots,\mathbf{I}_K$, we further have
\begin{align*}
  \bar{\q}(y_{1:T})&=\E_{\mathbf{I}_{1:K}}\left[\widetilde{\q}(y_{1:T}\mid\mathbf{I}_{1:K})\right]=\E_{\mathbf{I}_{1:K}} \left[\prod_{k=1}^K \widetilde{\q}_k(y_{\mathcal{I}_k}\mid\mathbf{I}_k) \right]\\
  &= \prod_{k=1}^K \E_{\mathbf{I}_k} \left[ \widetilde{\q}_k(y_{\mathcal{I}_k}\mid\mathbf{I}_k) \right]= \prod_{k=1}^K \bar{\q}_k(y_{\mathcal{I}_k}),
\end{align*}
where we define the marginal probability mass function on block $k$ as
$\bar{\q}_k(y_{\mathcal{I}_k})
=\Pr(Y_{\mathcal{I}_k}=y_{\mathcal{I}_k})
=\E_{\mathbf{I}_k}
[\widetilde{\q}_k(y_{\mathcal{I}_k}\mid\mathbf{I}_k)]$ and use
$\bar{\q}_{Y_{\mathcal{I}_k}}$ to denote the corresponding marginal
distribution. The minimax regret can be further bounded by
\begin{align}
  \mathcal{V}_T(\U_C) &\geq \E_{\mathbf{I}_{1:K}} \left[ \sum_{k=1}^K \KL( \widetilde{\q}_{Y_{\mathcal{I}_k}\mid\mathbf{I}_k} \Vert \bar{\q}_{Y_{\mathcal{I}_k}} ) \right] = \sum_{k=1}^K \E_{\mathbf{I}_k} \left[ \KL( \widetilde{\q}_{Y_{\mathcal{I}_k}\mid\mathbf{I}_k} \Vert \bar{\q}_{Y_{\mathcal{I}_k}} ) \right],\label{eq:lower-bound-V}
\end{align}
where the blockwise decomposition follows from the additivity of the KL
divergence for product distributions, and the equality holds
because the conditional distribution
$\widetilde{\q}_{Y_{\mathcal{I}_k}\mid\mathbf{I}_k}$ depends only on
$\mathbf{I}_k$.

We next analyze the KL divergence contributed by each block. For every
$j\in[d-1]$, define
\[
  Z_{k,j} = \indicator\left\{ \sum_{t\in\mathcal{I}_k} \indicator\{Y_t=j\} \geq1 \right\},
\]
and write
$\Zb_k=(Z_{k,1},\dots,Z_{k,d-1})$. Thus, $Z_{k,j}$ indicates whether
label $j$ is observed at least once on block $k$. Since $\Zb_k$ is a
deterministic function of $Y_{\mathcal{I}_k}$, we have
\begin{align}
  \E_{\mathbf{I}_k} \left[ \KL( \widetilde{\q}_{Y_{\mathcal{I}_k}\mid\mathbf{I}_k} \Vert \bar{\q}_{Y_{\mathcal{I}_k}} ) \right] &= \mathrm{I} \left( \mathbf{I}_k; Y_{\mathcal{I}_k} \right)\geq \mathrm{I} \left( \mathbf{I}_k; \Zb_k \right) \geq \sum_{j=1}^{d-1} \mathrm{I} \left( I_{k,j}; Z_{k,j} \right),\label{eq:lower-bound-mutual}
\end{align}
where $\mathrm{I}(\cdot;\cdot)$ denotes mutual information. The first
equality follows from
$\mathrm{I}(X;Y)=\E_X[\KL(P_{Y\mid X}\Vert P_Y)]$ by taking
$X=\mathbf{I}_k$ and $Y=Y_{\mathcal{I}_k}$. The first inequality
follows from the data-processing inequality. The last inequality
follows from the independence of the coordinates of $\mathbf{I}_k$
and the fact that conditioning reduces entropy, since
$H(\mathbf{I}_k\mid\Zb_k)
\leq\sum_{j=1}^{d-1}H(I_{k,j}\mid\Zb_k)
\leq\sum_{j=1}^{d-1}H(I_{k,j}\mid Z_{k,j})$.

We next analyze the information contributed by each coordinate.
Conditional on $I_{k,j}=0$, label $j$ has zero probability, and hence
$Z_{k,j}=0$ almost surely. Conditional on $I_{k,j}=1$, label $j$ is
generated with probability $\epsilon/(d-1)$ at every round. We have
\[
  p_k := \Pr( Z_{k,j}=1 \mid I_{k,j}=1 ) = 1- \left( 1-\frac{\epsilon}{d-1} \right)^{\vert \Ical_k\vert} \geq 1-e^{-1},
\]
where the last inequality holds because
$\vert \Ical_k \vert\geq L\geq(d-1)/\epsilon$.

Since $I_{k,j}\sim\mathrm{Bern}(1/2)$, we have
$\Pr(Z_{k,j}=1)=p_k/2$ and $H(I_{k,j})=\log2$. Moreover, conditional
on $Z_{k,j}=1$, the environment coordinate $I_{k,j}$ must be equal to
one, and hence $H(I_{k,j}\mid Z_{k,j}=1)=0$. Since
$H(I_{k,j}\mid Z_{k,j}=0)\leq\log2$, we obtain
\begin{align*}
  \mathrm{I} \left( I_{k,j}; Z_{k,j} \right) &= H(I_{k,j}) - H(I_{k,j}\mid Z_{k,j}) \\
  &\geq \log2 - \Pr(Z_{k,j}=0)\log2 \\
  &= \Pr(Z_{k,j}=1)\log2 = \frac{p_k}{2}\log2 \geq \frac{1-e^{-1}}{2}\log2.
\end{align*}

Combining the above inequalities, the KL divergence contributed by each
block satisfies
\[
  \E_{\mathbf{I}_k} \left[ \KL( \widetilde{\q}_{Y_{\mathcal{I}_k}\mid\mathbf{I}_k} \Vert \bar{\q}_{Y_{\mathcal{I}_k}} ) \right] \geq \frac{(d-1)(1-e^{-1})}{2}\log2.
\]
Then, we can conclude that
\[
  \mathcal{V}_T(\U_C) \geq \frac{(d-1)(1-e^{-1})}{2} K\log2.
\]

It remains to lower bound the number of intervals $K$. Since
$\epsilon\leq1/2$, we have
$L=\lceil(d-1)/\epsilon\rceil\leq2(d-1)/\epsilon$. Moreover,
$\epsilon\geq(d-1)/T$ implies $L\leq T$. Therefore,
\[
  K = \left\lfloor \frac{T}{L} \right\rfloor \geq \frac{T}{2L} \geq \frac{T\epsilon}{4(d-1)}.
\]
It follows that, in the main regime,
\begin{align}
  \mathcal{V}_T(\U_C) &\geq \frac{1-e^{-1}}{8} T\epsilon\log2 = \frac{(1-e^{-1})\log2}{8\sqrt{2}}\sqrt{(d-1)TC}. \label{eq:dynamic-main-regime}
\end{align}

\paragraph{Hanlding Corner Cases.} It remains to handle the values of $C$ outside the main regime.

If $C<2(d-1)/T$, we use the same construction with
$\epsilon_0=(d-1)/T$. In this case, $L=T$ and $K=1$, so that the
comparator sequence is constant and has zero path length. The same
one-block analysis gives
\[
  \mathcal{V}_T(\U_C) \geq \frac{(d-1)(1-e^{-1})}{2}\log2 = \Omega(d).
\]
Moreover, since $C<2(d-1)/T$, we have
$\sqrt{dTC}<\sqrt{2d(d-1)}=O(d)$. Therefore, the above one-block lower
bound already dominates the desired dynamic term. Combining this
one-block estimate with~\eqref{eq:dynamic-main-regime}, we obtain the
required $\Omega(\sqrt{dTC})$ lower bound for every
$C\leq T/(2(d-1))$.

If $C>T/(2(d-1))$, we apply the preceding construction with the smaller
budget $C_0=T/(2(d-1))$. Since $\U_{C_0}\subseteq\U_C$, the same lower
bound continues to hold. Moreover, applying
\eqref{eq:dynamic-main-regime} with $C_0$ gives
\[
  \mathcal{V}_T(\U_C) \geq \mathcal{V}_T(\U_{C_0}) = \Omega\left( \sqrt{(d-1)TC_0} \right) = \Omega(T).
\]
Combining these cases, we obtain
\begin{align}
  \mathcal{V}_T(\U_C) \geq \Omega\left( \min\left\{ T,\sqrt{dTC} \right\} \right). \label{eq:dynamic-lower-bound}
\end{align}

On the other hand, since any fixed comparator is also contained by $\mathcal{U}_C$, the standard lower bound for sequential probability assignment\citep[Chapter~9.1]{book/Cambridge/cesa2006prediction} shows that $V_T (\mathcal{U}_C) = \Omega(d\log T)$ for $C= 0$,
which leads to the lower bound of $ \mathcal{V}_T(\U_C)
  \geq
  \Omega\left(
    \max\left\{
      d\log T,\,
      \min\left\{
        T,\sqrt{dTC}
      \right\}
    \right\}
  \right),$
which completes the proof.
\end{proof}

\subsection{Proof of Lemma~\ref{lem:upper-bound-minimax}}
\label{app:black-box-reduction}
\begin{proof}[Proof of Lemma~\ref{lem:upper-bound-minimax}]
For any interval $\Ical = [s,e]\subseteq[T]$, let $\w_{\Ical} = \arg\min_{\w\in\Delta_d} \sum_{t\in\Ical} \ell_t(\w)$ be the optimal fixed prediction on interval $\Ical$. Then, for any sequence $\{\u_t\}_{t\in\Ical}$, we have
\begin{align}
  \sum_{t\in\Ical} \ell_t(\w_{\Ical}) - \sum_{t\in\Ical} \ell_t(\u_t) = \sum_{t\in\Ical} \log\!\left(\frac{\u_t^\top \x_t}{\w_{\Ical}^\top\x_t}\right) = \log\!\left(\frac{\prod_{t\in\Ical}\u_t^\top\x_t}{\prod_{t\in\Ical}\w_{\Ical}^\top\x_t}\right). \label{eq:proof-reduction-1}
\end{align}
Let $z^{\max}_{i} = \max_{t\in\Ical} \{u_{t,i}\}$, and define $\z_{\max} = [z_1^{\max},\dots, z_d^{\max}]^\top$ together with its normalized version $\bar{\u}= \z_{\max}/\norm{\z_{\max}}_1$.
Since $\x_t \in\R^d_+$, we can further upper bound~\eqref{eq:proof-reduction-1} by
\begin{align}
  \log\left(\frac{\prod_{t\in\Ical}\u_t^\top\x_t}{\prod_{t\in\Ical}\w_{\Ical}^\top\x_t}\right) &\leq \log\left(\frac{\prod_{t\in\Ical}\z_{\max}^\top\x_t}{\prod_{t\in\Ical}\w_{\Ical}^\top\x_t}\right) \notag\\
  &= \log \left(\frac{\prod_{t\in\Ical}\bar{\u}^\top\x_t}{\prod_{t\in\Ical}\w_{\Ical}^\top\x_t}\right) + \vert\Ical\vert\log \Vert \z_{\max}\Vert_1 \notag\\
  &\leq \vert\Ical\vert \log \Vert \z_{\max}\Vert_1, \label{eq:proof-reduction-2}
\end{align}
where the last inequality holds because $\w_{\Ical}$ minimizes the cumulative loss over the interval $\Ical$. For each coordinate of $\z_{\max}$, we have $z_i^{\max} \leq u_{s,i} + \sum_{t = s+1}^e [u_{t,i} - u_{t-1,i}]_+$, which implies
\begin{align}
  \Vert \z_{\max}\Vert_1 &\leq \sum_{i=1}^d u_{s,i} + \sum_{t=s+1}^e\sum_{i=1}^d [u_{t,i} - u_{t-1,i}]_+ \notag\\
  &=1+\frac{1}{2}\sum_{t=s+1}^e\Vert\u_t-\u_{t-1}\Vert_1 =1+\frac{1}{2}P_{\Ical}, \label{eq:proof-reduction-3}
\end{align}
where $P_{\Ical}=\sum_{t=s+1}^e\Vert\u_t-\u_{t-1}\Vert_1$. The equality uses the simplex identity that $\sum_{i=1}^d[v_i]_+$ equals $\frac12\lVert v\rVert_1$ whenever $\sum_{i=1}^d v_i=0$, applied to
$v=\u_t-\u_{t-1}$.

Combining~\eqref{eq:proof-reduction-1},~\eqref{eq:proof-reduction-2}, and~\eqref{eq:proof-reduction-3} with condition~\eqref{eq:strongly-adaptive}, we obtain that for any interval $\Ical\subseteq[T]$,
\begin{align}
  \sum_{t\in\Ical} \ell_t(\w_t) - \sum_{t\in\Ical} \ell_t(\u_t) &\leq B(T) + \vert \Ical\vert \log \left(1+\frac{P_{\Ical}}{2}\right) \leq B(T) + \frac{\vert\Ical\vert P_{\Ical}}{2}. \label{eq:proof-reduction-dynamic}
\end{align}

We now construct a partition into consecutive maximal blocks. Set $s_1=1$. Given the start $s_m$, let $e_m$ be the largest index $e\in\{s_m,\ldots,T\}$ for which the product $(e-s_m+1)P_{[s_m,e]}$ is at most $B(T)$. Such an index always exists because $P_{[s_m,s_m]}=0$. If $e_m<T$, set $s_{m+1}=e_m+1$ and continue; otherwise stop. This produces a partition $\{\Ical_m=[s_m,e_m]\}_{m=1}^M$ of $[T]$, and every block satisfies $\vert\Ical_m\vert P_{\Ical_m}\leq B(T)$. Therefore,~\eqref{eq:proof-reduction-dynamic} gives
\begin{align}
  \sum_{t=1}^T \ell_t(\w_t) - \sum_{t=1}^T \ell_t(\u_t) &=\sum_{m=1}^M \left(\sum_{t\in\Ical_m} \ell_t(\w_t) - \sum_{t\in\Ical_m} \ell_t(\u_t)\right) \leq \frac{3}{2}M B(T). \label{eq:reduction-final}
\end{align}

If $M=1$,~\eqref{eq:reduction-final} already gives the required bound $\DReg_T(\{\u_t\}_{t=1}^T)\leq\frac32 B(T)$. Henceforth, assume $M>1$. It remains to bound $M$. Write $L_m=\vert\Ical_m\vert$. For every nonfinal block, maximality implies that the product $(L_m+1)P_{[s_m,e_m+1]}$ exceeds $B(T)$ for $m=1,\ldots,M-1$. The increments in $P_{[s_m,e_m+1]}$ are indexed by $t=s_m+1,\ldots,e_m+1$. Those in the next extended block begin at $t=s_{m+1}+1=e_m+2$, so the sets of increments are disjoint. Consequently, the sum $\sum_{m=1}^{M-1}P_{[s_m,e_m+1]}$ is at most $P_T$. Moreover, since every $L_m\geq1$, $\sum_{m=1}^{M-1}(L_m+1)$ is at most $2\sum_{m=1}^{M-1}L_m$, which is at most $2T$. Taking square roots in the nonfinal-block inequality and summing gives the first line below. Cauchy--Schwarz and the two preceding sum bounds then yield
\begin{align*}
  (M-1)\sqrt{B(T)} &< \sum_{m=1}^{M-1} \sqrt{(L_m+1)P_{[s_m,e_m+1]}}\\
  &\leq \sqrt{\left(\sum_{m=1}^{M-1}(L_m+1)\right) \left(\sum_{m=1}^{M-1}P_{[s_m,e_m+1]}\right)}\\
  &\leq \sqrt{2TP_T}.
\end{align*} Thus $M\leq1+\sqrt{2TP_T/B(T)}$. Substitution into~\eqref{eq:reduction-final} shows that $\DReg_T(\{\u_t\}_{t=1}^T)$ is at most $\frac{3}{2}B(T)+\frac{3}{\sqrt{2}}\sqrt{B(T)TP_T}$, which is $\mathcal{O}\!\left(B(T)+\sqrt{B(T)TP_T}\right)$, this completes the proof.
\end{proof}

\section{Omitted Proofs for Section~\ref{subsec:jeffreys-path}}
\label{app:jeffreys-proofs}
This section presents the omitted proofs of Theorems~\ref{thm:main} and~\ref{thm:lower-bound-kl}. We first establish several auxiliary lemmas and then present the main proofs.
\subsection{Useful Lemmas}
\label{app:jeffreys-auxiliary}

\begingroup
\newtheorem*{mixabilityrestatement}{Lemma~\ref{lem:mixability}}
\begin{mixabilityrestatement}
For any distribution $Q_t$ with $\mathrm{supp}(Q_t) \subseteq \mathrm{supp}(P_t)$, it holds that
\begin{align*}
  \ell_t(\w_t) = \mathbb{E}_{\u \sim Q_t}[\ell_t(\u)] + \mathrm{KL}(Q_t \,\|\, P_t) - \mathrm{KL}(Q_t \,\|\, \tilde{P}_{t+1}),
\end{align*}
where $\tilde{P}_{t+1}(\u) \propto P_t(\u)\exp(-\ell_t(\u))$ for all $\u \in \mathrm{supp}(P_t)$.
\end{mixabilityrestatement}

\begin{proof}[Proof of Lemma~\ref{lem:mixability}]
This lemma holds by an exact identity for distributions defined over $\Delta_d$. Let $Z_t = \E_{\w\sim P_t}[\exp(-\ell_t(\w))]$, we have
\begin{align*}
  \mathrm{KL}\left(Q_t \,\|\,P_t \right) - \mathrm{KL}\left(Q_t \,\|\, \tilde{P}_{t+1} \right) ={}& \E_{\u\sim Q_t}\left[\ln\left(\frac{\tilde{P}_{t+1}(\u)}{P_t(\u)}\right)\right]\\
  ={}& \E_{\u\sim Q_t}\left[\ln\left(\frac{\exp(-\ell_t(\u))}{Z_t }\right)\right]\\
  ={}& -\E_{\u\sim Q_t}\left[\ell_t(\u)\right] - \ln Z_t\\
  ={}& -\E_{\u\sim Q_t}\left[\ell_t(\u)\right] + \ell_t(\w_t),
\end{align*}
where the second equality follows from $\tilde{P}_{t+1}(\u) = P_t(\u)\exp(-\ell_t(\u))/Z_t$.
\end{proof}
\endgroup

\begin{lemma}
\label{lemma:upper-bound}
Let $P_t$ be the distribution updated by Algorithm~\ref{alg:hedge-oco} with $\mu_t=1/t$. Then, $\ln P_t(\u)\leq (d-1)\ln (T+1) + d(\ln d +1)$ for all $t\in[T]$ and $\u\in\Delta_d$.
\end{lemma}

\begin{proof}[Proof of Lemma~\ref{lemma:upper-bound}] According to the EW and fixed-share update steps in Algorithm~\ref{alg:hedge-oco}, for any $\w\in\Delta_d$, we have
    \begin{align*}
      P_{t+1}(\w) = (1-\mu_{t+1}) \tilde{P}_{t+1}(\w) + \mu_{t+1} P_1(\w) = \frac{(1-\mu_{t+1}) e^{-\ell_t(\w)}}{\E_{\u\sim P_t}[e^{-\ell_t(\u)}] }P_t(\w) + \mu_{t+1} P_1(\w).
    \end{align*}
Unrolling this recursion shows that $P_{t+1}$ is a convex combination of posterior distributions initialized at the different restart times. Specifically, define
\begin{align}
  P_{t+1,i}(\w) = \frac{1}{W_{t+1,i}} P_1(\w) e^{-\sum_{s=i}^t \ell_s(\w)} \mbox{ for all } i\leq t+1,\;\label{eq:Pt-i-def}
\end{align}
where $W_{t+1,i} = \E_{\u\sim P_1}[e^{-\sum_{s=i}^t \ell_s(\u)}]$ is the normalization factor and the empty sum is zero when $i=t+1$. An induction on $t$ gives nonnegative weights $\alpha_{t+1,1},\ldots,\alpha_{t+1,t+1}$ that sum to one and satisfy
\begin{align}
  P_{t+1}(\w) = \sum_{i=1}^{t+1} \alpha_{t+1,i} P_{t+1,i}(\w) \leq \max_{i\in[t+1]} P_{t+1,i}(\w). \label{eq:Pt-weighted-sum}
\end{align}

For $i = t+1$, we have $P_{t+1,t+1}(\w) = P_1(\w)$ for all $\w\in\Delta_d$ and $\ln P_{t+1,t+1}(\w) = \ln \Gamma(d)$ since $P_1$ is a uniform distribution on $\Delta_d$. Then, to prove the lemma it is sufficient to upper bound $\ln P_{t+1,i}(\w)$ for all $i\in[t]$ and $\w\in\Delta_d$. According to the definition~\eqref{eq:Pt-i-def}, the distribution $P_{t+1,i}(\w)\propto P_{t,i}(\w)\cdot \exp(-\ell_t(\w))$ for all $i\leq t$. Then, we have
\begin{align}
  \ln P_{t+1,i}(\w) ={}& \ln P_{t,i}(\w) -\ln (\E_{\u\sim P_{t,i}}[e^{-\ell_t(\u)}]) - \ell_t(\w)\notag\\
  ={}&\ln P_1(\w) - \sum_{s=i}^t \ln (\E_{\u\sim P_{s,i}}[e^{-\ell_s(\u)}]) - \sum_{s=i}^t\ell_s(\w)\notag\\
  = {}& \ln P_1(\w) + \sum_{s=i}^t\E_{\u\sim Q}[\ell_s(\u)] - \sum_{s=i}^t\ell_s(\w)\notag\\
  &\quad + \mathrm{KL}\left(Q \,\|\, P_1\right) - \mathrm{KL}\left( Q \,\|\, P_{t+1,i} \right). \label{eq:Pt-i-inequality}
\end{align}
for any distribution $Q$ over $\Delta_d$ whose support is contained in that of $P_1$. The second equality is due to the recursive definition of $P_{t,i}$. Let
\[
  \w_*^{t+1,i} = \argmax_{\w\in\Delta_d} \ln P_{t+1,i}(\w), \qquad Q_*^{t+1,i} = \mbox{Dir}(\mathbf{1}+T \w_*^{t+1,i}).
\]
Then, the same arguments used to upper bound term~(a) in the proof of Theorem~\ref{thm:main} yield
\[
  \sum_{s=i}^t \E_{\u\sim Q_*^{t+1,i}}[\ell_s(\u)] - \sum_{s=i}^t\ell_s(\w_*^{t+1,i}) \leq T \ln \left(1+ \frac{d}{T}\right)\leq d.
\]
Besides, based on Lemma~\ref{lemma:KL-upper}, we have $\mbox{KL}(Q_*^{t+1,i} \,\|\, P_{1}) \leq (d-1) \ln (1+T)$. Plugging the above inequalities into~\eqref{eq:Pt-i-inequality} with the fact that $\ln P_1(\w) = \ln \Gamma(d)$ yields
\begin{align*}
  \ln P_{t+1,i}(\w) \leq{}& \ln \Gamma(d) + d + (d-1)\ln (1+T) \\
  \leq{}& (d-1)\ln (1+T) + d(\ln d +1)
\end{align*}
for any $\w\in\Delta_d$ and $i\in[t]$. Finally,~\eqref{eq:Pt-weighted-sum} gives
\[
  P_{t+1}(\w) \leq \max_{i\in[t+1]} P_{t+1,i}(\w),
\]
which completes the proof.
\end{proof}

\begin{lemma}
\label{lemma:dirichlet-js}
Let $\gamma>0$ and $Q_{\u}=\mathrm{Dir}(\mathbf{1}+\gamma\u)$ for $\u\in\Delta_d$. Then, for any $\u,\mathbf{v}\in\Delta_d$,
\[
  \int_{\Delta_d}\lvert Q_{\u}(\w)-Q_{\mathbf{v}}(\w)\rvert\mathrm{d}\w \leq\min\left\{2,2\sqrt{2\gamma}\,\mathrm{JS}(\u,\mathbf{v})\right\}.
\]
\end{lemma}

\begin{proof}[Proof of Lemma~\ref{lemma:dirichlet-js}]
For $\u,\mathbf{v}\in\Delta_d$, let $\mathbf{m}=(\u+\mathbf{v})/2$. We have
\begin{align}
  \lVert Q_{\u}-Q_{\mathbf{v}}\rVert_1 &\leq\sqrt{2\mathrm{KL}(Q_{\mathbf{m}}\Vert Q_{\u})} +\sqrt{2\mathrm{KL}(Q_{\mathbf{m}}\Vert Q_{\mathbf{v}})}\notag\\
  &=\sqrt{2B_F(\gamma\u\Vert\gamma\mathbf{m})} +\sqrt{2B_F(\gamma\mathbf{v}\Vert\gamma\mathbf{m})}, \label{eq:dirichlet-midpoint-pinsker}
\end{align}
where the first line follows from the triangle and Pinsker inequalities, and the second follows from \citet[Lemma~10]{COLT'18:Hoeven-EW} since Dirichlet distributions form an exponential family. Here $B_F(\mathbf{a}\Vert\mathbf{b})
=F(\mathbf{a})-F(\mathbf{b})-\langle\nabla F(\mathbf{b}),\mathbf{a}-\mathbf{b}\rangle$ is the Bregman divergence induced by the Dirichlet log-partition function

\[
  F(\z)=\sum_{i=1}^d\ln\Gamma(1+z_i)-\ln\Gamma\left(d+\sum_{i=1}^d z_i\right).
\]

We next bound $B_F$ by the KL divergence. Let $\psi$ denote the digamma function.
Since the trigamma function satisfies $\psi'(1+x)\leq 1/x$ for $x>0$, the function
$h(x)=x\ln x-\ln\Gamma(1+x)$ is convex on $(0,\infty)$.
The first-order convexity inequality
$h(a)\geq h(b)+h'(b)(a-b)$ indicates that
\begin{equation}
  \label{eq:scalar-gamma-bound} \ln\frac{\Gamma(1+a)}{\Gamma(1+b)}-(a-b)\psi(1+b) \leq a\ln\frac{a}{b}-a+b
\end{equation}
for any $a,b>0$. With the convention $0\ln0=0$, both sides are continuous in $a$ at zero, so the inequality also holds for $a=0$.
Since $m_i=0$ implies $u_i=v_i=0$, such coordinates contribute zero to $B_F(\gamma\u\Vert\gamma\mathbf{m})$. We therefore only need to consider coordinates with $m_i>0$.
Applying~\eqref{eq:scalar-gamma-bound} coordinatewise, we obtain
\begin{align*}
  B_F(\gamma\u\Vert\gamma\mathbf{m}) &=\sum_{i:m_i>0} \left[ \ln\frac{\Gamma(1+\gamma u_i)}{\Gamma(1+\gamma m_i)} -\gamma(u_i-m_i)\psi(1+\gamma m_i) \right]\\
  &\leq\gamma\sum_{i:m_i>0} \left[u_i\ln\frac{u_i}{m_i}-u_i+m_i\right] =\gamma\mathrm{KL}(\u\Vert\mathbf{m}),
\end{align*}
where the first and last equalities use $\sum_{i=1}^d u_i=\sum_{i=1}^d m_i=1$. The same argument applies to $\mathbf{v}$.

Substituting the above displayed inequality into~\eqref{eq:dirichlet-midpoint-pinsker} and applying the Cauchy-Schwarz inequality yields
\begin{align*}
  \lVert Q_{\u}-Q_{\mathbf{v}}\rVert_1 &\leq \sqrt{2\gamma\mathrm{KL}(\u\Vert\mathbf{m})} +\sqrt{2\gamma\mathrm{KL}(\mathbf{v}\Vert\mathbf{m})}\\
  &\leq 2\sqrt{\gamma\big( \mathrm{KL}(\u\Vert\mathbf{m}) +\mathrm{KL}(\mathbf{v}\Vert\mathbf{m}) \big)}\\
  &= 2\sqrt{2\gamma}\,\mathrm{JS}(\u,\mathbf{v}).
\end{align*}
Finally, we complete the proof by noting that
$\lVert Q_{\u}-Q_{\mathbf{v}}\rVert_1\leq2$
for any $\u,\mathbf{v}\in\Delta_d$.
\end{proof}

\subsection{Proof of Theorem~\ref{thm:main}}
\label{app:jeffreys-upper-bound}

\begin{proof}[Proof of Theorem~\ref{thm:main}]
Let $P_1 = \mbox{Dir}(\mathbf{1})$. According to the definition of the loss function $\ell_t(\w) = -\ln(\w^\top\x_t)$ and the prediction $\w_t = \E_{\w\sim P_t}[\w]$, we have
\begin{align*}
  \ell_t(\w_t) ={}& - \ln \left(\E_{\w\sim P_t}[\exp(-\ell_t(\w))]\right)\\
  = {}& \E_{\w\sim Q_t}[\ell_t(\w)] + \mathrm{KL}\left(Q_t \,\|\, P_t\right) - \mathrm{KL}\left(Q_t \,\|\, \tilde{P}_{t+1}\right)\\
  ={}&\E_{\w\sim Q_t}[\ell_t(\w)] + \mathrm{KL}\left(Q_t \,\|\, P_t \right) -\mathrm{KL}\left(Q_{t} \,\|\, P_{t+1} \right) + \E_{\w\sim Q_t}\left[\ln \frac{\tilde P_{t+1}(\w)}{{P}_{t+1}(\w)}\right]\\
  \leq{}& \E_{\w\sim Q_t}[\ell_t(\w)] + \mathrm{KL}\left(Q_t \,\|\, P_t \right) -\mathrm{KL}\left(Q_{t} \,\|\, P_{t+1} \right) + \ln \frac{1}{1-\mu_{t+1}}
\end{align*}
where $\tilde{P}_{t+1}(\w)\propto P_t(\w) \exp(-\ell_t(\w))$ for all $\w\in\Delta_d$. In the above, the last inequality holds since  $P_{t+1}(\w) = (1-\mu_{t+1}) \tilde{P}_{t+1}(\w) + \mu_{t+1} P_1(\w)$ for all $\w\in\Delta_d$ and $t\geq 1$. We can further upper bound the term $\ln (1/(1-\mu_{t+1}))\leq 1/t$ by the setting $\mu_{t+1} = 1/(1+t)$. Taking the sum from $t=1$ to $T$ rounds and rearranging the terms yields
\begin{align*}
  \sum_{t=1}^T \ell_t(\w_t) \leq{}& \underbrace{\sum_{t=1}^T \E_{\w\sim Q_t}[\ell_t(\w)]}_{\texttt{term~(a)}}\\
  &+ \underbrace{\sum_{t=2}^T \left(\mathrm{KL}\left(Q_t \,\|\, P_t \right) -\mathrm{KL}\left(Q_{t-1} \,\|\, P_{t} \right) \right) }_{\texttt{term~(b)}}\\
  &+ \underbrace{\mathrm{KL}\left(Q_1 \,\|\, P_1\right)}_{\texttt{term~(c)}} + (1+ \ln T).
\end{align*}
The next step is to choose $Q_t$ to make the bound tight. Here, we use the Dirichlet distribution $Q_t = \mbox{Dir}(\mathbf{1}+\gamma \u_t)$, where $\gamma>0$ is a parameter to be tuned later and $\u_t$ is the time-varying comparator sequence. Next, we bound the three terms separately.

\paragraph{Bounding term~(a).} For each round $t \in [T]$, the return vector $\x_t$ has at least one nonzero entry. For simplicity, we assume that the first $\bar{d}_t$ entries of $\x_t$ are nonzero. Besides, for each dimension $i \in [d]$, let $\{Z_{t,i}\}_{i=1}^d$ be independent random variables with $Z_{t,i} \sim \mathrm{Gamma}(1 + \gamma u_{t,i}, 1)$, and define $Z_{t,0} = \sum_{i=1}^d Z_{t,i} \sim \mathrm{Gamma}(d + \gamma, 1)$. It is known that the random vector $(Z_{t,1}/Z_{t,0}, \dots, Z_{t,d}/Z_{t,0})$ follows the Dirichlet distribution $\mathrm{Dir}(\mathbf{1} + \gamma \u_t)$. We have
\begin{align}
  \texttt{term~(a)} ={}& \sum_{t=1}^T \E_{Z_{t,1},\dots,Z_{t,d}}\left[-\ln\left(\sum_{i=1}^d \frac{x_{t,i} Z_{t,i}}{Z_{t,0}} \right)\right]\notag\\
  ={}&\sum_{t=1}^T \E_{Z_{t,0}}[\ln(Z_{t,0})]+ \sum_{t=1}^T \E_{Z_{t,1},\dots,Z_{t,d}}\left[ - \ln\left(\sum_{i=1}^d x_{t,i} Z_{t,i} \right)\right]\notag\\
  ={}& \sum_{t=1}^T \psi(d+ \gamma) + \sum_{t=1}^T \E_{Z_{t,1},\dots,Z_{t,d}}\left[ - \ln\left(\sum_{i=1}^{\bar{d}_t} x_{t,i} Z_{t,i} \right)\right]\notag\\
  \leq{}& \sum_{t=1}^T \psi(d+ \gamma) + \sum_{t=1}^T \E_{Z_{t,1},\dots,Z_{t,d}}\left[ - \sum_{i=1}^{\bar{d}_t} p_{t,i} \ln (x_{t,i}Z_{t,i}/p_{t,i})\right]\notag\\
  ={}& \sum_{t=1}^T \psi(d+ \gamma) -\sum_{t=1}^T \left( \sum_{i=1}^{\bar{d}_t} p_{t,i}\big(\ln x_{t,i}+\psi(1+\gamma u_{t,i})\big) - \sum_{i=1}^{\bar{d}_t} p_{t,i} \ln (p_{t,i}) \right).\label{eq:term-a-1}
\end{align}
In the third line, we use the identity $\E[\ln(Z_{t,0})] = \psi(d+\gamma)$ for $Z_{t,0} \sim \mathrm{Gamma}(d+\gamma,1)$, and note that $\sum_{i=1}^d x_{t,i} Z_{t,i} = \sum_{i=1}^{\bar d_t} x_{t,i} Z_{t,i}$ since the remaining entries are zero. For the fourth line, since $x_{t,i} Z_{t,i} > 0$ for all $i \in [\bar{d}_t]$ and $t \in [T]$, we apply the inequality $\ln\left(\sum_{i=1}^{\bar{d}_t} a_i\right) \ge \sum_{i=1}^{\bar{d}_t} p_i \ln\left(\frac{a_i}{p_i}\right),$ which holds for all $a_i > 0$ and any probability vector $\mathbf p = (p_1,\dots,p_{\bar{d}_t}) \in \mathrm{ri}(\Delta_{\bar{d}_t})$, where $\mathrm{ri}(\Delta_{\bar{d}_t})$ denotes the relative interior of the simplex $\Delta_{\bar{d}_t}$.

Then, we can tune the probability vector $\p_t$ to make the bound~\eqref{eq:term-a-1} tight. The goal is to solve the optimization problem
\[
  V_t^* = \max_{\p_t\in\mathrm{ri}(\Delta_{\bar{d}_t})} \sum_{i=1}^{\bar{d}_t} p_{t,i}\big(\ln x_{t,i}+\psi(1+\gamma u_{t,i})\big) - \sum_{i=1}^{\bar{d}_t} p_{t,i} \ln (p_{t,i}),
\]
which has the closed-form solution by $V_t^* = \ln\left(\sum_{i=1}^{\bar{d}_t} x_{t,i}\exp(\psi(1+\gamma u_{t,i}))\right)$ achieved at $p_{t,i}^* \propto x_{t,i}\exp(\psi(1+\gamma u_{t,i}))$. Plugging the optimal solution back to~\eqref{eq:term-a-1} yields
\begin{align*}
  \texttt{term~(a)} \leq{}& \sum_{t=1}^T \psi(d+ \gamma) - \sum_{t=1}^T \ln\left(\sum_{i=1}^{\bar{d}_t} x_{t,i}\exp(\psi(1+\gamma u_{t,i}))\right)\\
  \leq{}& \sum_{t=1}^T \psi(d+\gamma) - \sum_{t=1}^T \ln \left(\gamma\sum_{i=1}^{\bar{d}_t} x_{t,i} u_{t,i}\right)\\
  ={}&\sum_{t=1}^T \ell_t(\u_t) + T(\psi(d+\gamma) - \ln \gamma)\\
  \leq{}& \sum_{t=1}^T \ell_t(\u_t) + T \ln \left(1+ \frac{d}{\gamma}\right)
\end{align*}
where the second line holds since $\psi(1+x) \geq \ln x$ for all $x>0$. The second inequality also holds for the case $u_{t,i}=0$ since $\exp(\psi(1))>0$. The third line follows from the fact that only the first $\bar d_t$ entries of $\x_t$ are nonzero. The last inequality holds since $\psi(x)\leq  \ln (x)$ for all $x>0$.
~\\
~\\
{\textbf{Bounding term~(b):}} As for term~(b), we have
\begin{align*}
  \texttt{term~(b)} ={}& \sum_{t=2}^T \left(\mathrm{KL}\left(Q_t \,\|\, P_t \right) -\mathrm{KL}\left(Q_{t-1} \,\|\, P_{t} \right) \right)\\
  ={}& \sum_{t=2}^T \E_{\w\sim Q_t}\left[\ln \frac{Q_t(\w)}{P_t(\w)}\right] - \E_{\w\sim Q_{t-1}}\left[\ln \frac{Q_{t-1}(\w)}{P_{t}(\w)}\right]\\
  ={}& \sum_{t=2}^T \left(H(Q_{t-1}) - H(Q_t)\right) + \sum_{t=2}^T \E_{\w\sim Q_{t-1}}\left[\ln P_t(\w)\right] - \E_{\w\sim Q_{t}}\left[\ln P_{t}(\w)\right]\\
  ={}&\underbrace{\sum_{t=2}^T \int_{\w\in\Delta_d} (Q_{t-1}(\w) - Q_t(\w))\ln P_t(\w)\mathrm{d} \w }_{\texttt{term~(b-1)}}+ \underbrace{H(Q_1) - H(Q_T)}_{\texttt{term~(b-2)}}
\end{align*}
We then proceed to bound term~(b-1) and term~(b-2) separately. For notational simplicity, we denote by $C_{T,d} = (d-1)\ln (T+1) + d(\ln d +1)$. We can upper bound term~(b-1) by
\begin{align}
  \texttt{term~(b-1)} \leq{}& \sum_{t=2}^T \int_{\w\in\Delta_d} \vert Q_{t-1}(\w) - Q_t(\w)\vert \cdot \vert \ln P_t(\w)\vert \mathrm{d} \w\notag\\
  {}&\leq C_{T,d}\sum_{t=2}^T \int_{\Delta_d} |Q_{t-1}(\w) - Q_t(\w)|\mathrm{d} \w, \label{eq:term-b-1a}
\end{align}
The last inequality holds because $\lvert\ln P_t(\w)\rvert\leq C_{T,d}$ for all $t\in[T]$ and $\w\in\Delta_d$. Indeed, the upper bound $\ln P_t(\w)\leq C_{T,d}$ follows from Lemma~\ref{lemma:upper-bound}. For the lower bound, the fixed-share update~\eqref{eq:step2-FS} and $P_1(\w)=\Gamma(d)$ give $P_t(\w)\geq\mu_tP_1(\w)=\Gamma(d)/t\geq1/T$, so $\ln P_t(\w)\geq-\ln T\geq-C_{T,d}$.

Then, we can further bound the total variation between $Q_t$ and $Q_{t-1}$ by
\begin{align}
  \sum_{t=2}^T \int_{\Delta_d} |Q_{t-1}(\w)-Q_t(\w)| \mathrm{d}\w \leq{}& 2\sqrt{2\gamma} \sum_{t=2}^T \operatorname{JS}(\u_t,\u_{t-1}) = 2\sqrt{2\gamma}\,P_{T}^{\mathrm{JS}}, \label{eq:term-b1-b}
\end{align}
where the inequality holds by Lemma~\ref{lemma:dirichlet-js}, and the equality follows from the definition of $P_T^{\mathrm{JS}}$.
Combining~\eqref{eq:term-b-1a} and~\eqref{eq:term-b1-b}, we arrive at
\begin{align*}
  \texttt{term~(b-1)} \leq\, 2\sqrt{2}C_{T,d}\sqrt{\gamma}\, P_{T}^{\mathrm{JS}} .
\end{align*}
As for term~(b-2), we have
\begin{align*}
  \texttt{term~(b-2)} ={}& H(Q_1) - H(Q_T) = \mathrm{KL}\left(Q_T \,\|\, P_1 \right) - \mathrm{KL}\left(Q_1 \,\|\, P_1 \right)
\end{align*}
since $P_1 = \mbox{Dir}(\mathbf{1})$ is a uniform distribution over the simplex. Finally, we arrive at
\begin{align*}
  \texttt{term~(b)} \leq{}& 2\sqrt{2}C_{T,d}\sqrt{\gamma}\, P_{T}^{\mathrm{JS}} +\mathrm{KL}\left(Q_T\,\Vert\,P_1\right) -\mathrm{KL}\left(Q_1\,\Vert\,P_1\right).
\end{align*}
~\\
{\textbf{Combining All}.}
Combining the bounds on terms (a) and (b), we get
\begin{align}
  \sum_{t=1}^T \ell_t(\w_t) \leq{}& \sum_{t=1}^T \ell_t(\u_t) + T\ln\left(1+\frac{d}{\gamma}\right)+ 2\sqrt{2}C_{T,d}\sqrt{\gamma}\, P_{T}^{\mathrm{JS}}\notag\\
  &+ \mathrm{KL}\left(Q_T\,\Vert\,P_1\right) +(1+\ln T). \label{eq:final-regret-bound-1}
\end{align}
We can further upper bound the KL divergence term by
\begin{align*}
  \mathrm{KL}\left(Q_T \,\|\, P_1 \right) \leq {}&\mathrm{KL}\left( \mathrm{Dir}(\mathbf{1} + \gamma \mathbf{e}_i)\,\|\, P_1\right) \\
  ={}& \ln \frac{\Gamma(d+\gamma)}{\Gamma(d)\Gamma(1+\gamma)} + \gamma (\psi(1+\gamma) - \psi(d+\gamma)) \\
  \leq{}& (d-1) \ln (1+\gamma).
\end{align*}
where $\e_i$ is a one-hot vector with $1$ at the $i$-th position and $0$ elsewhere. The first inequality is due to Lemma~\ref{lemma:KL-upper}. For the last inequality, we use the identities ${\Gamma(d+\gamma)}/{\Gamma(1+\gamma)}
=\prod_{i=1}^{d-1}(\gamma+i)$ and $\Gamma(d)=\prod_{i=1}^{d-1}i$,
together with $\psi(1+\gamma)-\psi(d+\gamma)\leq0$. Then, the regret bound~\eqref{eq:final-regret-bound-1} becomes
\begin{align*}
  {}& \sum_{t=1}^T \ell_t(\w_t) - \sum_{t=1}^T \ell_t(\u_t) \\
  \leq{}& T\ln\left(1+\frac{d}{\gamma}\right) + 2\sqrt{2}C_{T,d}\sqrt{\gamma}\, P_{T}^{\mathrm{JS}} +(d-1)\ln(1+\gamma) +(1+\ln T),
\end{align*}
where $C_{T,d} = (d-1)\ln (T+1) + d(\ln d +1) = \O\big(d \ln (Td)\big)$. Since $\gamma$ appears only in the analysis, we can choose it to optimize the regret bound. We consider the following two cases, depending on the value of $P_{T}^{\mathrm{JS}}$:
\begin{itemize}
  \item
  \textbf{Case 1
  \big($P_{T}^{\mathrm{JS}}\leq 1/T$
  \big).}
  We set $\gamma=T$, which yields an
  $\mathcal{O}(d\ln(dT))$ regret bound.

  \item
  \textbf{Case 2
  \big($P_{T}^{\mathrm{JS}}>1/T$\big).}
  We choose $\gamma
  =T^{\frac{2}{3}}
      \big(P_{T}^{\mathrm{JS}}\big)^{-\frac{2}{3}}
      (\ln(dT))^{-\frac{2}{3}}.$
  Substituting this choice of $\gamma$ into the bound gives
  \[
    \sum_{t=1}^T \ell_t(\w_t) - \sum_{t=1}^T \ell_t(\u_t) = \mathcal{O}\left( d\Big( T^{\frac{1}{3}} \big(P_{T}^{\mathrm{JS}}\big)^{\frac{2}{3}} (\ln(dT))^{\frac{2}{3}} + \ln(dT) \Big) \right).
  \]
\end{itemize}
We have completed the proof by combining the two cases.
\end{proof}

\subsection{Proof of Theorem~\ref{thm:lower-bound-kl}}
\label{app:jeffreys-lower-bound}

\begin{proof}[Proof of Theorem~\ref{thm:lower-bound-kl}]
The proof follows the same overall argument as in the proof of Theorem~\ref{thm:lower-bound}. The main difference lies in the hard example construction. The lower bound in Theorem~\ref{thm:lower-bound} relies on a hard instance with comparator sequences near the boundary of the simplex. Here, we instead construct a hard instance showing that the $T^{1/3}(P_T^{\mathrm{JS}})^{2/3}$ dependence is optimal up to logarithmic factors even for uniformly interior comparator sequences.

Our goal remains to establish a lower bound on the minimax regret.
\begin{align*}
  \mathcal{W}_T(\U^{\mathrm{JS}}_C) = \inf_{f_{1:T}} \sup_{\x_1,\dots,\x_T\in \R_+^d} \sup_{\u_{1:T}\in\U_C^{\mathrm{JS}}} \left( \sum_{t=1}^T\ell_t(\w_t) - \sum_{t=1}^T\ell_t(\u_t) \right),
\end{align*}
where $f_{1:T}$ denotes the sequence of online prediction rules and
$\w_t=f_t(\x_{1:t-1})$ is the algorithm's prediction based only on past
observations. Here,
\[
  \U^{\mathrm{JS}}_{C} \mathrel{=} \left\{ \u_{1:T}\mathrel{\in}\Delta_d^T \,\middle|\, \sum_{t=2}^T \operatorname{JS}\left(\u_t,\u_{t-1}\right) \mathrel{\leq} C \right\}
\]
is the set of all comparator sequences whose JS-path length
is at most $C$. The same reduction as in the proof of
Theorem~\ref{thm:lower-bound} gives
\begin{align*}
  \mathcal{W}_T(\U_C^{\mathrm{JS}}) \geq \mathcal{V}_T(\U_C^{\mathrm{JS}}) &:={} \inf_{f_{1:T}} \sup_{y_1,\dots,y_T\in\Y} \sup_{\u_{1:T}\in\U_C^{\mathrm{JS}}} \left( \sum_{t=1}^T\ell_{\log}(\w_t,y_t) - \sum_{t=1}^T\ell_{\log}(\u_t,y_t) \right),
\end{align*}
where $\Y=[d]$ is an alphabet of size $d$ and
$\ell_{\log}(\w,y)=-\log[\w]_y$ for any $\w\in\Delta_d$.

\paragraph{Hard Example Construction.} We first restrict
attention to the main regime $C\in
\left[
  \sqrt{\frac{2(d-1)}{T}},
  \frac{T}{4\sqrt{2}(d-1)}
\right]$, which is non-empty by the assumption $T>4(d-1)$.
As in the proof of Theorem~\ref{thm:lower-bound}, we partition the
horizon into $K=\lfloor T/L\rfloor$ blocks, where the first $K-1$
blocks have length
\[
  L=\left\lceil \frac{d-1}{\epsilon^2}\right\rceil \mbox{ with } \epsilon= \left( \frac{(d-1)C}{\sqrt{2}T} \right)^{\frac{1}{3}}.
\]
Under the main regime, $\sqrt{(d-1)/T}\leq\epsilon\leq1/2$, so $L\leq T$.
The final block has length $T-(K-1)L\in[L,2L-1]$. We denote the
$k$-th block by $\mathcal{I}_k=[s_k,e_k]$.

For each block $k\in[K]$, let
$\mathbf{I}_k=[I_{k,1},\dots,I_{k,d-1}]\in\{0,1\}^{d-1}$, where
$I_{k,j}\sim\mathrm{Bern}(1/2)$ independently for every $j\in[d-1]$
and $k\in[K]$. For any $t\in\mathcal{I}_k$, we define
a probability vector $\widetilde{\u}_t$ in the interior of
$\Delta_d$ by
\begin{align*}
  [\widetilde{\u}_t]_j &= \frac{1}{2(d-1)}+\frac{\epsilon}{d-1} \left(I_{k,j}-\frac12\right) \quad\mbox{for all }j\in[d-1],\;t\in\mathcal{I}_k, \\
  [\widetilde{\u}_t]_d &= \frac12- \frac{\epsilon}{d-1} \sum_{j=1}^{d-1} \left(I_{k,j}-\frac12\right) \quad\mbox{for all }t\in\mathcal{I}_k.
\end{align*}
Indeed, since $\epsilon\leq1/2$, we have
$[\widetilde{\u}_t]_j\in\big[\frac{1}{4(d-1)},\frac{3}{4(d-1)}\big]$ for all $j\in[d-1]$ and $[\widetilde{\u}_t]_d\in[1/4,3/4]$.
We then generate $Y_t\sim\operatorname{Cat}(\widetilde{\u}_t)$
independently conditional on $\mathbf{I}_{1:K}$.

We next show that the comparator sequence
$\widetilde{\u}_{1:T}$ has JS-path length at most
$C$. Since $\widetilde{\u}_t$ is constant within each block, it can
change only at the $K-1$ block boundaries.
It therefore suffices to bound the JS distance between
$\widetilde{\u}_{s_k-1}$ and $\widetilde{\u}_{s_k}$, the comparator vectors on blocks
$k-1$ and $k$, respectively.
For any $k\in\{2,\dots,K\}$, we have
\begin{align}
  \operatorname{JS}(\widetilde{\u}_{s_k},\widetilde{\u}_{s_k-1})^2 ={}&\frac12\sum_{i=1}^d\Bigg[ [\widetilde{\u}_{s_k}]_i \ln\frac{2[\widetilde{\u}_{s_k}]_i} {[\widetilde{\u}_{s_k}]_i+[\widetilde{\u}_{s_k-1}]_i}+[\widetilde{\u}_{s_k-1}]_i \ln\frac{2[\widetilde{\u}_{s_k-1}]_i} {[\widetilde{\u}_{s_k}]_i+[\widetilde{\u}_{s_k-1}]_i} \Bigg]\notag\\
  \leq{}&\frac12\sum_{i=1}^d \frac{\big([\widetilde{\u}_{s_k}]_i-[\widetilde{\u}_{s_k-1}]_i\big)^2} {[\widetilde{\u}_{s_k}]_i+[\widetilde{\u}_{s_k-1}]_i}, \label{eq:kl-decomposition-interior}
\end{align}
where the inequality uses $\ln x\leq x-1$.
We bound the contributions of the first $d-1$ coordinates and the last
coordinate separately. For each $j\in[d-1]$, the coordinate
$[\widetilde{\u}_t]_j$ takes one of the two values
$(1-\epsilon)/(2(d-1))$ and $(1+\epsilon)/(2(d-1))$. Hence,
\begin{align}
  {}& \sum_{j=1}^{d-1} \frac{ \left( [\widetilde{\u}_{s_k}]_j - [\widetilde{\u}_{s_k-1}]_j \right)^{2} }{ [\widetilde{\u}_{s_k}]_j \mathbin{+} [\widetilde{\u}_{s_k-1}]_j } \mathrel{\leq} \frac{2\epsilon^2}{d-1} \sum_{j=1}^{d-1} \left| I_{k,j}-I_{k-1,j} \right| \leq 2\epsilon^2, \label{eq:first-coordinate-jeffreys-interior}
\end{align}
where the first inequality uses
$[\widetilde{\u}_{s_k}]_j+[\widetilde{\u}_{s_k-1}]_j\geq1/(2(d-1))$
and
$[\widetilde{\u}_{s_k}]_j-[\widetilde{\u}_{s_k-1}]_j
=\epsilon(I_{k,j}-I_{k-1,j})/(d-1)$.

For the last coordinate, we have
\begin{align}
  {}& \frac{ \left( [\widetilde{\u}_{s_k}]_d - [\widetilde{\u}_{s_k-1}]_d \right)^2 }{ [\widetilde{\u}_{s_k}]_d \mathbin{+} [\widetilde{\u}_{s_k-1}]_d } \leq 2\left( [\widetilde{\u}_{s_k}]_d - [\widetilde{\u}_{s_k-1}]_d \right)^2 \leq 2\epsilon^2. \label{eq:last-coordinate-jeffreys-interior}
\end{align}
The first inequality uses
$[\widetilde{\u}_{s_k}]_d,
[\widetilde{\u}_{s_k-1}]_d\geq1/4$. The last inequality follows from the
definition of $\widetilde{\u}_t$, which gives
\begin{align*}
  {}& \left| [\widetilde{\u}_{s_k}]_d - [\widetilde{\u}_{s_k-1}]_d \right| = \frac{\epsilon}{d-1} \left| \sum_{j=1}^{d-1} \left( I_{k,j}-I_{k-1,j} \right) \right| \leq \frac{\epsilon}{d-1} \sum_{j=1}^{d-1} \left| I_{k,j}-I_{k-1,j} \right| \leq \epsilon.
\end{align*}

Combining~\eqref{eq:first-coordinate-jeffreys-interior}
and~\eqref{eq:last-coordinate-jeffreys-interior} with~\eqref{eq:kl-decomposition-interior}, we obtain
\begin{align}
  P_T^{\mathrm{JS}} =\sum_{k=2}^K\operatorname{JS}\big(\widetilde{\u}_{s_k},\widetilde{\u}_{s_k-1}\big) \leq(K-1)\sqrt{2}\epsilon \mathrel{\leq}\frac{\sqrt{2}T\epsilon^3}{d-1}=C. \label{eq:jeffreys-path-hard-instance-interior}
\end{align}
Here, the second inequality uses
$K-1\leq T/L$ and $L\geq(d-1)/\epsilon^2$, and the last equality
follows from the definition of $\epsilon$. Thus, every realization of
the comparator sequence $\widetilde{\u}_{1:T}$ belongs to
$\U_C^{\mathrm{JS}}$.

\paragraph{Lower bounding the minimax regret.}
For each block $k\in[K]$, let
$Y_{\mathcal{I}_k}\coloneqq(Y_t)_{t\in\mathcal{I}_k}$ denote the random
label sequence on block $k$, and let
$y_{\mathcal{I}_k}\coloneqq(y_t)_{t\in\mathcal{I}_k}
\in\Y^{|\mathcal{I}_k|}$ denote one of its realizations. Conditional on
the environment index $\mathbf{I}_k$, the probability mass function of
$Y_{\mathcal{I}_k}$ is
$\widetilde{\q}_k(y_{\mathcal{I}_k}\mid\mathbf{I}_k)
\coloneqq
\prod_{t\in\mathcal{I}_k}[\widetilde{\u}_t]_{y_t}$, and we use
$\widetilde{\q}_{Y_{\mathcal{I}_k}\mid\mathbf{I}_k}$ to denote the
corresponding conditional distribution. We further define its marginal
probability mass function by
$\bar{\q}_k(y_{\mathcal{I}_k})
\coloneqq
\E_{\mathbf{I}_k}
[\widetilde{\q}_k(y_{\mathcal{I}_k}\mid\mathbf{I}_k)]$ and use
$\bar{\q}_{Y_{\mathcal{I}_k}}$ to denote the corresponding marginal
distribution. The same argument used to derive~\eqref{eq:lower-bound-V}
in the proof of Theorem~\ref{thm:lower-bound} then shows that
\begin{align}
  \mathcal{V}_T(\U_C^{\mathrm{JS}}) \geq \sum_{k=1}^K \E_{\mathbf{I}_k} \left[ \KL\left( \widetilde{\q}_{Y_{\mathcal{I}_k}\mid\mathbf{I}_k} \Vert \bar{\q}_{Y_{\mathcal{I}_k}} \right) \right]. \label{eq:blockwise-information-kl-interior}
\end{align}

We next lower bound the information contributed by each block. Recall that $\mathbf{I}_k=(I_{k,1},\dots,I_{k,d-1})$ is the binary environment index of block $k$, where $I_{k,j}\in\{0,1\}$ determines whether the $j$-th coordinate of $\widetilde{\u}_{s_k}$ is $(1-\epsilon)/(2(d-1))$ or $(1+\epsilon)/(2(d-1))$. For every $j\in[d-1]$, we define
\[
  Z_{k,j} =\sum_{t=s_k}^{s_k+L-1}\indicator\{Y_t=j\},
\]
which counts the number of occurrences of label $j$ in the first $L$ rounds of block $k$. We further define $\Zb_k=(Z_{k,1},\dots,Z_{k,d-1})$. Given a realization $I_{k,j}=b\mathrel{\in}\{0,1\}$, label $j$
is observed independently at each round with probability
$\pi_b=\frac{1}{2(d-1)}+\frac{\epsilon}{d-1}(b-\frac12)$.
Hence, conditional on $I_{k,j}=b$, the random variable $Z_{k,j}$ follows the binomial distribution
$P_b=\operatorname{Bin}(L,\pi_b)$, whose probability mass function is
\[
  P_b(z)=\Pr(Z_{k,j}=z\mid I_{k,j}=b) =\binom{L}{z}\pi_b^z(1-\pi_b)^{L-z}, \qquad z\in\{0,\dots,L\}.
\]
Since $\Zb_k$ is a deterministic function of
$Y_{\mathcal{I}_k}$, the same argument used
to derive~\eqref{eq:lower-bound-mutual} in the proof of
Theorem~\ref{thm:lower-bound} gives
\begin{align}
  {}&\E_{\mathbf{I}_k} \left[ \KL\left(\widetilde{\q}_{Y_{\mathcal{I}_k}\mid\mathbf{I}_k} \Vert\bar{\q}_{Y_{\mathcal{I}_k}} \right) \right]\geq \sum_{j=1}^{d-1} \mathrm{I}\left(I_{k,j};Z_{k,j} \right) \geq\frac{d-1}{2}\operatorname{TV}(P_0,P_1)^2. \label{eq:coordinate-information-kl-interior}
\end{align}
Here, $\mathrm{I}(X;Y)$ denotes the mutual information between the
random variables $X$ and $Y$, and
$\operatorname{TV}(P,Q)=\frac12\sum_z|P(z)-Q(z)|$
denotes the total variation distance.
The first inequality follows from the independence of the environment indices and the data-processing inequality, as in~\eqref{eq:lower-bound-mutual}.
The last inequality follows from
Pinsker's inequality since
$I_{k,j}\sim\operatorname{Bern}(1/2)$ and the marginal distribution of
$Z_{k,j}$ is the equally weighted mixture
$M=(P_0+P_1)/2$ of its two conditional distributions. In particular,
\[
  \mathrm{I}(I_{k,j};Z_{k,j}) =\frac12\KL(P_0\Vert M)+\frac12\KL(P_1\Vert M) \geq\frac12\operatorname{TV}(P_0,P_1)^2.
\]

We next lower bound the total variation distance between the two conditional distributions.
Let $\rho=\sqrt{\pi_0\pi_1}+\sqrt{(1-\pi_0)(1-\pi_1)}$.
We have
\begin{align*}
  \operatorname{TV}(P_0,P_1)=1-\sum_{z=0}^L\min\{P_0(z),P_1(z)\}\geq1-\sum_{z=0}^L\sqrt{P_0(z)P_1(z)}.
\end{align*}
The equality follows from
$|a-b|=a+b-2\min\{a,b\}$
and the fact that each probability mass function sums to one.
The inequality uses
$\min\{a,b\}\leq\sqrt{ab}$ for $a,b\geq0$.
Using the probability mass functions of the two binomial distributions, we obtain
\begin{align*}
  \sum_{z=0}^L\sqrt{P_0(z)P_1(z)} &=\sum_{z=0}^L\binom{L}{z} \big(\sqrt{\pi_0\pi_1}\big)^z \big(\sqrt{(1-\pi_0)(1-\pi_1)}\big)^{L-z}\\
  &=\left(\sqrt{\pi_0\pi_1}+\sqrt{(1-\pi_0)(1-\pi_1)}\right)^L =\rho^L.
\end{align*}
The second equality follows by expanding the
$L$-th power of the sum. To bound $\rho$, we note that
\begin{align}
  1-\rho &=\frac12\left[ (\sqrt{\pi_1}-\sqrt{\pi_0})^2 +(\sqrt{1-\pi_1}-\sqrt{1-\pi_0})^2 \right]\notag\\
  &\geq\frac12\frac{(\pi_1-\pi_0)^2}{(\sqrt{\pi_1}+\sqrt{\pi_0})^2} \geq\frac{\epsilon^2}{4(d-1)}, \label{eq:interior-affinity}
\end{align}
where the last inequality uses
$\pi_1-\pi_0=\epsilon/(d-1)$ and
$(\sqrt{\pi_1}+\sqrt{\pi_0})^2\leq2(\pi_0+\pi_1)=2/(d-1)$.
Substituting these bounds into the total variation inequality above gives
\begin{align}
  \operatorname{TV}(P_0,P_1) &\mathrel{\geq}1-\rho^L \geq1-\exp\left(-\frac{L\epsilon^2}{4(d-1)}\right) \geq1-e^{-1/4}, \label{eq:interior-total-variation}
\end{align}
where the second inequality uses $1-x\leq e^{-x}$
with $x=1-\rho$, and the last inequality follows from
$L\epsilon^2/(d-1)\geq1$. Substituting~\eqref{eq:interior-total-variation} into
\eqref{eq:coordinate-information-kl-interior}, we obtain
\begin{align}
  \E_{\mathbf{I}_k} \left[ \KL\left( \widetilde{\q}_{Y_{\mathcal{I}_k}\mid\mathbf{I}_k} \Vert \bar{\q}_{Y_{\mathcal{I}_k}} \right) \right] \geq \frac{d-1}{2} \left( 1-e^{-1/4} \right)^2. \label{eq:block-information-kl-interior}
\end{align}
Combining~\eqref{eq:blockwise-information-kl-interior}
and~\eqref{eq:block-information-kl-interior} yields
\begin{align}
  \mathcal{V}_T(\U_C^{\mathrm{JS}}) &\geq\frac{d-1}{2}(1-e^{-1/4})^2K \geq\frac{(1-e^{-1/4})^2}{8}T\epsilon^2 =\Omega\left((d-1)^{\frac23}T^{\frac13}C^{\frac23}\right). \label{eq:dynamic-main-regime-kl-interior}
\end{align}
Here, the second inequality uses
$K=\lfloor T/L\rfloor\geq T/(2L)\geq T\epsilon^2/(4(d-1))$, since
$T/L\geq1$ and
$L=\lceil(d-1)/\epsilon^2\rceil\leq2(d-1)/\epsilon^2$
in the main regime. The last equality follows from the definition
$\epsilon=((d-1)C/(\sqrt{2}T))^{1/3}$.

\paragraph{Handling Corner Cases.} We next consider the two regimes outside the main regime. First, suppose
that $C<\sqrt{2(d-1)/T}$. We use the same construction with $\epsilon=\epsilon_0\coloneqq\sqrt{(d-1)/T}$. In this case, $L=T$ and $K=1$, so the comparator
sequence is constant and has zero JS-path length.
Applying the one-block estimate in~\eqref{eq:block-information-kl-interior} gives
\[
  \mathcal{V}_T(\U_C^{\mathrm{JS}}) \geq \frac{d-1}{2} \left( 1-e^{-1/4} \right)^{2} = \Omega(d).
\]
Moreover, the condition on $C$ implies $(d-1)^{\frac{2}{3}}
T^{\frac{1}{3}}
C^{\frac{2}{3}}
<
2^{\frac{1}{3}}(d-1)
=
\O(d).$ Thus, the one-block lower bound already dominates the desired
dynamic term in this regime.

Next, suppose that
$C>T/(4\sqrt{2}(d-1))$. We apply the construction from the main regime with the
smaller path length budget
$
C_0
=
{T}/({4\sqrt{2}(d-1)}).
$
Since
$\U_{C_0}^{\mathrm{JS}}\subseteq\U_C^{\mathrm{JS}}$,
the monotonicity of the comparator classes and
\eqref{eq:dynamic-main-regime-kl-interior} give
\[
  \mathcal{V}_T(\U_C^{\mathrm{JS}}) \geq \mathcal{V}_T(\U_{C_0}^{\mathrm{JS}}) = \Omega\left( (d-1)^{\frac{2}{3}} T^{\frac{1}{3}} C_0^{\frac{2}{3}} \right) = \Omega(T).
\]

Combining the main regime with the two boundary regimes, we conclude that
\begin{align}
  \mathcal{V}_T(\U_C^{\mathrm{JS}}) \geq \Omega\left( \min\left\{ T, d^{\frac{2}{3}}T^{\frac{1}{3}}C^{\frac{2}{3}} \right\} \right). \label{eq:dynamic-lower-bound-kl-interior}
\end{align}

Combining this bound with the classical static lower bound
$\Omega(d\log(1+T/d))$ for $C=0$
completes the proof.
\end{proof}

\subsection{Proof of Corollary~\ref{Cor:Drichlet-minimax}}
\label{app:jeffreys-minimax-corollary}

\begin{proof}[Proof of Corollary~\ref{Cor:Drichlet-minimax}]
Fix any comparator sequence $\u_1,\dots,\u_T\in\Delta_d$ and let $\beta\in(0,1)$ be a certain parameter for mixing the comparator. We define the interior counterpart of $\u_t$ by $\tilde{\u}_t = (1-\beta)\u_t + \frac{\beta}{d}\mathbf{1}$ for all $t\in[T]$. Clearly, we have $\min_{i\in[d]}\tilde u_{t,i}\ge \beta/d$. Besides, the gap between $\tilde{\u}_t$ and $\u_t$ can be bounded by
\begin{align}
  \sum_{t=1}^T \ell_t(\tilde{\u}_t) - \sum_{t=1}^T \ell_t({\u}_t) ={}& \sum_{t=1}^T \ln\left(\frac{\u_t^\top \x_t}{(1-\beta) \u_t^\top \x_t + \frac{\beta}{d} \mathbf{1}^\top \x_t}\right)\leq T\ln \left(\frac{1}{1-\beta}\right). \label{eq:comparator-smoothing-gap}
\end{align}
We first relate the JS-path length of the smoothed sequence directly to the $L_1$-path length of the original sequence. For any $\p,\q\mathrel{\in}\Delta_d$ with $p_i,q_i\geq\beta/d$, let $\mathbf{m}=(\p+\q)/2$. Applying $\ln x\leq x-1$ coordinate-wise gives
\begin{align*}
  \operatorname{JS}(\p,\q)^2 &=\frac12\sum_{i=1}^d\left(p_i\ln\frac{p_i}{m_i}+q_i\ln\frac{q_i}{m_i}\right)\\
  &\leq\frac12\sum_{i=1}^d\left[p_i\left(\frac{p_i}{m_i}-1\right)+q_i\left(\frac{q_i}{m_i}-1\right)\right]\\
  &\mathrel{=}\frac12\sum_{i=1}^d\frac{(p_i-q_i)^2}{p_i+q_i} \mathrel{\leq}\frac{d}{4\beta}\norm{\p-\q}_2^2.
\end{align*}
Since $\tilde{\u}_t-\tilde{\u}_{t-1}=(1-\beta)(\u_t-\u_{t-1})$, it follows that
\begin{align}
  P_{T}^{\mathrm{JS}}(\tilde{\u}_{1:T}) &\leq \frac{1-\beta}{2}\sqrt{\frac{d}{\beta}} \sum_{t=2}^T\norm{\u_t-\u_{t-1}}_2 \leq \frac{1-\beta}{2}\sqrt{\frac{d}{\beta}}P_T. \label{eq:smoothed-jeffreys-path}
\end{align}

Then, we can upper bound the dynamic regret with respect to any comparator $\u_t\in\Delta_d$ by
\begin{align*}
  \DReg_T(\{\u_t\}_{t=1}^T) ={}& \sum_{t=1}^T \ell_t(\w_t) - \sum_{t=1}^T \ell_t(\tilde{\u}_t) + \sum_{t=1}^T \ell_t(\tilde{\u}_t) -\sum_{t=1}^T \ell_t(\u_t)\\
  \leq{}& \sum_{t=1}^T \ell_t(\w_t) - \sum_{t=1}^T \ell_t(\tilde{\u}_t) \;+\; T\ln \left(\frac{1}{1-\beta}\right)\\
  \leq{}& \O\left( d^{\frac43}B\beta^{-\frac13}(1-\beta)^{\frac23} T^{\frac13}P_T^{\frac23} +d\ln(dT) +T\ln \left(\frac{1}{1-\beta}\right) \right),
\end{align*}
where the first inequality follows from~\eqref{eq:comparator-smoothing-gap}, and the last inequality follows from~\eqref{eq:smoothed-jeffreys-path}
and the regret guarantee assumption in Corollary~\ref{Cor:Drichlet-minimax}.
We now choose $\beta$ to make the bound tight:
\begin{itemize}
\item \textbf{Case 1 ($P_T=0$).} We choose
  $\beta=1/T$, which yields
  \[
    \DReg_T(\{\u_t\}_{t=1}^T) \leq\O\left(d\ln(dT)\right).
  \]

\item \textbf{Case 2 ($P_T\mathrel{>}0$).}
  To balance the dynamic regret of $\tilde{\u}_{1:T}$ and smoothing terms, write
  $a=\beta/(1-\beta)>0$.
  Using $(1+a)^{-1/3}\leq1$ and
  $\ln(1+a)\leq a$, we obtain
  \[
    \DReg_T(\{\u_t\}_{t=1}^T) \leq\O\left( d^{\frac43}B T^{\frac13}P_T^{\frac23} a^{-\frac13}+Ta+d\ln(dT) \right).
  \]
  Balancing the first two terms gives $a=dB^{3/4}\sqrt{P_T/T}$.
  With this choice, both terms equal
  $dB^{3/4}\sqrt{TP_T}$, yielding
  \[
    \DReg_T(\{\u_t\}_{t=1}^T) \mathrel{\leq}\O\left(dB^{\frac34}\sqrt{TP_T}+d\ln(dT)\right).
  \]
\end{itemize}
Combining the two cases proves the claim.
\end{proof}

\subsection{Proof of Proposition~\ref{thm:interval-regret}}
\label{app:interval-regret}

\begin{proof}[Proof of Proposition~\ref{thm:interval-regret}]
Fix an interval $\mathcal I=[r,s]$ and a comparator $\u\in\Delta_d$. If $\ell_t(\u)=+\infty$ for some $t\in\mathcal I$, the interval-regret claim is immediate, so suppose that its loss is finite throughout $\mathcal I$.
Following the same steps as in the proof of Theorem~\ref{thm:main}, for any distribution $Q_t$ over $\Delta_d$, we have
\begin{align*}
  \ell_t(\w_t) \leq{}& \E_{\w\sim Q_t}[\ell_t(\w)] + \mathrm{KL}(Q_t\Vert P_t) - \mathrm{KL}(Q_t\Vert P_{t+1}) + \ln\frac{1}{1-\mu_{t+1}}.
\end{align*}
For each round $t\in\mathcal I=[r,s]$, we take the fixed distribution
$Q_t=Q_{\mathcal I}=\mathrm{Dir}(\mathbf{1}+\gamma\u)$.
Summing the above inequality over $\mathcal I$ with $\mu_t=1/t$ yields
\begin{align*}
  \sum_{t\in\mathcal I}\ell_t(\w_t) \leq{}& \underbrace{ \sum_{t\in\mathcal I} \E_{\w\sim Q_{\mathcal I}}[\ell_t(\w)] }_{\texttt{term~(a)}} + \underbrace{ \mathrm{KL}(Q_{\mathcal I}\Vert P_r) }_{\texttt{term~(b)}} - \mathrm{KL}(Q_{\mathcal I}\Vert P_{s+1}) + \sum_{t\in\mathcal I}\ln\left(1+\frac1t\right).
\end{align*}

For term~(a), the same argument as in the proof of Theorem~\ref{thm:main} gives
\begin{align*}
  \texttt{term~(a)} \leq \sum_{t\in\mathcal I}\ell_t(\u) + |\mathcal I|\ln\left(1+\frac d\gamma\right).
\end{align*}
For term~(b), the fixed-share update gives
$P_r(\w)\geq\mu_rP_1(\w)$, and hence
\begin{align*}
  \texttt{term~(b)} &= \mathrm{KL}(Q_{\mathcal I}\Vert P_1) + \E_{\w\sim Q_{\mathcal I}} \left[\ln\frac{P_1(\w)}{P_r(\w)}\right] \\
  &\leq \mathrm{KL}(Q_{\mathcal I}\Vert P_1) + \ln\frac1{\mu_r} \\
  &\leq (d-1)\ln(1+\gamma) + \ln T,
\end{align*}
where the last inequality follows from Lemma~\ref{lemma:KL-upper} and the calculation in the proof of Theorem~\ref{thm:main}.
Taking $\gamma=T$, dropping the nonpositive KL term, and using
\begin{align*}
  |\mathcal I|\ln\left(1+\frac dT\right) \leq d ~~~ \mbox{and}~~ \sum_{t=r}^s\ln\left(1+\frac1t\right) = \ln\frac{s+1}{r} \leq \ln(T+1)
\end{align*}
complete the proof for the interval regret.

For the switching guarantee, partition $[T]$ into the $\mathsf{S}_T+1$ maximal intervals on which the comparator is constant. Applying the interval-regret bound on each interval and summing gives
\begin{align*}
  \DReg_T(\{\u_t\}_{t=1}^T) \leq \O\bigl(d(\mathsf{S}_T+1)\ln(dT)\bigr),
\end{align*}
which completes the proof.
\end{proof}

\section{Omitted Proofs for Section~\ref{subsec:q-path}}
\label{app:q-path-proofs}
\subsection{Proof of Theorem~\ref{thm:q-path-bound}}
\label{app:q-path-bound}

\begin{proof}[Proof of Theorem~\ref{thm:q-path-bound}]
Fix any $q\in[0,1]$ and any comparator sequence $\u_1,\dots,\u_T\in\Delta_d$. We use the same mixability-based decomposition as in the proof of Theorem~\ref{thm:main}, with the virtual comparator distribution $Q_t=\mathrm{Dir}(\mathbf{1}+\gamma\u_t)$. The bounds on the expected-loss, endpoint, and fixed-share terms remain unchanged. The only modification concerns the total variation between $Q_t$ and $Q_{t-1}$ in term~\texttt{(b-1)}. For every $t\geq2$, Lemma~\ref{lemma:dirichlet-js} gives
\begin{align}
  \int_{\Delta_d} |Q_t(\w)-Q_{t-1}(\w)|\,\mathrm d\w &\leq \min\left\{2,2\sqrt{2\gamma}\,\operatorname{JS}(\u_t,\u_{t-1})\right\}. \label{eq:q-order-tv-js}
\end{align}
Therefore, under the endpoint convention above, for every $q\in[0,1]$,
\begin{align*}
  \int_{\Delta_d} |Q_t(\w)-Q_{t-1}(\w)|\,\mathrm d\w &\leq \min\{2,2\sqrt{2\gamma}\operatorname{JS}(\u_t,\u_{t-1})\}\\
  &\leq 2(2\gamma)^{q/2}\operatorname{JS}(\u_t,\u_{t-1})^q.
\end{align*}
Summing over time and using $2^{q/2}\leq\sqrt{2}$ yields
\begin{align}
  \sum_{t=2}^T \int_{\Delta_d} |Q_t(\w)-Q_{t-1}(\w)|\,\mathrm d\w \leq 2\sqrt{2}\gamma^{q/2}P_{T,q}^{\mathrm{JS}}. \label{eq:q-order-total-variation}
\end{align}

Substituting~\eqref{eq:q-order-total-variation} into~\eqref{eq:term-b-1a} and retaining the other bounds from the proof of Theorem~\ref{thm:main}, we obtain
\begin{align}
  \DReg_T(\{\u_t\}_{t=1}^T) \leq{}& T\ln\left(1+\frac d\gamma\right) + 2\sqrt{2}C_{T,d}\gamma^{q/2}P_{T,q}^{\mathrm{JS}}\notag\\
  &+(d-1)\ln(1+\gamma) + (1+\ln T), \label{eq:q-order-final-regret}
\end{align}
where $C_{T,d}=(d-1)\ln(T+1)+d(\ln d+1)=\O(d\ln(dT))$. If $P_{T,q}^{\mathrm{JS}}\leq T^{-q/2}$, choose $\gamma=T$. Equation~\eqref{eq:q-order-final-regret} then gives $\O(d\ln(dT))$. Otherwise, choose
\begin{align*}
  \gamma = T^{\frac{2}{q+2}} \bigl(P_{T,q}^{\mathrm{JS}}\bigr)^{-\frac{2}{q+2}} \bigl(\ln(dT)\bigr)^{-\frac{2}{q+2}}.
\end{align*}
The condition of this case ensures $\gamma\leq T$. Using $\ln(1+x)\leq x$ in~\eqref{eq:q-order-final-regret}, the first two terms are both bounded by
\begin{align*}
  \O\left( d T^{\frac{q}{q+2}} \bigl(P_{T,q}^{\mathrm{JS}}\bigr)^{\frac{2}{q+2}} \bigl(\ln(dT)\bigr)^{\frac{2}{q+2}} \right),
\end{align*}
while the remaining terms are $\O(d\ln(dT))$. Combining the two cases proves the theorem.
\end{proof}

\subsection{Calculations for the Rising Concave Path}
\label{app:rising-concave-calculation}

For the rising concave path, write $p_t=[\u_t]_1=3/4-1/(4t)$ and $J_t=\operatorname{JS}(\u_t,\u_{t-1})$. For every $t\geq2$, let $\delta_t\coloneqq p_t-p_{t-1}=1/(4t(t-1))$. Let $m_t=(p_t+p_{t-1})/2$ and $\bar{\u}_t=(\u_t+\u_{t-1})/2$. The JS distance between the consecutive Bernoulli distributions satisfies
\begin{align}
  J_{t}^{2} &\mathrel{=} \frac12\mathrm{KL}(\u_{t}\Vert\bar{\u}_t) \mathbin{+}\frac12\mathrm{KL}(\u_{t-1}\Vert\bar{\u}_t). \label{eq:rising-concave-js}
\end{align}
For either $\theta\mathrel{=}p_t$ or $\theta\mathrel{=}p_{t-1}$, the inequality $\ln x\leq x-1$ gives
\[
  \mathrm{KL}\bigl(\operatorname{Bern}(\theta)\Vert\operatorname{Bern}(m_t)\bigr) \mathrel{\leq}\frac{(\theta-m_t)^2}{m_t(1\mathbin{-}m_{t})} \mathrel{=}\frac{\delta_t^2}{4m_{t}(1\mathbin{-}m_{t})}.
\]
Since $m_t\in[1/2,3/4]$, each of the two KL divergences is at most $4\delta_t^2/3$. On the other hand, Pinsker's inequality bounds each of them below by $\delta_t^2/2$, since
\[
  \lVert\u_t-\bar{\u}_t\rVert_1=\lVert\u_{t-1}-\bar{\u}_t\rVert_1=\delta_t.
\]
Therefore,
\begin{align}
  \frac{\delta_t}{\sqrt{2}} \leq J_t \leq \frac{2\delta_t}{\sqrt{3}}. \label{eq:rising-concave-js-bounds}
\end{align}
Consequently, $J_t=\Theta([t(t-1)]^{-1})$. Because every transition is nonzero, $P_{T,0}^{\mathrm{JS}}=T-1$. For each fixed $q\in(0,1]$, summing~\eqref{eq:rising-concave-js-bounds} gives
\begin{align}
  P_{T,q}^{\mathrm{JS}} &= \Theta\left( \sum_{t=2}^T[t(t-1)]^{-q} \right) = \begin{cases} \Theta\bigl(T^{1-2q}\bigr), &0<q<\frac12,\\
  \Theta(\ln T), &q=\frac12,\\
  \Theta(1), &\frac12<q\leq1. \end{cases} \label{eq:rising-concave-q-path}
\end{align}
Suppressing logarithmic factors, substitution into Theorem~\ref{thm:q-path-bound} gives a regret bound of order $dT^{\alpha(q)}$, where
\begin{align}
  \alpha(q) = \begin{cases} \displaystyle\frac{2-3q}{q+2}, &0\leq q\leq\frac12,\\[2mm]
  \displaystyle\frac{q}{q+2}, &\frac12\leq q\leq1. \end{cases} \label{eq:rising-concave-regret-exponent}
\end{align}
The first branch is strictly decreasing and the second is strictly increasing, so $\alpha(q)$ is uniquely minimized at $q=1/2$, where $\alpha(1/2)=1/5$. In comparison, $\alpha(0)=1$ and $\alpha(1)=1/3$, yielding the three rates stated in Section~\ref{subsec:q-path}.

\section{Omitted Proofs for Section~\ref{sec:results-bounded-gradient}}
\label{app:general-oxo-proofs}
\label{appendix:proof-main-OCO}
\begin{proof}[Proof of Theorem~\ref{thm:main-OCO}]
We begin with a similar mixability-based regret decomposition as~\citet{ICML'25:Zhang-mixability}. Let $\tilde{m}_t(P_t) = -\frac{1}{\eta} \ln\left(\E_{\u\sim P_t}[e^{-\eta \tilde{\ell}_t(\u)}]\right)$ be the mix loss. The dynamic regret can be decomposed by
\begin{align*}
  {}&\sum_{t=1}^T \ell_t(\w_t) - \sum_{t=1}^T \ell_t(\u_t)\leq \sum_{t=1}^T \tilde{\ell_t}(\w_t) - \sum_{t=1}^T \tilde{\ell}_t(\u_t)\\
  = {}& \underbrace{\sum_{t=1}^T \tilde{\ell_t}(\w_t) -\sum_{t=1}^T \tilde{m}_t(P_t)}_{\texttt{term~(a)}} + \underbrace{\sum_{t=1}^T \tilde{m}_t(P_t)-\sum_{t=1}^T \E_{\u\sim Q_t}[\tilde{\ell}_t(\u)]}_{\texttt{term~(b)}}\\
  {}&+\underbrace{\sum_{t=1}^T \E_{\u\sim Q_t}[\tilde{\ell}_t(\u)] - \sum_{t=1}^T \tilde{\ell}_t(\u_t)}_{\texttt{term~(c)}},
\end{align*}
where the first line is due to~\citet[Lemma 4.2]{book'16:Hazan-OCO} under the step size setting $\eta = \frac{1}{5}\min\left\{\frac{1}{2G},\kappa\right\}$. In the analysis we choose $Q_t = \mathcal{N}(\u_t,\sigma^2 I_d)$, where $\sigma>0$ is a parameter that can be virtually tuned to make the bound tight.

One can handle terms~(a) and~(c) using arguments similar to those in~\citet{ICML'25:Zhang-mixability}. However, the most challenging part is the analysis of term~(b), where the domain constraint is enforced via an intractable I-projection of the distribution $P_t$ onto a set of infinite Gaussian mixtures. In our algorithm, instead of performing this intractable projection, we identify that it is sufficient to project each component of the mixture distribution $P_t$ individually rather than projecting the mixture as a whole. The latter would require a more in-depth analysis that leverages the two-layer structure of Algorithm~\ref{alg:flh-ew}. This leads to an efficient method. In what follows, we first analyze terms~(a) and~(c) using arguments similar to those in~\citep{ICML'25:Zhang-mixability}, and then turn to the most challenging term~(b).

\paragraph{Bounding term~(a).} Since $\tilde{\ell}_t(\w)$ is a quadratic function and the initial distribution of each base base-leaner $\mathcal{B}_i$ is a Gaussian, according to~\citet[Theorem 5]{COLT'18:Hoeven-EW} shows that the distribution $P_{t+1,i} = \mathcal{N}(\w_{t+1,i},H^{-1}_{t+1,i})$ for any base-learner $B_i$ updated by~\eqref{eq:update-rule} is also a Gaussian distribution. More precisely, the mean and covariance matrix can be updated by
\begin{equation}
  \begin{cases} & H_{t+1,i} = H_{t,i} + 2\eta^2 \g_t \g_t^\top\\
  & \w'_{t+1,i} = \w_{t,i} - \eta H_{t+1,i}^{-1}(1-2\eta\g_t^\top(\w_t-\w_{t,i})) \g_t\\
  & \w_{t+1,i} = \argmin_{\u\in\W} \norm{\u-\w'_{t+1,i}}_{H_{t+1,i}} \end{cases} \label{eq:update-ONS-EW}
\end{equation}
The above essentially follows the update procedure of online Newton step~\citep{MLJ'07:ONS}. The design matrix is also symmetric positive definite and $H_{t+1,i} = I_d+ 2\eta^2\sum_{s=i}^{t}\g_s\g_s^\top \leq (1+ \frac{dt}{2})I_d \leq dT I_d$ for any $\mathcal{B}_{i}\in H_t$ and $\w_{t,i}\in\W$ due to the projection step.

Our goal is to show $\texttt{term~(a)}\leq 0$. To show this, it is sufficient to have $\E_{\u\sim P_t}\big[\exp(-\eta \tilde{\ell}_t(\u))\big] \leq \exp\big(-\eta\tilde{\ell}_t(\w_t)\big) = 1$ for each iteration. This can be achieved by the following arguments
\begin{align*}
  \E_{\u\sim P_t}\left[\exp(-\eta \tilde{\ell}_t(\u))\right] ={}& \sum_{\mathcal{B}_i\in\mathcal{H}_t} p_{t,i} \E_{\u\sim P_{t,i}}[\exp(-\eta \tilde{\ell}_t(\u))]\\
  ={}&\sum_{\mathcal{B}_i\in\mathcal{H}_t} p_{t,i} \E_{\u\sim P_{t,i}}\left[\exp\Big(\eta \g_t^\top(\w_t - \u) - \eta^2 \norm{\u-\w_t}^2_{\g_t\g_t^\top}\Big)\right]\\
  \leq {}& \sum_{\mathcal{B}_i\in\mathcal{H}_t}p_{t,i} \exp\Big(\eta \g_t^\top (\w_t - \w_{t,i}) - \eta^2 \norm{\w_t-\w_{t,i}}^2_{\g_t\g_t^\top}\Big)\\
  \leq {}&\sum_{\mathcal{B}_i\in\mathcal{H}_t } p_{t,i} \left(1+\eta\g_t^\top(\w_t-\w_{t,i})\right) = 1,
\end{align*}
where the first inequality is due to~\citep[Lemma 10]{NIPS'16:MetaGrad} under the condition $\eta \leq 1/(10G)$. The second inequality holds because $e^{z-z^2}\leq 1+z$ for any $z\geq -\frac{2}{3}$. Then, we can have
\begin{align}
  \mbox{term~(a)}\leq 0 \label{eq:final-termA-OCO}.
\end{align}

\paragraph{Bounding term~(c).} A direct calculation according to the definition of $\tilde{\ell}_t$ shows that
\begin{align}
  \texttt{term~(c)} ={}& \sum_{t=1}^T \E_{\u\sim Q_t}[\tilde{\ell}_t(\u)] - \sum_{t=1}^T \tilde{\ell}_t(\u_t)\notag\\
  ={}& \eta \sum_{t=1}^T \E_{\u\sim Q_t}[(\g_t^\top(\u-\u_t))^2]\notag\\
  ={}& \eta \sigma^2\sum_{t=1}^T \norm{\g_t}_2^2\leq \eta dG^2T\sigma^2, \label{eq:term-C-final}
\end{align}
\pagebreak
where the last inequality is due to $\norm{\g_t}_2 \leq \sqrt{d}\norm{\g_t}_\infty \leq \sqrt{d} G$.

\paragraph{Bounding term~(b).}
 As for term~(b), different from the previous work~\citep{ICML'25:Zhang-mixability}, we decompose the mix loss by exploiting the two-layer structure. Denote by
 \[
   \tilde{m}_t(P_{t,i}) = -\frac{1}{\eta}\ln\left(\E_{P_{t,i}}[\exp(-\eta \tilde{\ell}_t(\u))]\right)
 \]
 the mix loss for the individual distribution $P_{t,i}$. We can rewrite the mix loss for the aggregated distribution $P_t$ as
 \begin{align}
   \tilde{m}_t(P_t) ={}& -\frac{1}{\eta} \ln\left(\sum_{\mathcal{B}_i\in\mathcal{H}_t }p_{t,i}\cdot \E_{\u\sim P_{t,i}}[\exp(-\eta \tilde{\ell}_t(\u))]\right)\notag\\
   ={}& -\frac{1}{\eta} \ln\left(\sum_{\mathcal{B}_i\in\mathcal{H}_t }p_{t,i}\cdot \exp(-\eta \tilde{m}_t(P_{t,i}))\right)\notag\\
   ={}& \sum_{\mathcal{B}_i\in\mathcal{H}_t} q_{t,i}\cdot \tilde{m}_t(P_{t,i}) + \frac{1}{\eta}\left(\mathrm{KL}\left(\q_t \,\|\, \p_t\right) - \mathrm{KL}\left(\q_t \,\|\, \tilde{\p}_{t+1}\right)\right), \label{eq:mix-meta}
 \end{align}
 where $\p_t\in\Delta_{\vert \mathcal{H}_t\vert}$ denotes the probability vector over the base-learner pool with the $i$-th entry $p_{t,i}$ and $\tilde{p}_{t+1,i} \propto p_{t,i} \cdot \exp(-\eta \tilde{m}_t(P_{t,i})) = p_{t,i}\cdot \E_{P_{t,i}}[\exp(-\eta\tilde{\ell}_t(\u))]$ is the same as the one defined in Algorithm~\ref{alg:flh-ew}. The last line holds for any $\q_t\in\Delta_{\vert \mathcal{H}_t\vert}$ that assigns weight to each elements in $\mathcal{H}_t$ and $q_{t,i}$ is the $i$-th entry of $\q_t$. Furthermore, we can also rewrite the mix loss for each base learner $\mathcal{B}_i$ as
 \begin{align}
   \tilde{m}_t(P_{t,i}) ={}& \E_{\u\sim Q_t}[\tilde{\ell}_t(\u)] + \frac{1}{\eta} \left(\mathrm{KL}\left(Q_t \,\|\,P_{t,i} \right) - \mathrm{KL}\left(Q_t \,\|\, P'_{t+1,i}\right) \right)\notag\\
   \leq{}&\E_{\u\sim Q_t}[\tilde{\ell}_t(\u)] + \frac{1}{\eta} \left(\mathrm{KL}\left(Q_t \,\|\,P_{t,i} \right) - \mathrm{KL}\left(Q_t \,\|\, P_{t+1,i}\right) \right),~\label{eq:mix-base}
 \end{align}
 where $P'_{t+1,i} \propto P_{t,i} \exp(-\eta \tilde{\ell}_t(\u))$ is the same as~\eqref{eq:update-rule} and the last line is to the Pythagorean theorem for KL divergence since $P_{t+1,i} = \argmin_{P'\in\mathscr{W}}\mathrm{KL}\left(P' \,\|\,P'_{t+1,i} \right)$ and $\mathscr{W}$ is a convex set. Then, plugging~\eqref{eq:mix-base} back into~\eqref{eq:mix-meta}, we arrive
 \begin{align}
   \tilde{m}_t(P_t) \leq \E_{\u\sim Q_t}[\tilde{\ell}_t(\u)] {}&+\underbrace{ \frac{1}{\eta}\left( \sum_{\mathcal{B}_i\in\mathcal{H}_t} q_{t,i} \mathrm{KL}\left(Q_t \,\|\,P_{t,i} \right)+\mathrm{KL}\left(\q_t \,\|\, \p_t\right) \right)}_{\texttt{term~(b-1)}} \notag \\
   {}&\underbrace{- \frac{1}{\eta}\left(\sum_{\mathcal{B}_i\in\mathcal{H}_t} q_{t,i} \mathrm{KL}\left(Q_t \,\|\, P_{t+1,i}\right)+ \mathrm{KL}\left(\q_t \,\|\, \tilde{\p}_{t+1}\right)\right)}_{\texttt{term~(b-2)}}\label{eq:decomposition-efficient}
 \end{align}
 holds for any $\q_t\in\Delta_{\vert \mathcal{H}_t\vert}$ and $Q_t$. Here, we specify $\q_t$ as the minimizer of the optimization problem
 \begin{align*}
   \q_t = \argmin_{\q\in\Delta_{\vert \mathcal{H}_t\vert}} \sum_{\mathcal{B}_i\in\mathcal{H}_t} q_{i} \mathrm{KL}\left(Q_t \,\|\,P_{t,i} \right)+\mathrm{KL}\left(\q \,\|\, \p_t\right),
 \end{align*}
 whose optimal value has the close form formulation as
 \[
   V_t(Q_t) = -\ln\left(\sum_{\mathcal{B}_i\in\mathcal{H}_t}p_{t,i}\cdot\exp(- \mathrm{KL}\left(Q_t \,\|\,P_{t,i} \right)) \right)\geq 0.
 \]
 The value $V_t(Q_t)$ is always greater than 0 since the objective function of the above optimization problem is non-negative. Then, we have
 \[
   \texttt{term~(b-1)}\leq \frac{1}{\eta} V_t(Q_t).
 \]
As for term~(b-2), we can similarly define
\[
  \tilde{\q}_{t+1} = \argmin_{\q\in\Delta_{\vert \mathcal{H}_t\vert}}\sum_{\mathcal{B}_i\in\mathcal{H}_t} q_{i} \mathrm{KL}\left(Q_t \,\|\, P_{t+1,i}\right)+ \mathrm{KL}\left(\q \,\|\, \tilde{\p}_{t+1}\right).
\]
We also have $\tilde{V}_{t+1}(Q_t) = -\ln\Big(\sum_{\mathcal{B}_i\in\mathcal{H}_t}\tilde{p}_{t+1,i}\cdot\exp(- \mathrm{KL}\left(Q_t \,\|\,P_{t+1,i} \right)) \Big)$ as the optimal value of the above optimization problem. Clearly, we have
\begin{align*}
  \texttt{term~(b-2)} =- \frac{1}{\eta}\left(\sum_{\mathcal{B}_i\in\mathcal{H}_t} q_{t,i} \mathrm{KL}\left(Q_t \,\|\, P_{t+1,i}\right)+ \mathrm{KL}\left(\q_t \,\|\, \tilde{\p}_{t+1}\right)\right) \leq -\frac{1}{\eta}\tilde{V}_{t+1}(Q_t),
\end{align*}
since $\q_{t,i}$ is not the minimizer of the objective function in term~(b-2). Plugging the upper bound of term~(b-1) and term~(b-2) into~\eqref{eq:decomposition-efficient}, we have
\begin{align}
  \tilde{m}_t(P_t) \leq \E_{\u\sim Q_t}[\tilde{\ell}_t(\u)] + \frac{1}{\eta}\left(V_t(Q_t) - \tilde{V}_{t+1}(Q_t)\right).\label{eq:termb-intemediate}
\end{align}
Then, we related $\tilde{V}_{t+1}(Q_t)$ to $V_{t+1}(Q_t)$ by
\begin{align}
  V_{t+1}(Q_t) ={}& -\ln\left(\sum_{\mathcal{B}_i\in\mathcal{H}_{t+1}}p_{t+1,i}\cdot\exp(- \mathrm{KL}\left(Q_t \,\|\,P_{t+1,i} \right)) \right)\notag\\
  ={}&-\ln\Bigl((1-\mu_{t+1})\sum_{\mathcal{B}_i\in\mathcal{H}_{t}}\tilde{p}_{t+1,i}\cdot\exp(- \mathrm{KL}\left(Q_t \,\|\,P_{t+1,i} \right))\notag\\
  {}&\qquad + \mu_{t+1} \exp(-\mathrm{KL}\left(Q_t \,\|\, N_0\right) ) \Bigr)\notag\\
  \leq{}& - \ln \left(\sum_{\mathcal{B}_i\in \mathcal{H}_t}\tilde{p}_{t+1,i}\cdot \exp(-\mathrm{KL}\left(Q_t \,\|\,P_{t+1,i} \right) )\right) +\ln \left(\frac{1}{1-\mu_{t+1}}\right)\notag\\
  ={}& \tilde{V}_{t+1}(Q_t) + \log\left(\frac{t+1}{t}\right), \label{eq:fixed-share-anlaysis}
\end{align}
where the second line is due to the fixed-share update~\eqref{eq:fixed-share-FLH} with $N_0 = \mathcal{N}(\u_0,I_d)$ and the last equality is due to the parameter setting $\mu_{t+1} = 1/(t+1)$. Plugging~\eqref{eq:fixed-share-anlaysis} back into~\eqref{eq:termb-intemediate} and taking a summation over $T$ rounds, we obtain
\begin{align}
  \sum_{t=1}^T\tilde{m}_t(P_t) \leq{}& \sum_{t=1}^T \E_{\u\sim Q_t}[\tilde{\ell}_t(\u)] + \frac{1}{\eta}\sum_{t=1}^T\left(V_t(Q_t) - V_{t+1}(Q_t)\right) + \sum_{t=1}^T\frac{1}{\eta}\log\left(\frac{t+1}{t}\right)\notag\\
  \leq {}&\sum_{t=1}^T \E_{\u\sim Q_t}[\tilde{\ell}_t(\u)] + \frac{1}{\eta}\sum_{t=2}^T\left(V_t(Q_t) - V_{t}(Q_{t-1})\right)+ \frac{1}{\eta}V_1(Q_1) + \frac{\log(T+1)}{\eta}.\label{eq:termb-intermediate-2}
\end{align}
where the second inequality holds because $V_t(Q)$ is always non-negative for any $Q$.

It remains to handle the variation term $V_t(Q_t) - V_t(Q_{t-1})$. Denote by
\begin{equation}
  h_{t,i}(\u) = \frac{1}{2}\left(\log \vert H_{t,i}^{-1}\vert + \sigma^2 \mbox{Tr}(H_{t,i}) + \norm{\u-\w_{t,i}}^2_{H_{t,i}}\right), \label{eq:define-h}
\end{equation}
where $\mbox{Tr}(A)$ indicates the trace of a matrix $A$. Then, the variation term can be expressed as
\begin{align}
  V_t(Q_t) - V_{t}(Q_{t-1}) ={}& \ln\left(\frac{\sum_{\mathcal{B}_i\in\mathcal{H}_t}p_{t,i}\cdot\exp(- \mathrm{KL}\left(Q_{t-1} \,\|\,P_{t,i} \right))}{\sum_{\mathcal{B}_i\in\mathcal{H}_t}p_{t,i}\cdot\exp(- \mathrm{KL}\left(Q_t \,\|\,P_{t,i} \right))}\right)\notag \\
  ={}& \ln\left(\frac{\sum_{\mathcal{B}_i\in\mathcal{H}_t}p_{t,i}\cdot\exp(-h_{t,i}(\u_{t-1}))}{\sum_{\mathcal{B}_i\in\mathcal{H}_t}p_{t,i}\cdot\exp(-h_{t,i}(\u_t))}\right) \notag\\
  = {}& J_t(\u_t) - J_t(\u_{t-1})\notag\\
  \leq{}& \sup_{\u\in\W} \norm{\nabla J_t(\u)}_2\cdot \norm{\u_t-\u_{t-1}}_2\notag\\
  \leq{}& \sup_{\u\in\W} \norm{\nabla J_t(\u)}_2\cdot \norm{\u_t-\u_{t-1}}_1 \label{eq:path-length-termB}
\end{align}
where the second line is due to the definition of KL divergence for Gaussian distributions and we define $J_t(\u) = -\ln \left(\sum_{\mathcal{B}_i\in\mathcal{H}_t }p_{t,i}\cdot \exp(-h_{t,i}(\u))\right)$ in the last line. The next step is to control the gradient of the function $J_t(\u)$, which can be calculated as
\begin{align*}
  \nabla J_t(\u) = \sum_{\mathcal{B}_i\in\mathcal{H}_t} \beta_{t,i}(\u) H_{t,i} (\u - \w_{t,i})\quad \forall \u\in\W,
\end{align*}
where $\beta_{t,i}(\u)\in\Delta_{\vert \mathcal{H}_t\vert}$ and $ \beta_{t,i}(\u) \propto p_{t,i} \exp(-h_{t,i}(\u))$. We then  bound the norm of $\nabla J_t(\u)$ by
\begin{align}
  \norm{\nabla J_t(\u)}_2 \leq{}& \sum_{\mathcal{B}_i\in\mathcal{H}_t} \beta_{t,i}(\u) \cdot \norm{H_{t,i}(\u-\w_{t,i})}_2\notag\\
  \leq{}& \sum_{\mathcal{B}_i\in\mathcal{H}_t}\beta_{t,i}(\u) \cdot\sqrt{\mbox{Tr}(H_{t,i})}\cdot \norm{\u-\w_{t,i}}_{H_{t,i}}\notag\\
  \leq{}& \sqrt{\sum_{\mathcal{B}_i\in\mathcal{H}_t}\beta_{t,i}(\u)\norm{\u-\w_{t,i}}_{H_{t,i}}^2}\sqrt{\sum_{\mathcal{B}_i\in\mathcal{H}_t} \beta_{t,i}\mbox{Tr}(H_{t,i})}\label{eq:upperbound-gradient}
\end{align}
where the second inequality is by Cauchy–Schwarz inequality and $\beta_{t,i}(\u)\in\Delta_{\vert \mathcal{H}_t\vert}$. The third line holds because $\| H_{t,i}(\u - \w_{t,i}) \|_2^2 \leq \norm{ H_{t,i}}_2 \, \norm{\u - \w_{t,i}}_{H_{t,i}}^2 \leq \mbox{Tr}(H_{t,i})\norm{ \u - \w_{t,i} }^2_{H_{t,i}}$  for a symmetric positive definite matrix $H_{t,i}$.

Then, we proceed to relate the above two terms back to $h_{t,i}(\u)$. As shown in~\eqref{eq:update-ONS-EW}, $H_{t,i} = I_d + 2\eta^2 \sum_{s=i}^{t-1} g_s g_s^\top \preceq \bigl(1+\tfrac{td}{2}\bigr) I_d \preceq dT I_d$ for any $\mathcal{B}_i \in \mathcal{H}_t$ and $t \in [T]$, which implies $\lambda_{\min}(H_{t,i}^{-1}) \ge 1/(dT)$. Consequently, one has $\log \lvert H_{t,i}^{-1} \rvert \ge - d \log (dT)$. Plugging the lower bound into~\eqref{eq:define-h} yields
\begin{align}
  \mbox{Tr}(H_{t,i}) \leq \frac{2 h_{t,i}(\u) + d\log (dT)}{\sigma^2}~~~\mbox{and}~~~ \norm{\u-\w_{t,i}}_{H_{t,i}}^2 \leq 2h_{t,i}(\u) + d\log (dT).\label{eq:termb-bound-h}
\end{align}

Then, plugging~\eqref{eq:termb-bound-h} into~\eqref{eq:upperbound-gradient}, we can further upper bound the gradient norm by
\begin{align}
  \norm{\nabla J_t(\u)}_2 \leq {}& \frac{1}{\sigma}\left(d\log (dT) + 2\sum_{\mathcal{B}_i\in\mathcal{H}_t}\beta_{t,i}(\u)h_{t,i}(\u)\right)\notag \\
  \leq{}& \frac{1}{\sigma}\left(d\log (dT) + 2\sum_{\mathcal{B}_i\in\mathcal{H}_t}\beta_{t,i}(\u)h_{t,i}(\u) + 2 \mathrm{KL}\left(\beta_t(\u) \,\|\, \p_t\right) \right)\notag\\
  = {}& \frac{1}{\sigma} \left(d\log (dT)- 2\ln\left(\sum_{\mathcal{B}_{i}\in\mathcal{H}_t}p_{t,i} \exp(-h_{t,i}(\u))\right) \right)\notag\\
  \leq{}&\frac{1}{\sigma}\left(d\log (dT) - 2\ln \big(\mu_t\cdot e^{-h_{t,t}(\u)}\big)\right)\notag\\
  ={}& \frac{1}{\sigma} (d\log (dT) + 2\ln t + \sigma^2 d + \norm{\u-\u_0}^2_{2} )\notag\\
  \leq {}&\frac{(d+2) \log (dT) +4}{\sigma}+d(1+\sigma^2)\label{eq:J-bound}
\end{align}
where the third line is by the definition $ \beta_{t,i}(\u) \propto p_{t,i} \exp(-h_{t,i}(\u))$. The fourth line is due to the fixed share update~\eqref{eq:fixed-share-FLH} such that there always exists a base algorithm with weight $\mu_t = 1/t$ and distribution $P_{t,t} = N_0 = \mathcal{N}(\u_{0},I_d)$. The last line holds since $\norm{\u-\u_0}_2\leq \norm{\u-\u_0}_1\leq 2$ for any $\u\in\W$ by Assumption~\ref{assum:bounded-domain} and $\sigma\leq1+\sigma^2$.

Finally, we can upper bound term~(b) by
\begin{align}
  {}&\texttt{term~(b)}\notag\\
  ={}& \sum_{t=1}^T\tilde{m}_t(P_t) - \sum_{t=1}^T \E_{\u\sim Q_t}[\tilde{\ell}_t(\u)]\notag\displaybreak[1]\\
  \leq{}&\frac1\eta\left(\frac{(d+2)\log(dT)+4}{\sigma}+d\sigma^2+d\right) \sum_{t=2}^T\norm{\u_t-\u_{t-1}}_1 +\frac1\eta V_1(Q_1)+\frac2\eta\log(dT)\notag\\
  ={}&\frac1\eta\left(\frac{(d+2)\log(dT)+4}{\sigma}+d\sigma^2+d\right)P_T +\frac1\eta\mathrm{KL}(Q_1\Vert P_1)+\frac2\eta\log(dT)\notag\\
  \leq{}&\frac1\eta\left(\frac{(d+2)\log(dT)+4}{\sigma}+d\sigma^2+d\right)P_T+\frac1\eta\left(3+d\log\left(\frac1\sigma\right) +\frac{d\sigma^2}{2}+2\log(dT)\right),\label{eq:final-termB-OCO}
\end{align}
where the first inequality comes from a combination of~\eqref{eq:termb-intermediate-2},~\eqref{eq:path-length-termB} and~\eqref{eq:J-bound}. The second equality is by the definition of $V_1(Q_1)$ and the last inequality is due to the closed-form expression for the KL divergence between two Gaussian distributions.

\paragraph{Combining All.} Combining the upper bounds~\eqref{eq:final-termA-OCO},~\eqref{eq:final-termB-OCO} and~\eqref{eq:term-C-final} on term~(a), term~(b) and term~(c), we obtain
\begin{align*}
  {}&\DReg_T(\{\u_t\}_{t=1}^T)\\
  \leq{}&\frac{((d+2)\log(dT)+4)P_T}{\eta\sigma} +d\left(\frac{P_T+1/2}{\eta}+\eta G^2T\right)\sigma^2\\
  &{}+\frac1\eta\left(dP_T+3+d\log\frac1\sigma+2\log(dT)\right)\\
  \leq{}&\frac{((d+2)\log(dT)+4)P_T}{\eta\sigma} +\frac{5dT\sigma^2}{2\eta}+\frac1\eta\left(dP_T+3+d\log\frac1\sigma+2\log(dT)\right).
\end{align*}
The last inequality uses $P_T\leq2T$, $\eta\leq1/(2G)$, and $T\geq2$.
We consider the following two cases:
\begin{itemize}
  \item \textbf{Case 1: $P_T\leq T^{-1/2}$.}
  We choose $\sigma=T^{-1/2}$. Substituting this choice into the above bound gives
  \begin{align*}
    \DReg_T(\{\u_t\}_{t=1}^T)\leq\O\left(\frac d\eta\ln(dT)\right).
  \end{align*}
  \item \textbf{Case 2: $P_T>T^{-1/2}$.}
  We choose $\sigma=(P_T\ln(edT)/T)^{1/3}$. Substituting this choice into the above bound gives
  \begin{align*}
    \DReg_T(\{\u_t\}_{t=1}^T) \leq\O\left(\frac d\eta\left[\ln(dT)+T^{1/3}P_T^{2/3}(\ln(dT))^{2/3}\right]\right).
  \end{align*}
\end{itemize}
The proof is completed by combining the two cases.

\end{proof}

\clearpage

\section{Technical Lemmas}
\label{app:technical-lemmas}
\enlargethispage{4\baselineskip}

\begin{lemma}
\label{lemma:trigamma_function}
Let $\psi(u) = \frac{\mathrm{d}}{\mathrm{d} u} \ln \Gamma(u)$ be the digamma function. Then, $g(u) = \gamma u \psi(1+\gamma u) - \ln \Gamma(1+\gamma u)$ is a convex function for $u>0$ and $\gamma > 0$.
\end{lemma}

\begin{proof}[Proof of Lemma~\ref{lemma:trigamma_function}]
Let $\psi(u)=\frac{\mathrm{d}}{\mathrm{d}u}\ln\Gamma(u)$ and define
$g(u)=\gamma u \psi(1+\gamma u)-\ln\Gamma(1+\gamma u)$ for $\gamma>0$.
Set $x=1+\gamma u$ and define
$h(x)=(x-1)\psi(x)-\ln\Gamma(x)$ for $x>1$.
Since $g(u)=h(1+\gamma u)$, it suffices to show that $h$ is convex on $(1,\infty)$.

A direct computation yields
$h'(x)=(x-1)\psi'(x)$ and
$h''(x)=\psi'(x)+(x-1)\psi''(x)$.
Using the integral representations of the polygamma functions, we have
\[
  \psi'(x)=\int_0^\infty \frac{t e^{-xt}}{1-e^{-t}}\,\mathrm{d}t, \qquad \psi''(x)=-\int_0^\infty \frac{t^2 e^{-xt}}{1-e^{-t}}\,\mathrm{d}t,
\]
for $x>1$. Then, by the integration by parts arguments, we obtain for $x>1$,
\begin{align*}
  h''(x) &=\int_0^\infty \frac{t}{e^t-1}\,(1-(x-1)t)\,e^{-(x-1)t}\,\mathrm{d}t \\
  &=\int_0^\infty \frac{t}{e^t-1}\,\frac{\mathrm{d}}{\mathrm{d}t}\!\Big(t e^{-(x-1)t}\Big)\,\mathrm{d}t \\
  &=\left[\frac{t^2 e^{-(x-1)t}}{e^t-1}\right]_{0}^{\infty} -\int_0^\infty \frac{\mathrm{d}}{\mathrm{d}t}\!\left(\frac{t}{e^t-1}\right)\, t e^{-(x-1)t}\,\mathrm{d}t\\
  = {}&-\int_0^\infty \frac{\mathrm{d}}{\mathrm{d}t}\!\left(\frac{t}{e^t-1}\right) \, t e^{-(x-1)t}\,\mathrm{d}t.
\end{align*}
The last equality holds because the function $\frac{t^2 e^{-(x-1)t}}{e^t-1}\rightarrow 0$ when $t\rightarrow 0$ and $t\rightarrow \infty$.

Since $\frac{\mathrm{d}}{\mathrm{d}t}\big(\frac{t}{e^t-1}\big)<0$ for $t>0$, the integrand is nonnegative and not identically zero.
Hence $h''(x)>0$ for all $x>1$, so $h$ is strictly convex. Therefore $g(u)=h(1+\gamma u)$ is strictly convex for $u>0$ and $\gamma>0$.
\end{proof}

\end{document}